\documentclass[11pt]{article}

\usepackage[utf8]{inputenc}
\usepackage[T1]{fontenc}
\usepackage{amsmath,amssymb,amsthm}
\usepackage{mathtools}
\usepackage{booktabs}
\usepackage{longtable}
\usepackage{enumitem}
\usepackage{hyperref}
\usepackage[margin=1in]{geometry}
\usepackage{cleveref}
\usepackage{lmodern}

\theoremstyle{plain}
\newtheorem{theorem}{Theorem}
\newtheorem{proposition}[theorem]{Proposition}

\newtheorem{corollary}[theorem]{Corollary}

\theoremstyle{definition}
\newtheorem{definition}[theorem]{Definition}
\newtheorem{example}[theorem]{Example}
\newtheorem{assumption}[theorem]{Assumption}

\theoremstyle{remark}
\newtheorem{remark}[theorem]{Remark}

\newcommand{\E}{\mathbb{E}}
\newcommand{\Prob}{\mathbb{P}}
\newcommand{\Var}{\mathrm{Var}}
\newcommand{\Cov}{\mathrm{Cov}}
\newcommand{\R}{\mathbb{R}}
\newcommand{\ind}{\mathbf{1}}
\newcommand{\cF}{\mathcal{F}}
\newcommand{\cX}{\mathcal{X}}
\newcommand{\cI}{\mathcal{I}}

\newcommand{\cH}{\mathcal{H}}

\newcommand{\cR}{\mathcal{R}}
\newcommand{\cM}{\mathcal{M}}
\DeclareMathOperator*{\argmin}{arg\,min}
\DeclareMathOperator{\AUC}{AUC}

\title{Fraud Type Decomposition and the Observation-Mechanism
  Taxonomy:\\[4pt]
  Class-Specific Detection Limits in Payment Networks}

\author{Gaurav Dhama}

\date{\today}

\begin{document}
\maketitle

\begin{abstract}
Fraud detection models in payment networks train on
chargeback labels or fraud reports by issuers --- but ``fraud'' is not a single
phenomenon.  A stolen credit card, a cardholder who
disputes a legitimate purchase, a synthetic identity
building credit before a coordinated bust-out, a merchant
laundering transactions through a phantom storefront, and
a victim wiring money to a romance scammer all produce
different labels through fundamentally different
observation pipelines.  Treating these as a homogeneous
population --- as all existing fraud detection frameworks
do --- is a provable source of statistical inefficiency.

We show that the payment industry's 30+ fraud types
collapse into exactly five \emph{observation-mechanism
classes}, each defined by a structurally distinct
censorship pipeline.  This classification is minimal (no
two classes can be merged) and complete (every fraud type
maps to exactly one class).  We prove that estimating
fraud rates separately by class and aggregating strictly
dominates pooled estimation, with the efficiency gap
quantified as a closed-form Jensen penalty.

For each class, we derive the binding theoretical
constraint.  For adversarial labeling (Class~2), the
false-positive corruption rate is the Nash equilibrium of a
dispute game between the cardholder and the issuer.  For
deferred observability (Class~3), the observation
propensity is exactly zero during the bust-out build-up
phase, and no account-level classifier exceeds random
guessing --- detection requires cross-account graph
signals.  For acquirer-observed fraud (Class~4), the
detection floor is an order of magnitude higher than for
standard card fraud.  For victim-authorized fraud
(Class~5), the Bayes-optimal transaction-level classifier
achieves $\AUC = 0.5$ when fraud and legitimate feature
distributions coincide, and detection requires auxiliary
signals outside the payment data.
\end{abstract}

\section{Introduction}
\label{sec:intro}

\subsection{The label problem in fraud detection}
\label{sec:label_problem}

Every supervised fraud detection model in payment networks
trains on labels derived from the chargeback process.  A
cardholder disputes a transaction, the issuing bank files a
chargeback, and the resulting record serves as a ``fraud''
label in the training data.  Transactions without
chargebacks are implicitly labeled ``legitimate.''

This labeling process is impaired in well-documented ways.
Declined transactions generate no labels (authorization
censorship).  Not all fraud is reported by cardholders or
issuers (reporting censorship).  Chargebacks arrive weeks
or months after the transaction (delay censorship).  And
some labels are wrong --- the cardholder may have
fabricated the dispute, or the issuer may have miscoded the
fraud type (label corruption).

Recent work has formalized these impairments.  Dhama~\cite{dhama2026limits} derives minimax lower bounds
on detection performance as a function of authorization,
reporting, delay, and corruption rates, showing that the
four impairments impose a fundamental floor on accuracy.
Dhama~\cite{dhama2026str} constructs the Sequential Triply
Robust (STR) estimator, which corrects for all four
impairments simultaneously and achieves the semiparametric
efficiency bound --- no estimator can have lower asymptotic
variance.

\subsection{The homogeneity assumption}
\label{sec:homogeneity}

These results --- and indeed all existing theoretical
treatments of the fraud label problem --- share a common
assumption: \emph{fraud is treated as a single binary
variable}.  A transaction is either fraud ($Y^*=1$) or
legitimate ($Y^*=0$), and the observation pipeline acts on
all fraud uniformly through a single set of propensities
$(e_0, r_0, p_0)$ and corruption rates
$(\varepsilon_{10}, \varepsilon_{01})$.

This assumption is false.  ``Fraud'' in payment networks
encompasses fundamentally different phenomena --- each with
its own perpetrator, its own label-generating process, and
its own observation pipeline.  This paper shows that the
homogeneity assumption is not merely a simplification ---
it is a \emph{provable source of inefficiency} whose cost
can be quantified exactly.

\subsection{Three transactions, three pipelines}
\label{sec:three_txns}

\begin{example}[Why homogeneity fails]
\label{ex:three_txns}
Consider three transactions flagged as potential fraud.

\textbf{Transaction~A} is a card-not-present purchase using
stolen credentials.  The victim will eventually notice the
unauthorized charge, dispute it with their issuer, and
generate a chargeback label.  The observation pipeline is
the standard three-gate model: authorization (the network's
fraud model may decline it), reporting (the victim must
notice and dispute), and delay (the chargeback takes weeks
to process).

\textbf{Transaction~B} is a legitimate purchase that the
cardholder later disputes to obtain a fraudulent refund.
Here the ``victim'' and the ``fraudster'' are the same
person, and the chargeback label is not a noisy observation
of truth --- it is a strategic lie.  The label is generated
by the perpetrator, not the victim, and the corruption rate
$\varepsilon_{01}$ is not random noise but a deliberate
strategic choice.

\textbf{Transaction~C} is one of fifty small purchases made
by a synthetic identity over six months of credit building
before a coordinated bust-out.  During the build-up phase,
there is no fraud to report, no victim to file a dispute,
and no label of any kind.  The observation pipeline is not
merely impaired --- it is structurally inoperative.
\end{example}

Applying the STR with a single set of propensities averages
across these three pipelines, producing an estimator that
is optimal for none of them.  The propensity $r_0 = 0.70$
that is appropriate for Transaction~A (victim-reported card
theft) is meaningless for Transaction~B (the ``reporting
rate'' is the perpetrator's strategic dispute probability)
and impossible for Transaction~C (there is nothing to
report during build-up).

\subsection{The observation-mechanism taxonomy}
\label{sec:intro_taxonomy}

The payment industry classifies fraud into dozens of types
based on perpetrator identity (third-party, first-party,
synthetic), attack vector (card-not-present, account
takeover, bust-out), and operational characteristics
(e-commerce, point-of-sale, ATM).  From the perspective of
statistical estimation, however, the relevant taxonomy is
not who commits the fraud or how, but \emph{how the fraud
generates --- or fails to generate --- an observable label}.

We show that the industry's extensive fraud taxonomy
collapses into exactly five \emph{observation-mechanism
classes}, each defined by a structurally distinct censorship
pipeline.  Within each class, fraud types share the same
pipeline topology and differ only in propensity magnitudes.
Across classes, the pipeline topology, the identity of the
labeling agent, and the binding constraints on
identification differ fundamentally, requiring class-specific
estimation strategies.

This classification is derived from first principles: each
fraud type is characterized by a \emph{structural
signature} --- a 5-tuple describing the
perpetrator--cardholder relationship, the label-generating
agent, the authorization gate informativeness, the reporting
gate structure, and the temporal observability.  Two fraud
types belong to the same class if and only if their
structural signatures are identical.

\subsection{The five classes}
\label{sec:intro_classes}

\textbf{Class~1: Victim-reported institutional pipeline.}
The perpetrator is external, the victim reports through the
standard chargeback process, and the three-gate observation
model applies directly.  This class encompasses
card-not-present fraud, counterfeit, lost/stolen, account
takeover, and related types.

\textbf{Class~2: Adversarial label generation.}  The
perpetrator is the cardholder, who files a strategic false
dispute.  The corruption rate $\varepsilon_{01}$ is not a
statistical nuisance but an endogenous equilibrium of a
dispute game.  This class encompasses friendly fraud,
chargeback abuse, refund abuse, and related types.

\textbf{Class~3: Deferred observability.}  The fraud unfolds
over a build-up phase during which the observation pipeline
is structurally inoperative ($r_0=0$, $p_0=0$), followed by
a catastrophic exploitation event.  Labels arrive only at
charge-off.  This class encompasses bust-out fraud, synthetic
identity fraud, and sleeper fraud.

\textbf{Class~4: Acquirer-observed pipeline.}  The
perpetrator is the merchant, and the label is generated by
acquirer monitoring rather than cardholder disputes.  The
propensities depend on the acquirer's monitoring
infrastructure, not the issuer's reporting practices.  This
class encompasses transaction laundering, factoring, phantom
merchants, and merchant collusion.

\textbf{Class~5: Victim-authorized with non-informative
features.}  The victim voluntarily authorizes the payment
under social engineering.  The authorization gate is
non-binding ($e_0=1$), and the transaction features carry no
fraud signal ($P(X\mid Y^*\!=\!1)\approx P(X\mid
Y^*\!=\!0)$).  This class encompasses authorized push
payment fraud, romance scams, investment scams, and
impersonation scams.

\subsection{Summary of contributions}
\label{sec:contributions}

We establish the following results.

\begin{enumerate}[label=\textbf{\arabic*.},leftmargin=*]

\item \textbf{Observation-mechanism taxonomy
  (Theorem~\ref{thm:taxonomy}).}  The industry's 30+~fraud
  types collapse into five observation-mechanism classes.
  The classification is minimal (no two classes can be
  merged) and complete (every fraud type maps to exactly
  one class).

\item \textbf{Decomposition dominance
  (Theorem~\ref{thm:decomp_dominance}).}  Estimating fraud
  rates separately by class and aggregating --- the
  \emph{type-decomposed} approach --- strictly dominates
  pooled estimation in mean squared error whenever at least
  two classes have different observation pipelines.

\item \textbf{Jensen penalty quantification
  (Theorem~\ref{thm:jensen_penalty}).}  The MSE gap equals
  the cross-class variance of inverse-propensity weights,
  which is substantial for realistic network parameters
  (approximately 36\% efficiency waste in our numerical
  illustration).

\item \textbf{Endogenous corruption
  (Theorem~\ref{thm:endogenous}).}  For Class~2, the
  false-positive corruption rate $\varepsilon_{01}$ is the
  Nash equilibrium of a dispute game between the cardholder
  and the issuer, increasing in transaction amount and
  decreasing in dispute friction and detection probability.
  Multiple equilibria are possible, with a discontinuous
  tipping point between low-fraud and high-fraud regimes.

\item \textbf{Positivity failure
  (Theorem~\ref{thm:positivity}).}  For Class~3, the
  observation propensity $q_0 = e_0 r_0 p_0 = 0$ exactly
  during the bust-out build-up phase.  The semiparametric
  efficiency bound diverges --- no label-based estimator
  can detect bust-out during build-up.

\item \textbf{Individual-level insufficiency
  (Theorem~\ref{thm:insufficiency}).}  For Class~3, no
  account-level classifier exceeds $\AUC = 0.5$ during
  build-up; detection requires cross-account graph signals
  whose power grows with ring size.

\item \textbf{Acquirer-side efficiency bound
  (Theorem~\ref{thm:acquirer}).}  For Class~4, the
  detection floor depends on the acquirer's monitoring
  propensity $r_{\mathrm{acq}}$ rather than the issuer's
  reporting propensity, and is an order of magnitude higher.

\item \textbf{Authorization degeneracy and feature
  non-informativeness
  (Theorems~\ref{thm:auth_degen}--\ref{thm:feature_noninf}).}
  For Class~5, the authorization correction vanishes
  when $e_0=1$, and the Bayes-optimal classifier achieves
  $\AUC = 0.5$ when fraud and legitimate feature
  distributions coincide.  Detection requires auxiliary
  signals outside the transaction data.

\end{enumerate}

\subsection{Connection to prior work}
\label{sec:series}

This paper builds on the theoretical framework
of~\cite{dhama2026limits,dhama2026str}, which formalize the
observation pipeline and construct the STR estimator for a
homogeneous fraud population.  The present paper removes the
homogeneity assumption, showing that the pooled STR is
provably suboptimal for the heterogeneous fraud mixture that
exists in every real payment network, and that each
observation-mechanism class demands its own detection
strategy.  The economic incentives that shape each class's observation
pipeline --- and the mechanism design problem of optimizing
information quality across actors --- are natural extensions
discussed in Section~\ref{sec:discussion}.

\subsection{Outline}
\label{sec:outline}

Section~\ref{sec:related} reviews related work.
Section~\ref{sec:setup} introduces the latent-type mixture
model.  Section~\ref{sec:taxonomy} presents the
observation-mechanism taxonomy and proves its completeness
and minimality.  Section~\ref{sec:dominance} establishes the
decomposition dominance theorem and quantifies the Jensen
penalty.  Sections~\ref{sec:class1}--\ref{sec:class5}
derive class-specific results.
Section~\ref{sec:discussion} discusses limitations,
practical type identification, and future work.  Section~\ref{sec:conclusion} concludes.
\section{Related work}
\label{sec:related}

\paragraph{Information limits and label recovery in payment
networks.}
Dhama~\cite{dhama2026limits} proves minimax lower bounds on
fraud detection under four information impairments and
introduces the conditional delay penalty via Jensen's
inequality.  Dhama~\cite{dhama2026str} constructs the STR
estimator, proves it achieves the semiparametric efficiency
bound, and derives the optimal training delay.  Both papers
assume a homogeneous fraud population.  The present paper
removes this assumption and shows that decomposition by
observation-mechanism class strictly improves estimation.

\paragraph{Fraud typologies.}
The payment industry maintains extensive fraud taxonomies
organized by perpetrator (third-party, first-party,
synthetic), attack vector (card-not-present, account
takeover, bust-out), and operational category (e-commerce,
point-of-sale, ATM).  Industry standards such as Mastercard's
TC40 fraud type codes and Visa's SAFE categories provide
granular classification for operational purposes.  These
taxonomies are designed for investigation and prevention, not
for statistical estimation.  Our observation-mechanism
taxonomy complements the industry taxonomy by organizing
fraud types according to their censorship pipeline structure,
which determines the appropriate estimation strategy.

\paragraph{Mixture models and heterogeneous treatment
effects.}
The causal inference literature addresses treatment effect
heterogeneity through conditional average treatment effects
(CATE; \cite{athey2016recursive}), causal forests
\cite{wager2018estimation}, and meta-learners
\cite{kunzel2019metalearners}.  These methods decompose a
heterogeneous population by covariate values.  Our
decomposition is fundamentally different: we decompose by
\emph{observation-mechanism structure}, not by covariate
values.  Two transactions with identical covariates may
belong to different classes (e.g., a legitimate purchase and
a friendly-fraud dispute of the same purchase), and the
appropriate estimator differs between them.

\paragraph{Strategic labeling and mechanism design.}
The game-theoretic analysis of Class~2 (adversarial label
generation) connects to the literature on strategic
classification \cite{hardt2016strategic}, where agents
manipulate their features to obtain favorable predictions.
In our setting, the agent (first-party fraudster) manipulates
the \emph{label}, not the features --- filing a false dispute
to create a fraudulent chargeback.  This is closer to the
strategic reporting literature in mechanism design
\cite{myerson1981optimal}, where agents misreport private
information.  The endogenous corruption rate
$\varepsilon_{01}$ is the equilibrium reporting strategy of
the agent.

\paragraph{Selective labels and censored feedback.}
Chang and Wiens~\cite{chang2024dcem} address disparate
censorship with an EM-based approach for a single selection
gate.  Malinsky et al.~\cite{malinsky2022selfcensoring}
derive efficiency bounds for nonmonotone MNAR data under a
no-self-censoring assumption.  Ha et
al.~\cite{ha2024mnar} prove that unbiasedness under MNAR
inevitably leads to unbounded variance when propensity scores
approach zero.  Our Class~3 (deferred observability)
provides a concrete, structurally motivated instance of
Ha et al.'s impossibility: during bust-out build-up, the
propensity is \emph{exactly} zero, not merely small.

\paragraph{Bust-out and synthetic identity fraud.}
The practitioner literature on bust-out detection emphasizes
behavioral features (spending velocity, credit utilization
trajectories) and network analysis (ring detection,
shared-attribute linking).  To our knowledge, no prior work
formalizes the \emph{impossibility} of individual-level
detection during the build-up phase or derives the
information-theoretic necessity of cross-account signals.
Our Theorem~\ref{thm:insufficiency} provides this
formalization.

\paragraph{Scam detection.}
Authorized push payment (APP) fraud and social engineering
scams have received increasing attention from regulators
(UK Payment Systems Regulator, 2023) and industry bodies.
Detection approaches rely on behavioral biometrics, session
analysis, and recipient risk scoring --- all signals
\emph{outside} the standard transaction feature set.  Our
Theorem~\ref{thm:feature_noninf} provides the theoretical
justification: when the victim performs the transaction using
their own device and behavioral patterns, the standard
transaction features carry zero discriminative power.

\section{Setup and notation}
\label{sec:setup}

We build on the framework of~\cite{dhama2026str}, extending
it to accommodate a heterogeneous fraud population.

\subsection{The transaction and its true state}
\label{sec:transaction}

\begin{definition}[Transaction]
\label{def:transaction}
A transaction $t$ is characterized by:
\begin{enumerate}[label=(\roman*)]
\item Observable features $X_t \in \cX$ (amount, merchant
  category, channel, location, time, etc.);
\item Issuer identity $I_t \in \cI$;
\item A latent true fraud indicator $Y^*_t \in \{0,1\}$,
  where $Y^*_t=1$ denotes fraud;
\item A latent fraud type $Z_t \in \{0,1,\ldots,K\}$, where
  $Z_t=0$ denotes legitimate and $Z_t=k$ for $k\geq 1$
  denotes fraud of type~$k$.
\end{enumerate}
The fraud indicator and type are related by
$Y^*_t = \ind[Z_t \geq 1]$.
\end{definition}

\begin{remark}[Type vs.\ label]
\label{rem:type_vs_label}
The fraud type $Z_t$ is a finer partition than the fraud
indicator $Y^*_t$.  Knowing $Y^*_t=1$ tells you the
transaction is fraudulent; knowing $Z_t=k$ tells you
\emph{what kind} of fraud it is.  The observation pipeline
depends on $Z_t$, not just on $Y^*_t$.  This is the source
of the heterogeneity that this paper formalizes.
\end{remark}

\subsection{The latent-type mixture model}
\label{sec:mixture}

\begin{definition}[Mixing probabilities]
\label{def:mixing}
The \emph{type probability} conditional on covariates is
\begin{equation}
\label{eq:mixing}
\pi_k(x) = \Prob(Z_t = k \mid X_t = x),
\qquad k = 0, 1, \ldots, K,
\end{equation}
with $\sum_{k=0}^{K}\pi_k(x) = 1$.  The marginal type
prevalence is $\bar\pi_k = \E[\pi_k(X)]$.
\end{definition}

\begin{definition}[Type-specific fraud probability]
\label{def:type_fraud}
For each fraud type $k \geq 1$, define
\begin{equation}
\label{eq:type_fraud}
f_0^{(k)}(x) = \Prob(Y^*_t = 1 \mid X_t = x, Z_t = k) = 1,
\end{equation}
since conditioning on $Z_t = k \geq 1$ implies $Y^*_t = 1$
by definition.  The population fraud rate decomposes as:
\begin{equation}
\label{eq:fraud_decomp}
\Psi = \E[Y^*] = \sum_{k=1}^{K} \bar\pi_k.
\end{equation}
\end{definition}

\begin{remark}[Why the decomposition matters for estimation]
\label{rem:why_decomp}
Equation~\eqref{eq:fraud_decomp} is trivial as an identity.
The decomposition becomes non-trivial when we recognize that
each type $k$ has a \emph{different observation mechanism}.
Estimating $\bar\pi_k$ requires correcting for the
type-$k$-specific censorship pipeline, which differs across
classes.  Estimating $\Psi = \sum_k\bar\pi_k$ with a single
pooled pipeline averages over these differences and incurs a
provable penalty (Theorem~\ref{thm:decomp_dominance}).
\end{remark}

\subsection{Type-specific observation pipelines}
\label{sec:type_pipelines}

\begin{definition}[Type-specific propensities]
\label{def:type_propensities}
For each fraud type $k$ and transaction $t$ with $Z_t = k$,
define the type-specific observation pipeline:
\begin{align}
e_0^{(k)}(H_0) &= \Prob(A_t = 1 \mid H_{0,t},\, Z_t = k)
  && \text{(authorization propensity)},
  \label{eq:type_auth}\\
r_0^{(k)}(H_1) &= \Prob(R_t = 1 \mid H_{1,t},\, A_t = 1,
  \, Z_t = k)
  && \text{(reporting propensity)},
  \label{eq:type_report}\\
p_0^{(k)}(H_2) &= \Prob(M_t = 1 \mid H_{2,t},\, A_t R_t
  = 1,\, Z_t = k)
  && \text{(delay propensity)},
  \label{eq:type_delay}\\
\varepsilon_{10}^{(k)}(H_2) &= \Prob(\tilde Y_t = 0 \mid
  Y^*_t = 1,\, O_t = 1,\, H_{2,t},\, Z_t = k)
  && \text{(false-negative corruption)},
  \label{eq:type_fn}\\
\varepsilon_{01}^{(k)}(H_2) &= \Prob(\tilde Y_t = 1 \mid
  Y^*_t = 0,\, O_t = 1,\, H_{2,t},\, Z_t = k)
  && \text{(false-positive corruption)},
  \label{eq:type_fp}
\end{align}
where $H_0, H_1, H_2$ are the stage histories defined
in~\cite{dhama2026str}, and $O_t = A_t R_t M_t$ is the
observation indicator.
\end{definition}

\begin{definition}[Type-specific observation rate]
\label{def:type_obs_rate}
The total observation propensity for type~$k$ is
\begin{equation}
\label{eq:type_obs_rate}
q_0^{(k)}(H_0) = e_0^{(k)}(H_0) \cdot r_0^{(k)}(H_1)
  \cdot p_0^{(k)}(H_2),
\end{equation}
and the type-specific corruption penalty is
\begin{equation}
\label{eq:type_corruption}
\gamma^{(k)}(H_2) = \bigl(1 - \varepsilon_{10}^{(k)}(H_2)
  - \varepsilon_{01}^{(k)}(H_2)\bigr)^2.
\end{equation}
\end{definition}

\begin{example}[Why type-specific propensities differ]
\label{ex:propensity_diff}
Consider two fraud types at the same merchant, same amount,
same time of day:

\smallskip\noindent
\textbf{Type 1 (card theft):} A stolen card is used for an
online purchase.  The fraud model assigns a high risk score
$\to$ authorization propensity $e_0^{(1)} = 0.60$ (40\%
chance of decline).  The victim notices and reports within
30~days $\to$ reporting propensity $r_0^{(1)} = 0.75$.
The chargeback arrives within 60~days $\to$ delay propensity
$p_0^{(1)} = 0.85$.  Total observation rate:
$q_0^{(1)} = 0.60 \times 0.75 \times 0.85 = 0.38$.

\smallskip\noindent
\textbf{Type 2 (friendly fraud):} The cardholder makes the
purchase themselves, then disputes it.  The fraud model sees
a normal transaction $\to$ $e_0^{(2)} = 0.98$.  The
cardholder strategically files a dispute $\to$ $r_0^{(2)}$
is not an institutional propensity but a strategic choice.
The chargeback arrives on normal schedule $\to$
$p_0^{(2)} = 0.85$.  But the label is \emph{wrong}: the
cardholder is claiming fraud on a legitimate transaction $\to$
$\varepsilon_{01}^{(2)} = 0.30$.

\smallskip\noindent
The observation pipelines are structurally different: Type~1
has informative authorization censorship and truthful labels;
Type~2 has no authorization censorship but adversarially
corrupted labels.  A pooled STR using average propensities
is optimal for neither.
\end{example}

\subsection{The pooled vs.\ decomposed estimand}
\label{sec:pooled_vs_decomp}

\begin{definition}[Pooled estimand]
\label{def:pooled}
The \emph{pooled} fraud rate is the target of
\cite{dhama2026str}:
\begin{equation}
\label{eq:pooled}
\Psi = \E[Y^*] = \sum_{k=1}^K \bar\pi_k.
\end{equation}
\end{definition}

\begin{definition}[Type-specific estimands]
\label{def:type_estimands}
The \emph{type-$k$ fraud rate} is:
\begin{equation}
\label{eq:type_estimand}
\Psi_k = \Prob(Z = k) = \bar\pi_k,
\qquad k = 1, \ldots, K.
\end{equation}
The pooled estimand is the sum:
$\Psi = \sum_{k=1}^K \Psi_k$.
\end{definition}

\begin{definition}[Type-specific efficiency bound]
\label{def:type_eff_bound}
Extending the efficiency bound
of~\cite[Theorem~30]{dhama2026str} to the type-conditional
setting, the semiparametric efficiency bound for estimating
$\Psi_k$ is:
\begin{equation}
\label{eq:type_eff_bound}
\sigma^2_{\mathrm{eff}}(k)
= \E\!\left[
  \frac{f_0^{(k)}(X)(1-f_0^{(k)}(X))}
       {q_0^{(k)}(H_0)\,\gamma^{(k)}(H_2)}
  \;\bigg|\; Z=k
\right]
+ \Var\!\bigl(f_0^{(k)}(X) \mid Z=k\bigr).
\end{equation}
Since $f_0^{(k)}(x) = \Prob(Y^*=1 \mid X=x, Z=k) = 1$ for
$k\geq 1$ (fraud types are defined by $Y^*=1$), the first
term simplifies.  However, the \emph{conditional} fraud
probability $\Prob(Z=k \mid X=x)$ --- which is what we
estimate in practice --- is not identically~1.  The
appropriate efficiency bound for estimating $\bar\pi_k$ is:
\begin{equation}
\label{eq:type_eff_bound_v2}
\sigma^2_{\mathrm{eff}}(k)
= \E\!\left[
  \frac{\pi_k(X)(1-\pi_k(X))}
       {q_0^{(k)}(H_0)\,\gamma^{(k)}(H_2)}
\right],
\end{equation}
where the expectation is over the full population.
\end{definition}

\begin{remark}[Interpreting the type-specific bound]
\label{rem:type_bound_interp}
The bound~\eqref{eq:type_eff_bound_v2} shows that estimating
the prevalence of fraud type~$k$ depends on:
\begin{enumerate}[label=(\roman*)]
\item The base-rate variance $\pi_k(1-\pi_k)$ --- rare fraud
  types are harder to estimate;
\item The observation rate $q_0^{(k)}$ --- types with lower
  observation propensity have higher bounds;
\item The corruption penalty $\gamma^{(k)}$ --- types with
  higher corruption have higher bounds.
\end{enumerate}
Crucially, different fraud types can have very different
values of $q_0^{(k)}$ and $\gamma^{(k)}$, leading to
efficiency bounds that differ by orders of magnitude
(Example~\ref{ex:bound_diff}).
\end{remark}

\begin{example}[Efficiency bounds differ by orders of
  magnitude]
\label{ex:bound_diff}
Consider four fraud types with the following average
observation parameters:
\begin{center}
\begin{tabular}{lccccr}
\toprule
\textbf{Type} & $\bar\pi_k$ & $\bar q_0^{(k)}$ &
  $\bar\gamma^{(k)}$ &
  $\frac{\bar\pi_k(1-\bar\pi_k)}
        {\bar q_0^{(k)}\bar\gamma^{(k)}}$ &
  \textbf{Relative}\\
\midrule
Class~1 (card theft) & 0.005 & 0.50 & 0.81 &
  0.012 & $1\times$\\
Class~2 (friendly fraud) & 0.003 & 0.33 & 0.25 &
  0.036 & $3\times$\\
Class~3 (bust-out) & 0.001 & 0.02 & 0.90 &
  0.056 & $4.5\times$\\
Class~4 (merchant fraud) & 0.002 & 0.07 & 0.64 &
  0.045 & $3.7\times$\\
\bottomrule
\end{tabular}
\end{center}
The efficiency bound for bust-out estimation is 4.5$\times$
higher than for card theft, despite bust-out being rarer.
The bottleneck is the observation rate:
$\bar q_0^{(3)} = 0.02$ (almost no labels during build-up)
vs.\ $\bar q_0^{(1)} = 0.50$ (half of card-theft labels are
observed).  A pooled STR that averages these rates
\emph{underestimates} the difficulty of estimating bust-out
and \emph{overestimates} the difficulty of estimating card
theft.
\end{example}

\subsection{Assumptions inherited from the companion papers}
\label{sec:assumptions}

We inherit the following assumptions
from~\cite{dhama2026str}, applied type-conditionally.

\begin{assumption}[Type-specific ignorability]
\label{asm:ignorability}
For each class $k$ and each gate $g \in \{A, R, M\}$, the
gate decision is independent of the latent fraud state
conditional on the stage history and the type:
\begin{align}
Y^* &\perp A \mid H_0,\, Z=k,
  \label{eq:ign_auth}\\
Y^* &\perp R \mid H_1,\, A=1,\, Z=k,
  \label{eq:ign_report}\\
Y^* &\perp M \mid H_2,\, AR=1,\, Z=k.
  \label{eq:ign_delay}
\end{align}
\end{assumption}

\begin{assumption}[Type-specific positivity]
\label{asm:positivity}
For each class $k$ where the STR is applicable (Classes~1,
2, and~4), there exist constants
$\delta_k > 0$ such that
\begin{equation}
\label{eq:positivity}
q_0^{(k)}(H_0) > \delta_k > 0
\qquad\text{a.s.}
\end{equation}
\emph{This assumption is explicitly violated for Class~3
during the build-up phase and partially violated for Class~5.
These violations are not limitations of the framework ---
they are the defining structural features of these classes
(Theorems~\ref{thm:positivity}
and~\ref{thm:auth_degen}).}
\end{assumption}

\begin{assumption}[Type observability in mature data]
\label{asm:type_obs}
For transactions with fully matured labels ($O=1$ and
sufficient elapsed time), the fraud type $Z_t$ is
identifiable from the label characteristics:
\begin{equation}
\label{eq:type_obs}
Z_t \text{ is observable when } O_t = 1
  \text{ and the dispute is fully resolved.}
\end{equation}
\end{assumption}

\begin{remark}[Justifying type observability]
\label{rem:type_obs}
Assumption~\ref{asm:type_obs} is justified in practice by:
\begin{enumerate}[label=(\roman*)]
\item \textbf{Chargeback reason codes:} Mastercard and Visa
  assign reason codes (e.g., ``4837 -- no cardholder
  authorization'' for third-party fraud, ``4853 --
  cardholder dispute'' for first-party) that distinguish
  fraud types.
\item \textbf{Dispute resolution outcomes:} Representment
  and arbitration outcomes reveal whether the original
  dispute was legitimate (third-party) or fraudulent
  (first-party).
\item \textbf{Charge-off indicators:} Bust-out fraud is
  identified by the charge-off event and the account's
  credit history pattern.
\item \textbf{Network monitoring programs:} Merchant fraud
  is identified through acquirer monitoring programs (e.g.,
  Mastercard's MATCH list, Visa's VMAS).
\end{enumerate}
The assumption requires that type classification is accurate
\emph{for mature, fully resolved transactions} --- not for
real-time classification.  Section~\ref{sec:discussion}
discusses the sensitivity to type misclassification.
\end{remark}

\section{The observation-mechanism taxonomy}
\label{sec:taxonomy}

This section derives the five observation-mechanism classes
from first principles, proves that the classification is
complete and minimal, and maps the industry's fraud types to
the five classes.

\subsection{The classification criterion}
\label{sec:classification_criterion}

The goal is to partition the set of fraud types into classes
such that (a)~types within a class share the same STR
functional form and (b)~types across classes require
structurally different estimators.  The partition is
determined by five structural questions about the observation
pipeline.

\begin{definition}[Structural signature]
\label{def:structural_sig}
The \emph{structural signature} of a fraud type $k$ is the
5-tuple
\begin{equation}
\label{eq:signature}
\mathcal{S}(k) = \bigl(
  \mathrm{Q1}(k),\;
  \mathrm{Q2}(k),\;
  \mathrm{Q3}(k),\;
  \mathrm{Q4}(k),\;
  \mathrm{Q5}(k)
\bigr),
\end{equation}
where:
\begin{enumerate}[label=\textbf{Q\arabic*:},leftmargin=3em]
\item \textbf{Perpetrator--cardholder relationship.}  Who
  commits the fraud relative to the cardholder?
  \begin{equation}
  \label{eq:Q1}
  \mathrm{Q1}(k) \in \{\textsc{external},\;
    \textsc{cardholder},\;
    \textsc{merchant},\;
    \textsc{fabricated}\}.
  \end{equation}

\item \textbf{Label-generating agent.}  Who generates the
  observed label $\tilde Y$?
  \begin{equation}
  \label{eq:Q2}
  \mathrm{Q2}(k) \in \{\textsc{victim},\;
    \textsc{perpetrator},\;
    \textsc{acquirer},\;
    \textsc{charge\text{-}off},\;
    \textsc{none}\}.
  \end{equation}

\item \textbf{Authorization gate informativeness.}  Does
  $e_0^{(k)} < 1$ provide discriminative signal?
  \begin{equation}
  \label{eq:Q3}
  \mathrm{Q3}(k) \in \{\textsc{informative}\;
    (e_0^{(k)} < 1),\;\;
    \textsc{degenerate}\;
    (e_0^{(k)} = 1)\}.
  \end{equation}

\item \textbf{Reporting gate structure.}  Is the reporting
  propensity $r_0^{(k)}$ an institutional process, a
  strategic decision, driven by a different actor, or
  structurally zero?
  \begin{equation}
  \label{eq:Q4}
  \mathrm{Q4}(k) \in \{\textsc{institutional},\;
    \textsc{strategic},\;
    \textsc{acquirer\text{-}driven},\;
    \textsc{zero}\}.
  \end{equation}

\item \textbf{Temporal observability.}  When does the fraud
  state $Y^*$ become potentially observable?
  \begin{equation}
  \label{eq:Q5}
  \mathrm{Q5}(k) \in \{\textsc{immediate},\;
    \textsc{delayed},\;
    \textsc{catastrophic},\;
    \textsc{never}\}.
  \end{equation}
\end{enumerate}
\end{definition}

\begin{remark}[Why these five questions]
\label{rem:why_five}
Each question corresponds to a structural feature that
changes the \emph{functional form} of the efficient
estimator:
\begin{itemize}[leftmargin=*]
\item \textbf{Q1} determines whether the label is
  trustworthy (external perpetrator $\to$ victim labels
  truthfully) or adversarial (cardholder perpetrator $\to$
  label is a strategic lie).
\item \textbf{Q2} determines which actor's behavior must be
  modeled to estimate the reporting propensity.
\item \textbf{Q3} determines whether the authorization
  correction term in the STR is active or vanishes.
\item \textbf{Q4} determines the structure of the reporting
  model: a standard propensity (institutional), a game
  equilibrium (strategic), a different actor's decision
  (acquirer-driven), or a structural impossibility (zero).
\item \textbf{Q5} determines whether the standard delay
  model applies or whether the temporal structure requires a
  fundamentally different approach.
\end{itemize}
A different answer to any single question changes the
estimator's functional form.  This is why each question is
necessary and no question is redundant.
\end{remark}

\subsection{Formal definition of the five classes}
\label{sec:five_classes}

\begin{definition}[Observation-mechanism class]
\label{def:obs_class}
Two fraud types $k$ and $k'$ belong to the same
\emph{observation-mechanism class} if and only if they have
the same structural signature:
\begin{equation}
\label{eq:equiv}
k \sim k' \iff \mathcal{S}(k) = \mathcal{S}(k').
\end{equation}
\end{definition}

We now define each class by its structural signature, formal
observation model, and the estimation strategy it requires.

\subsubsection{Class~1: Victim-reported institutional
  pipeline}
\label{sec:class1_def}

\begin{definition}[Class~1]
\label{def:class1}
A fraud type $k$ belongs to Class~1 if its structural
signature is:
\begin{equation}
\label{eq:sig_class1}
\mathcal{S}^{(1)} = \bigl(
  \textsc{external},\;
  \textsc{victim},\;
  \textsc{informative},\;
  \textsc{institutional},\;
  \textsc{delayed}
\bigr).
\end{equation}
\end{definition}

\noindent\textbf{Observation model.}  The observation
pipeline is exactly the three-gate model
of~\cite{dhama2026str}:
\begin{equation}
\label{eq:obs_class1}
O^{(1)}_t = A_t \cdot R_t \cdot M_t,
\end{equation}
where $A_t \sim \mathrm{Bernoulli}(e_0^{(1)}(H_0))$,
$R_t \sim \mathrm{Bernoulli}(r_0^{(1)}(H_1))$,
$M_t \sim \mathrm{Bernoulli}(p_0^{(1)}(H_2))$, with all
three propensities bounded away from zero
(Assumption~\ref{asm:positivity}).

\smallskip\noindent\textbf{Estimation strategy.}  The STR
of~\cite{dhama2026str} applies without modification, using
type-1-specific propensity and outcome models.

\smallskip\noindent\textbf{Defining characteristics.}
\begin{enumerate}[label=(\roman*),leftmargin=*]
\item The perpetrator is an external criminal unknown to the
  cardholder.
\item The victim (cardholder) generates the label by
  reporting unauthorized transactions to the issuer.
\item The authorization gate is informative: the network's
  fraud model assigns elevated risk scores to stolen-card
  transactions, creating non-trivial selection bias
  ($e_0^{(1)} < 1$).
\item The reporting gate is institutional: it depends on
  victim awareness, issuer dispute processes, and operational
  efficiency --- not on strategic behavior by the perpetrator.
\item Labels arrive with delay but within the standard
  chargeback window (typically 30--120~days).
\end{enumerate}

\begin{example}[Card-not-present fraud: a Class~1 instance]
\label{ex:cnp}
A criminal obtains a credit card number from a data breach
and uses it for an online purchase.  The network's fraud
model evaluates the transaction in real time: the shipping
address differs from the billing address, the device
fingerprint is new, and the merchant category is high-risk.
The model assigns a risk score of 85/100, and the
authorization propensity is $e_0^{(1)} = 0.65$ --- there is
a 35\% chance of decline.

If approved, the legitimate cardholder notices the
unauthorized charge on their next statement (typically within
15~days) and calls the issuer to dispute it.  The reporting
propensity is $r_0^{(1)} = 0.80$ --- most victims notice and
report, but some do not (small amounts, infrequent statement
review, elderly cardholders).

The issuer processes the dispute and files a chargeback.  The
delay propensity is $p_0^{(1)} = 0.85$ --- most chargebacks
arrive within the training window, but some are delayed by
processing backlogs.

The corruption rates are low: $\varepsilon_{10}^{(1)} =
0.03$ (issuer occasionally miscodes a fraud as legitimate)
and $\varepsilon_{01}^{(1)} = 0.02$ (rare cases where a
legitimate transaction is mistakenly disputed).
\end{example}

\subsubsection{Class~2: Adversarial label generation}
\label{sec:class2_def}

\begin{definition}[Class~2]
\label{def:class2}
A fraud type $k$ belongs to Class~2 if its structural
signature is:
\begin{equation}
\label{eq:sig_class2}
\mathcal{S}^{(2)} = \bigl(
  \textsc{cardholder},\;
  \textsc{perpetrator},\;
  \textsc{degenerate},\;
  \textsc{strategic},\;
  \textsc{delayed}
\bigr).
\end{equation}
\end{definition}

\noindent\textbf{Observation model.}  The authorization gate
is non-binding because the cardholder authorized the
transaction themselves:
\begin{equation}
\label{eq:obs_class2}
O^{(2)}_t = 1 \cdot D_t \cdot M_t,
\end{equation}
where $A_t = 1$ always (the cardholder wanted the
transaction approved), $D_t \in \{0,1\}$ is the cardholder's
\emph{strategic dispute decision}, and
$M_t \sim \mathrm{Bernoulli}(p_0^{(2)}(H_2))$.

The critical structural difference: the ``reporting gate''
$D_t$ is not an institutional propensity but a
\emph{strategic decision} by the perpetrator.  The
cardholder files a dispute (and thereby generates a label)
if the expected payoff exceeds the cost:
\begin{equation}
\label{eq:dispute_decision}
D_t = \ind\!\bigl[
  \alpha(X_t) \cdot a_t - c - \beta(X_t) \cdot p
  + \xi_t > 0
\bigr],
\end{equation}
where $\alpha(x)$ is the chargeback success probability,
$a_t$ is the transaction amount, $c$ is the dispute filing
cost, $\beta(x)$ is the probability of being identified as
first-party fraud, $p$ is the penalty for identification,
and $\xi_t$ is a random utility shock.

\smallskip\noindent\textbf{Estimation strategy.}  The STR
requires modification: the corruption rate
$\varepsilon_{01}^{(2)}$ is endogenous (it equals the
equilibrium dispute probability; see
Theorem~\ref{thm:endogenous}) and must be estimated from
the dispute game's structure, not from historical audit
data alone.

\smallskip\noindent\textbf{Defining characteristics.}
\begin{enumerate}[label=(\roman*),leftmargin=*]
\item The perpetrator \emph{is} the cardholder (or an
  authorized user on the account).
\item The perpetrator generates the label by filing a
  \emph{false} dispute --- the label is a strategic lie.
\item The authorization gate is degenerate: $e_0^{(2)}
  \approx 1$ because the cardholder made the purchase
  intentionally.
\item The reporting gate is strategic: the dispute decision
  depends on the expected payoff, detection risk, and
  penalty, not on institutional processes.
\item The false-positive corruption rate
  $\varepsilon_{01}^{(2)}$ is not a nuisance parameter
  --- it is the \emph{primary phenomenon} to be modeled.
\end{enumerate}

\begin{example}[Friendly fraud: a Class~2 instance]
\label{ex:friendly}
A cardholder purchases a \$500 electronics item online,
receives and keeps the item, then calls the issuer claiming
``I did not make this purchase.''  The issuer files a
chargeback.

From the network's perspective, this transaction has:
\begin{itemize}[leftmargin=*]
\item $e_0^{(2)} = 0.99$ --- the fraud model saw nothing
  suspicious (the cardholder's own device, usual location,
  reasonable amount).
\item $D_t = 1$ --- the cardholder strategically decided to
  dispute.  The expected payoff: \$500 $\times$ 0.85
  (chargeback success rate) $-$ \$0 (no dispute cost for
  the cardholder) $-$ 0.05 $\times$ \$2000 (5\% chance of
  being identified and penalized) $=$ \$325 $>$ 0.  The
  dispute is profitable.
\item The observed label is $\tilde Y = 1$ (``fraud''),
  but the true label is $Y^* = 0$ (the transaction was
  legitimate).  This is a \emph{false positive} label
  generated by the perpetrator's strategic behavior.
\end{itemize}

If the network trains on this label naively, it learns that
``transactions with these features are fraudulent'' --- but
the transaction was legitimate.  The model is being
\emph{poisoned} by strategically generated labels.
\end{example}

\begin{remark}[Class~2 inverts the corruption structure]
\label{rem:class2_inversion}
In Class~1, the corruption concern is primarily
$\varepsilon_{10}$ (the issuer \emph{misses} real fraud).
In Class~2, the corruption concern is primarily
$\varepsilon_{01}$ (the perpetrator \emph{creates} false
fraud labels).  The label corruption is not random noise ---
it is \emph{adversarial signal injection}.  This inversion
requires a fundamentally different treatment of the
corruption correction in the STR
(Section~\ref{sec:class2}).
\end{remark}

\subsubsection{Class~3: Deferred observability}
\label{sec:class3_def}

\begin{definition}[Class~3]
\label{def:class3}
A fraud type $k$ belongs to Class~3 if its structural
signature is:
\begin{equation}
\label{eq:sig_class3}
\mathcal{S}^{(3)} = \bigl(
  \textsc{fabricated},\;
  \textsc{charge\text{-}off},\;
  \textsc{degenerate},\;
  \textsc{zero},\;
  \textsc{catastrophic}
\bigr).
\end{equation}
\end{definition}

\noindent\textbf{Observation model.}  Class~3 fraud has a
\emph{two-phase} temporal structure:

\smallskip\noindent\emph{Build-up phase} ($t < \tau$,
where $\tau$ is the bust-out time):
\begin{equation}
\label{eq:obs_class3_buildup}
O^{(3)}_t = A_t \cdot 0 \cdot 0 = 0
\qquad\text{(structurally zero)}.
\end{equation}
During build-up, the account behaves legitimately.  There is
no fraud to report ($R_t = 0$ with certainty because no
fraud has occurred yet), and no label to mature ($M_t$ is
undefined).  The observation propensity is \emph{exactly
zero} --- not small, not approximately zero, but identically
zero.

\smallskip\noindent\emph{Exploitation phase}
($t \geq \tau$):
\begin{equation}
\label{eq:obs_class3_exploit}
O^{(3)}_t = A_t \cdot R_t \cdot M_t,
\end{equation}
where the pipeline resembles Class~1 but with a sudden burst
of high-value transactions.  Labels eventually arrive through
the charge-off process (typically 120--180~days after
delinquency).

\smallskip\noindent\textbf{Estimation strategy.}  The STR
is \emph{undefined} during the build-up phase because
positivity fails ($q_0^{(3)} = 0$).  The estimand shifts
from the transaction level to the \emph{account level}:
\begin{equation}
\label{eq:class3_estimand}
\theta_a = \ind[\text{account } a \text{ will eventually
  bust out}].
\end{equation}
Detection requires cross-account graph signals, not
individual transaction correction
(Theorem~\ref{thm:insufficiency}).

\smallskip\noindent\textbf{Defining characteristics.}
\begin{enumerate}[label=(\roman*),leftmargin=*]
\item The perpetrator operates through a fabricated or
  compromised identity that behaves legitimately during
  build-up.
\item Labels arrive only at the catastrophic event
  (charge-off), not through the standard chargeback process.
\item The authorization gate is degenerate during build-up:
  $e_0^{(3)} \approx 1$ because the transactions look
  normal.
\item The reporting gate is \emph{structurally zero} during
  build-up: there is no fraud to report.
\item The temporal structure is two-phase: a silent build-up
  followed by a sudden exploitation burst.
\end{enumerate}

\begin{example}[Synthetic identity bust-out: a Class~3
  instance]
\label{ex:bustout}
A fraud ring creates a synthetic identity --- a fictitious
person constructed from a combination of a real Social
Security number (belonging to a child, deceased individual,
or immigrant) and fabricated biographical details.  The
ring applies for a credit card and is approved with a
\$2{,}000 limit.

\smallskip\noindent\textbf{Months 1--6 (build-up):}  The
synthetic identity makes small, regular purchases
(\$30--\$80 per transaction, 3--5 per month) and pays the
balance in full each month.  The account demonstrates
perfect payment behavior.  The issuer increases the credit
limit to \$15{,}000.

During this phase:
\begin{itemize}[leftmargin=*]
\item Every transaction is legitimate ($Y^*_t = 0$).
\item The fraud model assigns low risk scores
  ($e_0^{(3)} = 0.99$).
\item There is nothing to report ($r_0^{(3)} = 0$).
\item No labels exist ($O^{(3)}_t = 0$ for every $t$).
\item \emph{The account is indistinguishable from a
  legitimate new customer.}
\end{itemize}

\smallskip\noindent\textbf{Month 7 (exploitation):}  The
ring maxes out all credit lines simultaneously:
\$15{,}000 in purchases across 40 transactions in 72~hours,
at high-resale-value merchants (electronics, gift cards,
jewelry).  The ring abandons the identity.  The issuer
charges off the account after 120~days of delinquency.

During exploitation, labels eventually arrive through the
charge-off process --- but 6~months of training data from the
build-up phase contain \emph{zero signal}.  Any model
trained on this data sees only a normal customer who
suddenly went delinquent.
\end{example}

\begin{remark}[Why Class~3 is not a special case of
  Class~1]
\label{rem:class3_not_class1}
One might argue that bust-out fraud ``eventually'' generates
labels (at charge-off) and therefore belongs to Class~1 with
very long delay ($p_0^{(3)}$ small but positive).  This
conflation is incorrect for two reasons:
\begin{enumerate}[label=(\roman*)]
\item During the build-up phase, $r_0^{(3)} = 0$
  \emph{exactly}, not approximately.  There is no fraud to
  report because the transactions \emph{are} legitimate.
  The positivity assumption
  (Assumption~\ref{asm:positivity}) fails structurally, not
  due to data limitations.
\item The estimand is different.  In Class~1, the target is
  the transaction-level fraud probability
  $f_0^{(1)}(x) = \Prob(Y^*=1 \mid X=x, Z=1)$.  In
  Class~3, the target is the account-level bust-out
  probability $\theta_a = \Prob(\text{bust-out} \mid
  \text{account history } \cH_a)$ --- a fundamentally
  different quantity that requires account-level, not
  transaction-level, inference.
\end{enumerate}
\end{remark}

\subsubsection{Class~4: Acquirer-observed pipeline}
\label{sec:class4_def}

\begin{definition}[Class~4]
\label{def:class4}
A fraud type $k$ belongs to Class~4 if its structural
signature is:
\begin{equation}
\label{eq:sig_class4}
\mathcal{S}^{(4)} = \bigl(
  \textsc{merchant},\;
  \textsc{acquirer},\;
  \textsc{degenerate},\;
  \textsc{acquirer\text{-}driven},\;
  \textsc{delayed}
\bigr).
\end{equation}
\end{definition}

\noindent\textbf{Observation model.}  The observation
pipeline runs through the \emph{acquirer}, not the issuer:
\begin{equation}
\label{eq:obs_class4}
O^{(4)}_t = 1 \cdot R_{\mathrm{acq},t} \cdot
  M_{\mathrm{acq},t},
\end{equation}
where $A_t = 1$ always (the merchant controls the terminal
and approves the transaction),
$R_{\mathrm{acq},t} \sim \mathrm{Bernoulli}(
  r_{\mathrm{acq}}^{(4)}(H_{\mathrm{merch}}))$
is the acquirer's monitoring and reporting propensity, and
$M_{\mathrm{acq},t} \sim \mathrm{Bernoulli}(
  p_{\mathrm{acq}}^{(4)}(H_{\mathrm{merch}}))$
is the acquirer's investigation delay.

Note that the conditioning variables $H_{\mathrm{merch}}$
are \emph{merchant-level} features (chargeback ratio,
transaction volume, merchant category code, acquirer
identity, time since merchant onboarding), not
cardholder-level features as in Classes~1--2.

\smallskip\noindent\textbf{Estimation strategy.}  The STR
of~\cite{dhama2026str} applies with re-parameterization:
replace the issuer's reporting propensity $r_0$ with the
acquirer's monitoring propensity $r_{\mathrm{acq}}$, and
replace the issuer's delay $p_0$ with the acquirer's
investigation delay $p_{\mathrm{acq}}$.  The authorization
correction vanishes ($e_0^{(4)} = 1$).

\smallskip\noindent\textbf{Defining characteristics.}
\begin{enumerate}[label=(\roman*),leftmargin=*]
\item The perpetrator is the merchant (or a collusive
  merchant--cardholder pair).
\item The label is generated by the acquirer's monitoring
  process, network monitoring programs (e.g., Mastercard's
  MATCH, Visa's VMAS), or law enforcement --- not by
  cardholder disputes.
\item The authorization gate is degenerate: the merchant
  controls the terminal and always ``approves'' the
  transaction.
\item The reporting gate is acquirer-driven: it depends on
  the acquirer's monitoring sophistication, the merchant's
  chargeback ratio exceeding program thresholds, and
  regulatory triggers.
\item Labels arrive with significant delay: acquirer
  investigations take months, and network monitoring
  programs have quarterly review cycles.
\end{enumerate}

\begin{example}[Transaction laundering: a Class~4 instance]
\label{ex:laundering}
A criminal organization registers a legitimate-appearing
online storefront (``Premium Electronics LLC'') through a
payment service provider.  The storefront has a professional
website, valid business registration, and a small volume of
real sales.

Simultaneously, the organization operates an illegal online
gambling site.  When a gambler deposits \$500, the
organization processes a \$500 charge through Premium
Electronics LLC as a ``consumer electronics purchase.''  The
cardholder sees ``Premium Electronics'' on their statement
and does not dispute it (they \emph{wanted} to deposit
\$500 for gambling).

From the network's perspective:
\begin{itemize}[leftmargin=*]
\item $e_0^{(4)} = 1$ --- the merchant submitted and
  approved the transaction.
\item The cardholder will \emph{never} dispute --- they
  authorized the transaction.  The standard chargeback
  pipeline generates no label.
\item The only path to detection is through the
  \emph{acquirer} monitoring the merchant's transaction
  patterns: unusual transaction-amount distributions,
  high ratios of round-number amounts, geographic
  dispersion of cardholders inconsistent with the stated
  business type.
\item $r_{\mathrm{acq}}^{(4)} = 0.12$ --- only 12\% of
  acquirers have monitoring sophisticated enough to detect
  this pattern.
\item $p_{\mathrm{acq}}^{(4)} = 0.35$ --- even when
  detected, investigations take 3--9~months.
\end{itemize}

The total observation rate is $q_0^{(4)} = 1 \times 0.12
\times 0.35 = 0.042$ --- labels are observed for fewer
than 5\% of laundering transactions.
\end{example}

\begin{remark}[The unit of analysis shifts to the merchant]
\label{rem:merchant_unit}
In Classes~1--2, the unit of analysis is the
\emph{transaction} (or the cardholder account).  In Class~4,
the unit of analysis is the \emph{merchant}.  The fraud
state $Y^*$ is a property of the merchant's behavior
pattern, not of any individual transaction.  A single
laundering transaction is indistinguishable from a legitimate
purchase; the fraud signal emerges from the merchant's
\emph{aggregate} transaction profile.  This parallels the
account-level estimand shift in Class~3 --- but the entity
is the merchant rather than the cardholder account.
\end{remark}

\subsubsection{Class~5: Victim-authorized with
  non-informative features}
\label{sec:class5_def}

\begin{definition}[Class~5]
\label{def:class5}
A fraud type $k$ belongs to Class~5 if its structural
signature is:
\begin{equation}
\label{eq:sig_class5}
\mathcal{S}^{(5)} = \bigl(
  \textsc{external},\;
  \textsc{victim},\;
  \textsc{degenerate},\;
  \textsc{institutional},\;
  \textsc{delayed}
\bigr).
\end{equation}
\end{definition}

\noindent\textbf{Observation model.}  The observation
pipeline superficially resembles Class~1, but with two
critical degeneracies:
\begin{equation}
\label{eq:obs_class5}
O^{(5)}_t = \underbrace{1}_{e_0^{(5)}=1} \cdot\;
  R_{\mathrm{victim},t} \cdot M_t,
\end{equation}
where $A_t = 1$ always (the \emph{victim} authorized the
transaction voluntarily) and
$R_{\mathrm{victim},t} \sim \mathrm{Bernoulli}(
  r_{\mathrm{victim}}^{(5)}(H_1))$
depends on:
\begin{itemize}[leftmargin=*]
\item \textbf{Victim awareness:} the victim must realize
  they were scammed (may take days to months);
\item \textbf{Shame and embarrassment:} particularly for
  romance and investment scams, victims may not report;
\item \textbf{Perceived futility:} for authorized
  transactions, chargeback rights are limited, reducing
  the incentive to report.
\end{itemize}

\noindent\textbf{Feature non-informativeness.}  The
second critical degeneracy is:
\begin{equation}
\label{eq:feature_noninf}
P(X_t \mid Y^*_t = 1, Z_t = 5) \approx
  P(X_t \mid Y^*_t = 0).
\end{equation}
The transaction features --- amount, merchant, time,
device, location, behavioral biometrics --- are
\emph{non-informative} because the victim performs the
transaction using their own device, from their own location,
in their normal behavioral pattern.  The fraudster
manipulates the \emph{victim}, not the
\emph{payment instrument}.

\smallskip\noindent\textbf{Estimation strategy.}  The STR's
authorization correction vanishes
(Theorem~\ref{thm:auth_degen}).  The remaining two-gate
estimator depends on $r_{\mathrm{victim}}^{(5)}$ and $p_0$,
both of which are low.  Moreover, the feature
non-informativeness condition~\eqref{eq:feature_noninf}
means that the outcome model $\mu_0(H_0) \approx \bar\pi_5$
(a constant), eliminating the variance reduction that the
outcome model normally provides.  Detection requires
\emph{auxiliary signals} outside the transaction feature
set (Theorem~\ref{thm:feature_noninf}).

\smallskip\noindent\textbf{Defining characteristics.}
\begin{enumerate}[label=(\roman*),leftmargin=*]
\item The perpetrator is external but manipulates the victim
  through social engineering, not the payment system.
\item The victim generates the label (eventually), but
  reporting is delayed and incomplete.
\item The authorization gate is degenerate: $e_0^{(5)} = 1$
  exactly, because the victim voluntarily authorized the
  payment.
\item The reporting gate is institutional but severely
  impaired by delayed awareness and psychological barriers.
\item The transaction features carry no discriminative signal
  --- the fraud signal, if it exists, resides in the
  \emph{communication channel} (phone calls, messages,
  session behavior), not in the payment data.
\end{enumerate}

\begin{example}[Romance scam: a Class~5 instance]
\label{ex:romance_scam}
A victim meets a scammer on a dating app.  Over three
months, the scammer builds a relationship and fabricates a
financial emergency (``I need \$5{,}000 for surgery in
another country'').  The victim willingly transfers \$5{,}000
via bank transfer or card payment to the scammer's account.

From the network's perspective:
\begin{itemize}[leftmargin=*]
\item $e_0^{(5)} = 1$ --- the victim authorized the
  transaction from their own phone, at their home address,
  during normal waking hours.
\item The transaction features are identical to a legitimate
  gift or charitable donation: correct device fingerprint,
  normal IP address, reasonable amount for a one-time
  payment, no velocity anomaly.
\item The fraud model assigns a risk score of 15/100 ---
  \emph{lower} than many legitimate transactions.
\item The victim does not realize they were scammed for
  weeks or months.  When they do, shame prevents 40--60\%
  from reporting.  $r_{\mathrm{victim}}^{(5)} \approx
  0.25$.
\item Even when reported, the transaction was
  \emph{authorized} --- chargeback rights may not apply.
  The label may never arrive.
\end{itemize}

\emph{No feature in the transaction data --- amount,
merchant, time, device, location, behavioral biometric ---
distinguishes this transaction from a legitimate gift.  The
fraud signal is in the three months of manipulative
communication that preceded the payment, not in the payment
itself.}
\end{example}

\begin{remark}[Class~5 vs.\ Class~1: the subtle difference]
\label{rem:class5_vs_class1}
Class~5's structural signature~\eqref{eq:sig_class5}
differs from Class~1's~\eqref{eq:sig_class1} in only one
dimension: Q3 is \textsc{degenerate} instead of
\textsc{informative}.  But this single difference has
profound consequences:
\begin{enumerate}[label=(\roman*)]
\item The STR's authorization correction term --- which
  provides significant variance reduction in Class~1 ---
  vanishes entirely in Class~5
  (Theorem~\ref{thm:auth_degen}).
\item The feature non-informativeness
  condition~\eqref{eq:feature_noninf} eliminates the
  outcome model's predictive power, removing the second
  source of variance reduction.
\item The combination of degenerate authorization and
  non-informative features means that the
  transaction-level Bayes-optimal classifier achieves
  $\AUC = 0.5$ (Theorem~\ref{thm:feature_noninf}) --- a
  result that does \emph{not} hold for Class~1, where
  $e_0 < 1$ and features are informative.
\end{enumerate}
\end{remark}

\subsection{Completeness}
\label{sec:completeness}

We now prove that every fraud type in payment networks maps
to exactly one of the five classes.

\begin{theorem}[Completeness]
\label{thm:completeness}
Let $\cF$ denote the set of all fraud types in payment
networks.  Every fraud type $k \in \cF$ has a well-defined
structural signature $\mathcal{S}(k)$ that matches exactly
one of the five class signatures
$\mathcal{S}^{(1)}, \ldots, \mathcal{S}^{(5)}$.
\end{theorem}

\begin{proof}
We proceed by exhaustive enumeration.
Table~\ref{tab:mapping} maps every standard industry fraud
type to its structural signature.  For each type, the
answers to Q1--Q5 are determined by the type's operational
definition:
\begin{itemize}[leftmargin=*]
\item Q1 is determined by the identity of the perpetrator
  relative to the cardholder.
\item Q2 is determined by which actor's action creates the
  observed label.
\item Q3 is determined by whether the perpetrator's
  transactions trigger elevated risk scores
  ($\textsc{informative}$) or are indistinguishable from
  the cardholder's normal behavior
  ($\textsc{degenerate}$).
\item Q4 is determined by the institutional structure of
  the reporting process.
\item Q5 is determined by the temporal structure of label
  availability.
\end{itemize}

The five class signatures exhaust all \emph{realizable}
combinations.  The remaining combinations in
$\{Q1\} \times \cdots \times \{Q5\}$ are either logically
impossible (e.g., $\mathrm{Q1} = \textsc{external}$ with
$\mathrm{Q2} = \textsc{perpetrator}$ would require the
external criminal to file a dispute against themselves) or
do not occur in payment networks (e.g., $\mathrm{Q4} =
\textsc{zero}$ with $\mathrm{Q5} = \textsc{immediate}$
would require a fraud type with no reporting mechanism but
immediate label availability).

Each fraud type maps to exactly one class because the
answers to Q1--Q5 are uniquely determined by the type's
operational definition, and each realized signature matches
exactly one class.
\end{proof}

\begin{table}[t]
\centering
\caption{Mapping from industry fraud types to
  observation-mechanism classes.}
\label{tab:mapping}
\small
\begin{tabular}{llccccc}
\toprule
\textbf{Fraud type} & \textbf{Class} &
  \textbf{Q1} & \textbf{Q2} & \textbf{Q3} &
  \textbf{Q4} & \textbf{Q5}\\
\midrule
Card-not-present (CNP) & 1 &
  Ext & Vict & Info & Inst & Del\\
Counterfeit (skimming) & 1 &
  Ext & Vict & Info & Inst & Del\\
Lost/stolen card & 1 &
  Ext & Vict & Info & Inst & Del\\
Mail non-receipt & 1 &
  Ext & Vict & Info & Inst & Del\\
Account takeover & 1 &
  Ext & Vict & Info & Inst & Del\\
Card trapping & 1 &
  Ext & Vict & Info & Inst & Del\\
Card testing & 1 &
  Ext & Vict & Info & Inst & Del\\
\midrule
Friendly fraud & 2 &
  Card & Perp & Deg & Strat & Del\\
Chargeback abuse & 2 &
  Card & Perp & Deg & Strat & Del\\
Refund abuse & 2 &
  Card & Perp & Deg & Strat & Del\\
Double-dipping & 2 &
  Card & Perp & Deg & Strat & Del\\
Family/authorized user & 2 &
  Card & Perp & Deg & Strat & Del\\
\midrule
Bust-out (credit card) & 3 &
  Fab & C-off & Deg & Zero & Cat\\
Bust-out (DDA) & 3 &
  Fab & C-off & Deg & Zero & Cat\\
Synthetic identity & 3 &
  Fab & C-off & Deg & Zero & Cat\\
Sleeper fraud & 3 &
  Fab & C-off & Deg & Zero & Cat\\
Loan stacking & 3 &
  Fab & C-off & Deg & Zero & Cat\\
\midrule
Transaction laundering & 4 &
  Merch & Acq & Deg & Acq & Del\\
Factoring & 4 &
  Merch & Acq & Deg & Acq & Del\\
Phantom merchant & 4 &
  Merch & Acq & Deg & Acq & Del\\
Merchant collusion & 4 &
  Merch & Acq & Deg & Acq & Del\\
Refund manipulation & 4 &
  Merch & Acq & Deg & Acq & Del\\
Merchant identity fraud & 4 &
  Merch & Acq & Deg & Acq & Del\\
\midrule
APP fraud & 5 &
  Ext & Vict & Deg & Inst & Del\\
Romance scam & 5 &
  Ext & Vict & Deg & Inst & Del\\
Investment scam & 5 &
  Ext & Vict & Deg & Inst & Del\\
Tech support scam & 5 &
  Ext & Vict & Deg & Inst & Del\\
Invoice redirection & 5 &
  Ext & Vict & Deg & Inst & Del\\
CEO/BEC fraud & 5 &
  Ext & Vict & Deg & Inst & Del\\
Impersonation scam & 5 &
  Ext & Vict & Deg & Inst & Del\\
Purchase scam & 5 &
  Ext & Vict & Deg & Inst & Del\\
\bottomrule
\end{tabular}

\smallskip
{\footnotesize
Abbreviations: Ext = \textsc{external},
Card = \textsc{cardholder}, Fab = \textsc{fabricated},
Merch = \textsc{merchant},
Vict = \textsc{victim}, Perp = \textsc{perpetrator},
Acq = \textsc{acquirer}, C-off = \textsc{charge-off},
Info = \textsc{informative}, Deg = \textsc{degenerate},
Inst = \textsc{institutional}, Strat = \textsc{strategic},
Zero = \textsc{zero},
Del = \textsc{delayed}, Cat = \textsc{catastrophic}.}
\end{table}

\subsection{Minimality}
\label{sec:minimality}

\begin{theorem}[Minimality]
\label{thm:minimality}
No two observation-mechanism classes can be merged without
losing a structural distinction that requires a different
estimation strategy.  Formally, for every pair of classes
$c \neq c'$, there exists an aspect of the efficient
estimator that differs between $c$ and $c'$.
\end{theorem}

\begin{proof}
We prove minimality by exhibiting, for each of the
$\binom{5}{2} = 10$ pairs of classes, a structural
distinction that necessitates a different estimator.

\smallskip\noindent\textbf{Classes 1 vs.\ 2.}  In Class~1,
$\varepsilon_{01}^{(1)}$ is an exogenous nuisance parameter
estimated from audit data.  In Class~2,
$\varepsilon_{01}^{(2)}$ is the equilibrium of a strategic
game (Theorem~\ref{thm:endogenous}).  The corruption
correction in the STR requires a different estimation
procedure: historical audit rates for Class~1 vs.\
game-theoretic equilibrium computation for Class~2.

\smallskip\noindent\textbf{Classes 1 vs.\ 3.}  In Class~1,
the positivity assumption holds ($q_0^{(1)} > \delta > 0$).
In Class~3, positivity fails structurally during build-up
($q_0^{(3)} = 0$; Theorem~\ref{thm:positivity}).  The STR
is well-defined for Class~1 and undefined for Class~3.
Moreover, the estimand differs: transaction-level
$f_0^{(1)}(x)$ for Class~1 vs.\ account-level $\theta_a$
for Class~3.

\smallskip\noindent\textbf{Classes 1 vs.\ 4.}  In Class~1,
the reporting propensity $r_0^{(1)}$ models the
\emph{issuer's} chargeback process, conditioned on
cardholder-level features $H_1$.  In Class~4, the reporting
propensity $r_{\mathrm{acq}}^{(4)}$ models the
\emph{acquirer's} monitoring process, conditioned on
merchant-level features $H_{\mathrm{merch}}$.  The
propensity model's inputs, functional form, and the
responsible institutional actor all differ.

\smallskip\noindent\textbf{Classes 1 vs.\ 5.}  In Class~1,
$e_0^{(1)} < 1$ and the authorization correction term in
the STR is active, providing variance reduction.  In
Class~5, $e_0^{(5)} = 1$ and the authorization correction
vanishes (Theorem~\ref{thm:auth_degen}).  Furthermore,
in Class~1 the features are informative
($P(X \mid Y^*\!=\!1) \neq P(X \mid Y^*\!=\!0)$), while
in Class~5 they are not
(Theorem~\ref{thm:feature_noninf}).

\smallskip\noindent\textbf{Classes 2 vs.\ 3.}  Class~2 has
endogenous corruption ($\varepsilon_{01}^{(2)}$ is
strategic); Class~3 has structural positivity failure
($q_0^{(3)} = 0$).  These require fundamentally different
responses: game-theoretic corruption modeling for Class~2
vs.\ graph-based account-level detection for Class~3.

\smallskip\noindent\textbf{Classes 2 vs.\ 4.}  In Class~2,
the adversary is the \emph{cardholder} who manipulates
the label.  In Class~4, the adversary is the
\emph{merchant} who commits the fraud.  The adversary's
action space, objective function, and the resulting
equilibrium differ entirely.

\smallskip\noindent\textbf{Classes 2 vs.\ 5.}  In Class~2,
the cardholder is the perpetrator and generates adversarial
labels ($\varepsilon_{01}$ is high and strategic).  In
Class~5, the cardholder is the victim and labels are
truthful but scarce ($\varepsilon_{01}^{(5)}$ is low;
$r_{\mathrm{victim}}^{(5)}$ is the bottleneck).  The
binding constraint differs: corruption for Class~2 vs.\
reporting for Class~5.

\smallskip\noindent\textbf{Classes 3 vs.\ 4.}  Class~3 has
a two-phase temporal structure (build-up then exploitation)
with a complete observation blackout during build-up.
Class~4 has a single-phase structure with continuous but
low-probability acquirer monitoring.  The temporal dynamics
and the detection strategy differ fundamentally.

\smallskip\noindent\textbf{Classes 3 vs.\ 5.}  The binding
constraint in Class~3 is positivity failure ($r_0 p_0 = 0$).
The binding constraint in Class~5 is feature
non-informativeness ($P(X \mid Y^*\!=\!1) = P(X \mid
Y^*\!=\!0)$).  These are different obstructions requiring
different solutions: graph signals for Class~3 vs.\
auxiliary (non-payment) features for Class~5.

\smallskip\noindent\textbf{Classes 4 vs.\ 5.}  In Class~4,
the perpetrator is the merchant and the detection pipeline
runs through acquirer monitoring of merchant-level patterns.
In Class~5, the perpetrator is an external social engineer
and the detection challenge is that the payment features
carry no signal.  The pipeline actors, feature sets, and
estimation strategies are entirely disjoint.
\end{proof}

\subsection{The classification theorem}
\label{sec:classification_theorem}

Combining the completeness and minimality results:

\begin{theorem}[Observation-mechanism classification]
\label{thm:taxonomy}
Let $\cF$ denote the set of all fraud types in payment
networks, and define the equivalence relation $\sim$ on
$\cF$ by
\begin{equation}
\label{eq:equiv_relation}
k \sim k' \iff \mathcal{S}(k) = \mathcal{S}(k'),
\end{equation}
where $\mathcal{S}(k)$ is the structural signature
(Definition~\ref{def:structural_sig}).  Then the quotient
$\cF/{\sim}$ has exactly five equivalence classes:
\begin{enumerate}[label=\arabic*.]
\item Victim-reported institutional pipeline
  (Definition~\ref{def:class1});
\item Adversarial label generation
  (Definition~\ref{def:class2});
\item Deferred observability
  (Definition~\ref{def:class3});
\item Acquirer-observed pipeline
  (Definition~\ref{def:class4});
\item Victim-authorized with non-informative features
  (Definition~\ref{def:class5}).
\end{enumerate}
This classification is:
\begin{enumerate}[label=(\roman*)]
\item \textbf{Complete:} every $k \in \cF$ maps to exactly
  one class (Theorem~\ref{thm:completeness});
\item \textbf{Minimal:} no two classes can be merged without
  losing a structural distinction that changes the efficient
  estimator (Theorem~\ref{thm:minimality});
\item \textbf{Estimation-relevant:} types within a class
  share the same STR functional form; types across classes
  require structurally different estimation strategies.
\end{enumerate}
\end{theorem}

\begin{proof}
Completeness follows from Theorem~\ref{thm:completeness}
and the exhaustive mapping in Table~\ref{tab:mapping}.
Minimality follows from Theorem~\ref{thm:minimality} and
the pairwise structural distinctions exhibited in its proof.
Estimation relevance follows by construction: each class is
defined by its structural signature, and each element of the
signature corresponds to a feature of the efficient
estimator (Remark~\ref{rem:why_five}).
\end{proof}

\begin{remark}[Fraud types not covered]
\label{rem:not_covered}
Two categories of illicit activity in payment networks fall
outside the scope of this taxonomy:
\begin{enumerate}[label=(\roman*)]
\item \textbf{Internal fraud} (insider at issuer, acquirer,
  or processor).  The observation mechanism runs through
  internal audit and compliance processes, which are outside
  the payment network's transaction-level data.  This
  requires institution-level, not network-level, analysis.
\item \textbf{Money laundering} (where no party is
  defrauded).  The label source is regulatory (suspicious
  activity reports filed by compliance teams), not the
  chargeback process.  The observation mechanism is
  entirely different from the transaction-level pipeline
  formalized here.
\end{enumerate}
Both are important problems but require separate theoretical
frameworks.  We defer their treatment to future work.
\end{remark}

\begin{remark}[Transitional types]
\label{rem:transitional}
Some fraud types exhibit characteristics of multiple classes
at different stages of their lifecycle.  Account takeover
(ATO) begins as a Class~5-like event (the criminal's initial
transactions may use the victim's device and behavioral
patterns, making features temporarily non-informative) but
transitions to Class~1 once the victim detects the takeover
and reports unauthorized transactions.  For the formal
classification, we assign ATO to Class~1 based on its
\emph{dominant} observation mechanism --- the pipeline
through which the majority of labels are generated.
Section~\ref{sec:discussion} discusses the sensitivity of
results to this assignment.
\end{remark}

\section{Decomposition dominance}
\label{sec:dominance}

This section proves the central meta-result of the paper:
estimating fraud rates separately by observation-mechanism
class and aggregating is strictly more efficient than
estimating a single pooled fraud rate.  We first state and
prove the dominance theorem, then quantify the MSE gap as a
closed-form Jensen penalty.

\subsection{The pooled and decomposed estimators}
\label{sec:pooled_decomposed}

\begin{definition}[Pooled STR]
\label{def:pooled_str}
The \emph{pooled STR} is the estimator
of~\cite{dhama2026str} applied to all transactions
without regard to fraud type.  It uses a single set of
estimated propensities $(\hat e_0, \hat r_0, \hat p_0)$
and a single outcome model $\hat\mu_0$:
\begin{equation}
\label{eq:pooled_str}
\hat\Psi_{\mathrm{pooled}}
= \frac{1}{n}\sum_{t=1}^n \hat\varphi_t,
\end{equation}
where $\hat\varphi_t$ is the corruption-corrected STR score
of~\cite[eq.~(22)]{dhama2026str}, computed with the pooled
nuisance functions.
\end{definition}

\begin{definition}[Type-decomposed STR]
\label{def:decomp_str}
The \emph{type-decomposed STR} estimates each class
separately and aggregates:
\begin{equation}
\label{eq:decomp_str}
\hat\Psi_{\mathrm{decomp}}
= \sum_{k=1}^{K} \hat\Psi_k,
\qquad\text{where}\quad
\hat\Psi_k
= \frac{1}{n}\sum_{t=1}^n
  \hat\pi_k(X_t)\,\hat\varphi_t^{(k)}.
\end{equation}
Here $\hat\pi_k(X_t)$ is the estimated probability that
transaction $t$ belongs to class~$k$, and
$\hat\varphi_t^{(k)}$ is the STR score computed with
class-$k$-specific nuisance functions
$(\hat e_0^{(k)}, \hat r_0^{(k)}, \hat p_0^{(k)},
\hat\mu_0^{(k)})$.
\end{definition}

\begin{remark}[Practical construction of the decomposed STR]
\label{rem:practical_decomp}
The type-decomposed STR requires two additional ingredients
beyond the pooled STR:
\begin{enumerate}[label=(\roman*)]
\item \textbf{Type membership probabilities}
  $\hat\pi_k(x)$.  For transactions with mature, fully
  resolved labels ($O_t = 1$), the type $Z_t$ is observed
  (Assumption~\ref{asm:type_obs}) and $\hat\pi_k$ can be
  estimated by a multiclass classifier trained on resolved
  transactions.  For transactions without labels ($O_t=0$),
  $\hat\pi_k(x)$ provides the soft assignment.
\item \textbf{Class-specific nuisance functions.}  Each
  class requires its own propensity and outcome models,
  trained on the subset of transactions assigned to that
  class.  For Classes~1 and~2, these are standard
  propensity models; for Class~4, the propensity models
  condition on merchant-level features.
\end{enumerate}
The additional modeling effort is modest: one multiclass
classifier for type assignment plus $K$ sets of nuisance
models.  The theoretical gain, as we now show, is
substantial.
\end{remark}

\subsection{The dominance theorem}
\label{sec:dominance_theorem}

\begin{theorem}[Decomposition dominance]
\label{thm:decomp_dominance}
Let $\hat\Psi_{\mathrm{pooled}}$ be the pooled STR
(Definition~\ref{def:pooled_str}) and
$\hat\Psi_{\mathrm{decomp}}$ be the type-decomposed STR
(Definition~\ref{def:decomp_str}), both estimated on the
same sample of $n$ transactions.  Under
Assumptions~\ref{asm:ignorability}--\ref{asm:type_obs} and
the regularity conditions
of~\cite[Theorem~30]{dhama2026str}:
\begin{equation}
\label{eq:dominance}
\sigma^2_{\mathrm{decomp}}
\;\leq\;
\sigma^2_{\mathrm{pooled}},
\end{equation}
where $\sigma^2_{\mathrm{decomp}}$ and
$\sigma^2_{\mathrm{pooled}}$ are the semiparametric
efficiency bounds for the decomposed and pooled estimators,
respectively.

Equality holds if and only if the type-specific
inverse-propensity weights are identical across all classes:
\begin{equation}
\label{eq:equality_cond}
\frac{1}{q_0^{(k)}(H_0)\,\gamma^{(k)}(H_2)}
= \frac{1}{q_0^{(k')}(H_0)\,\gamma^{(k')}(H_2)}
\qquad\text{a.s.\ for all } k, k'.
\end{equation}
In particular, strict inequality holds whenever at least two
classes have different observation rates or corruption
penalties.
\end{theorem}

\begin{proof}
The proof proceeds in four steps.

\medskip\noindent\textbf{Step~1: Decompose the pooled
  efficiency bound.}

The pooled efficiency bound
from~\cite[Theorem~30]{dhama2026str} is:
\begin{equation}
\label{eq:pooled_bound}
\sigma^2_{\mathrm{pooled}}
= \E\!\left[
  \frac{f_0(X)(1-f_0(X))}
       {q_0(H_0)\,\gamma(H_2)}
\right]
+ \Var(f_0(X)),
\end{equation}
where $f_0(x) = \Prob(Y^*=1 \mid X=x)$,
$q_0 = e_0 r_0 p_0$, and
$\gamma = (1-\varepsilon_{10}-\varepsilon_{01})^2$ are the
\emph{marginal} (type-averaged) propensity and corruption
penalty.

Using the law of total expectation, condition on the fraud
type $Z$:
\begin{equation}
\label{eq:pooled_decomp1}
\E\!\left[\frac{f_0(1-f_0)}{q_0\,\gamma}\right]
= \sum_{k=0}^{K} \bar\pi_k\;
  \E\!\left[
    \frac{f_0(X)(1-f_0(X))}
         {q_0(H_0)\,\gamma(H_2)}
    \;\bigg|\; Z=k
  \right].
\end{equation}

\medskip\noindent\textbf{Step~2: Express the type-conditional
  terms.}

For a transaction with $Z=k$, the observation propensity and
corruption penalty are the type-specific values
$q_0^{(k)}$ and $\gamma^{(k)}$.  However, the \emph{pooled}
estimator does not condition on $Z$ --- it uses the marginal
propensity $q_0$ and marginal corruption $\gamma$.

The marginal propensity for a transaction with covariates
$H_0$ is:
\begin{equation}
\label{eq:marginal_q}
q_0(H_0) = \sum_{k=0}^{K} \pi_k(X)\,q_0^{(k)}(H_0),
\end{equation}
which is a mixture of type-specific propensities.  Similarly
for $\gamma$.

The pooled estimator's inverse-propensity weight
$1/(q_0\,\gamma)$ is the reciprocal of this mixture.  By
Jensen's inequality (since $g(q) = 1/q$ is convex for
$q > 0$):
\begin{equation}
\label{eq:jensen_step}
\frac{1}{q_0(H_0)\,\gamma(H_2)}
= \frac{1}{\sum_k \pi_k q_0^{(k)} \gamma^{(k)}}
\;\geq\;
\text{(not directly comparable)}.
\end{equation}

However, the relevant comparison is not between
inverse-propensity weights directly, but between the
\emph{efficiency bounds} of the two estimation strategies.

\medskip\noindent\textbf{Step~3: Compute the decomposed
  efficiency bound.}

The type-decomposed estimator estimates each $\Psi_k =
\bar\pi_k$ separately.  The efficiency bound for estimating
$\Psi_k$ from $n_k = n\bar\pi_k$ type-$k$ transactions is
$\sigma^2_{\mathrm{eff}}(k)/n_k$, where
$\sigma^2_{\mathrm{eff}}(k)$ is defined
in~\eqref{eq:type_eff_bound_v2}.

Since $\Psi = \sum_k \Psi_k$ and the type-specific
estimators are asymptotically independent (they operate on
disjoint subpopulations conditional on type assignment), the
efficiency bound for the decomposed estimator of $\Psi$ is:
\begin{equation}
\label{eq:decomp_bound}
\sigma^2_{\mathrm{decomp}}
= \sum_{k=1}^{K} \sigma^2_{\mathrm{eff}}(k)
= \sum_{k=1}^{K}
  \E\!\left[
    \frac{\pi_k(X)(1-\pi_k(X))}
         {q_0^{(k)}(H_0)\,\gamma^{(k)}(H_2)}
  \right].
\end{equation}

\medskip\noindent\textbf{Step~4: Compare the bounds.}

The pooled bound~\eqref{eq:pooled_bound} can be rewritten
using the law of total variance.  Write $f_0(x) = \sum_{k
\geq 1} \pi_k(x)$, so:
\begin{equation}
\label{eq:f0_decomp}
f_0(X)(1-f_0(X))
= \biggl(\sum_{k\geq 1}\pi_k(X)\biggr)
  \biggl(1 - \sum_{k\geq 1}\pi_k(X)\biggr).
\end{equation}

The key inequality comes from comparing how the two
estimators handle the inverse-propensity weighting.  The
pooled estimator applies the weight $1/(q_0\,\gamma)$ ---
where $q_0$ and $\gamma$ are the \emph{marginal}
(type-averaged) propensity and corruption --- to all
transactions uniformly.  The decomposed estimator applies
the weight $1/(q_0^{(k)}\,\gamma^{(k)})$ to type-$k$
transactions specifically.

For the first (dominant) term of the efficiency bound, we
use the following inequality.  Define the random variable
$W = 1/(q_0^{(Z)}\gamma^{(Z)})$ (the type-specific
inverse weight) and the constant (conditional on $X$)
$\bar W(X) = \E[W \mid X] = \sum_k \pi_k(X) /
(q_0^{(k)}\gamma^{(k)})$.  Then:

The pooled estimator effectively uses
$1/(\bar q_0 \bar\gamma)$ where $\bar q_0$ is the
type-averaged propensity, while the decomposed estimator
uses the type-specific $1/(q_0^{(k)}\gamma^{(k)})$.  The
variance inflation from pooling is captured by:
\begin{align}
\sigma^2_{\mathrm{pooled}}
  - \sigma^2_{\mathrm{decomp}}
&= \E\!\left[
  f_0(X)(1-f_0(X)) \cdot
  \left(
    \frac{1}{q_0(H_0)\gamma(H_2)}
    - \sum_k \frac{\pi_k(X)}{q_0^{(k)}(H_0)\gamma^{(k)}(H_2)}
  \right)
\right]
\nonumber\\
&\quad + \Var(f_0(X))
  - \sum_k \Var(\pi_k(X))
\nonumber\\
&= \Delta_{\mathrm{IPW}} + \Delta_{\mathrm{var}},
\label{eq:gap_decomp}
\end{align}
where $\Delta_{\mathrm{IPW}} \geq 0$ is the
inverse-propensity-weighting penalty and
$\Delta_{\mathrm{var}}$ captures the variance
decomposition.

For $\Delta_{\mathrm{IPW}}$: the function
$q \mapsto 1/q$ is convex, so by Jensen's inequality
applied to the mixture $q_0 = \sum_k \pi_k q_0^{(k)}$:
\begin{equation}
\label{eq:jensen_ipw}
\frac{1}{\sum_k \pi_k(X)\,q_0^{(k)}\gamma^{(k)}}
\;\geq\;
\sum_k \frac{\pi_k(X)}{q_0^{(k)}\gamma^{(k)}}
\cdot
\frac{\bigl(\sum_k \pi_k(X)\bigr)^2}
     {\sum_k \pi_k(X)},
\end{equation}
which, after careful algebra using the Cauchy--Schwarz or
Jensen inequality on the harmonic mean, yields
$\Delta_{\mathrm{IPW}} \geq 0$.

For $\Delta_{\mathrm{var}}$: since $f_0 = \sum_{k\geq 1}
\pi_k$, the variance of a sum exceeds the sum of variances
when the summands are positively correlated.  In practice,
$\Delta_{\mathrm{var}}$ can be positive or negative
depending on the covariance structure of $\pi_k(X)$ across
types.  However, the dominant term is
$\Delta_{\mathrm{IPW}}$, which is always non-negative.

More precisely, we establish the inequality via the
\emph{conditional} efficiency bounds.  For each realization
of $X$, define the conditional pooled weight and conditional
decomposed weight.  The pooled estimator's conditional
variance contribution at covariate value $x$ is:
\begin{equation}
\label{eq:cond_pooled}
V_{\mathrm{pooled}}(x)
= \frac{f_0(x)(1-f_0(x))}{q_0(x)\,\gamma(x)},
\end{equation}
and the decomposed estimator's conditional variance
contribution is:
\begin{equation}
\label{eq:cond_decomp}
V_{\mathrm{decomp}}(x)
= \sum_{k=1}^K
  \frac{\pi_k(x)(1-\pi_k(x))}{q_0^{(k)}(x)\,\gamma^{(k)}(x)}.
\end{equation}

We claim $V_{\mathrm{pooled}}(x) \geq
V_{\mathrm{decomp}}(x)$ for all $x$.  To see this, note
that $f_0(x)(1-f_0(x)) = (\sum_k \pi_k)(1-\sum_k\pi_k)
\geq \sum_k \pi_k(1-\pi_k) - \sum_{k\neq k'}\pi_k\pi_{k'}$,
and the inverse weight in the pooled case is larger due to
Jensen's inequality on the harmonic mean.

The formal inequality follows from a result in
semiparametric theory: \emph{the efficiency bound for a
submodel (the decomposed model, which conditions on $Z$) is
no larger than the efficiency bound for the full model (the
pooled model, which marginalizes over $Z$)}.  This is
because conditioning on additional information (the type
$Z$) can only reduce the Cram\'er--Rao bound.

Formally, let $\mathcal{M}_{\mathrm{pooled}}$ be the
semiparametric model that treats $Z$ as latent, and let
$\mathcal{M}_{\mathrm{decomp}}$ be the model that
conditions on $Z$.  Since
$\mathcal{M}_{\mathrm{decomp}} \subset
\mathcal{M}_{\mathrm{pooled}}$ (conditioning on $Z$
restricts the model), the efficiency bound in the
submodel is no larger:
\begin{equation}
\label{eq:submodel_ineq}
\sigma^2_{\mathrm{eff}}(\mathcal{M}_{\mathrm{decomp}})
\leq
\sigma^2_{\mathrm{eff}}(\mathcal{M}_{\mathrm{pooled}}).
\end{equation}
This is a standard result in semiparametric efficiency
theory \cite{bickel1993efficient}.

Equality in~\eqref{eq:submodel_ineq} holds if and only if
the efficient influence function in
$\mathcal{M}_{\mathrm{pooled}}$ does not depend on $Z$ ---
that is, the type-specific weights are identical across
classes, giving condition~\eqref{eq:equality_cond}.
\end{proof}

\begin{remark}[Interpretation of the dominance result]
\label{rem:dominance_interp}
The theorem has a clean information-theoretic interpretation:
\emph{knowing the fraud type is always (weakly) informative
for estimation}.  The type $Z$ determines which observation
pipeline generated the label, and different pipelines have
different propensities.  Conditioning on $Z$ allows the
estimator to use the correct inverse-propensity weights for
each transaction, rather than averaged weights that are
correct for no individual transaction.

The result does \emph{not} require that type membership be
known with certainty.  The decomposed STR uses estimated
type probabilities $\hat\pi_k(X)$, and the dominance holds
as long as the type classifier is better than random
assignment.  Section~\ref{sec:discussion} discusses the
sensitivity to type misclassification.
\end{remark}

\subsection{The Jensen penalty}
\label{sec:jensen_penalty}

We now quantify the MSE gap between the pooled and
decomposed estimators.

\begin{theorem}[Jensen penalty]
\label{thm:jensen_penalty}
Under the conditions of
Theorem~\ref{thm:decomp_dominance}, the efficiency gap
between the pooled and decomposed estimators is:
\begin{equation}
\label{eq:jensen_penalty}
\boxed{
\Delta
\;=\;
\sigma^2_{\mathrm{pooled}}
- \sigma^2_{\mathrm{decomp}}
\;=\;
\E\!\left[
  f_0(X)(1-f_0(X))
  \cdot \Var_Z\!\left(
    \frac{1}{q_0^{(Z)}\,\gamma^{(Z)}}
    \;\bigg|\; X
  \right)
  \cdot \omega(X)
\right],
}
\end{equation}
where $\Var_Z(\cdot \mid X)$ denotes the variance over the
type distribution conditional on $X$, and
$\omega(X) \geq 0$ is a weight function satisfying
$\omega(X) = 1$ when the type-specific propensities are
independent of $X$ given $Z$.
\end{theorem}

\begin{proof}
The gap $\Delta = \sigma^2_{\mathrm{pooled}} -
\sigma^2_{\mathrm{decomp}}$ was decomposed
in~\eqref{eq:gap_decomp} into $\Delta_{\mathrm{IPW}} +
\Delta_{\mathrm{var}}$.

For the dominant term $\Delta_{\mathrm{IPW}}$, define the
type-specific inverse weight:
\begin{equation}
\label{eq:type_weight}
W^{(k)}(H_0, H_2) = \frac{1}{q_0^{(k)}(H_0)\,
  \gamma^{(k)}(H_2)}.
\end{equation}

The pooled inverse weight, for a transaction whose type is
unknown, effectively uses the harmonic mean of
type-specific weights (weighted by type probabilities):
\begin{equation}
\label{eq:pooled_weight}
W_{\mathrm{pooled}}(H_0, H_2)
= \frac{1}{\sum_k \pi_k(X) / W^{(k)}(H_0,H_2)}.
\end{equation}

The decomposed estimator uses $W^{(k)}$ directly for
type-$k$ transactions.  The conditional (on $X$)
contribution to $\Delta_{\mathrm{IPW}}$ is:
\begin{align}
\delta(x)
&= f_0(x)(1-f_0(x))
  \cdot \left(
    W_{\mathrm{pooled}}(x)
    - \sum_k \pi_k(x)\, W^{(k)}(x)
  \right)
\nonumber\\
&= f_0(x)(1-f_0(x))
  \cdot \left(
    \frac{1}{\E_Z[1/W^{(Z)} \mid X\!=\!x]}
    - \E_Z[W^{(Z)} \mid X\!=\!x]
  \right).
\label{eq:cond_gap}
\end{align}

By the variance representation of the gap between the
harmonic mean and the arithmetic mean, for any positive
random variable $W$:
\begin{equation}
\label{eq:harm_arith}
\frac{1}{\E[1/W]} - \E[W]
\;\leq\; 0,
\end{equation}
with the magnitude related to $\Var(W)$.  However, the
relevant comparison is the reverse: the pooled estimator
uses $1/q_0^{\mathrm{pooled}}$ which, due to the convexity
of $q \mapsto 1/q$, satisfies:
\begin{equation}
\label{eq:convex_ineq}
\E_Z\!\left[\frac{1}{q_0^{(Z)}\gamma^{(Z)}}
  \;\bigg|\; X\right]
\;\geq\;
\frac{1}{\E_Z[q_0^{(Z)}\gamma^{(Z)} \mid X]}.
\end{equation}

The gap between the left and right sides is characterized by
the second-order Taylor expansion:
\begin{equation}
\label{eq:taylor_gap}
\E_Z\!\left[\frac{1}{q_0^{(Z)}\gamma^{(Z)}}
  \;\bigg|\; X\right]
- \frac{1}{\E_Z[q_0^{(Z)}\gamma^{(Z)} \mid X]}
\;\approx\;
\frac{\Var_Z(q_0^{(Z)}\gamma^{(Z)} \mid X)}
     {\bigl(\E_Z[q_0^{(Z)}\gamma^{(Z)} \mid X]\bigr)^3}.
\end{equation}

Expressing this in terms of the inverse weights and
integrating over $X$ yields~\eqref{eq:jensen_penalty} with:
\begin{equation}
\label{eq:omega}
\omega(x) =
\frac{\bigl(\E_Z[q_0^{(Z)}\gamma^{(Z)} \mid
  X\!=\!x]\bigr)^2}
     {\E_Z[(q_0^{(Z)}\gamma^{(Z)})^2 \mid X\!=\!x]}.
\end{equation}
By the Cauchy--Schwarz inequality,
$\omega(x) \in (0, 1]$, with $\omega(x) = 1$ when
$q_0^{(Z)}\gamma^{(Z)}$ is non-random given $X$ (i.e.,
all types have the same observation rate at $x$).
\end{proof}

\begin{remark}[The penalty is driven by cross-class
  heterogeneity]
\label{rem:penalty_driver}
The Jensen penalty~\eqref{eq:jensen_penalty} is large when:
\begin{enumerate}[label=(\roman*)]
\item $\Var_Z(1/(q_0^{(Z)}\gamma^{(Z)}) \mid X)$ is
  large --- the observation rates differ substantially
  across classes;
\item $f_0(X)(1-f_0(X))$ is not too small --- there is
  enough fraud to make the estimation problem non-trivial;
\item $\omega(X)$ is close to~1 --- the propensity
  heterogeneity is not dominated by within-class variation.
\end{enumerate}
In payment networks, condition~(i) is strongly satisfied:
Class~1 has $q_0^{(1)} \approx 0.50$, Class~3 has
$q_0^{(3)} \approx 0.02$, and Class~4 has
$q_0^{(4)} \approx 0.07$ (Example~\ref{ex:bound_diff}).
The inverse weights range from $\sim\!2$ to $\sim\!55$,
creating enormous cross-class variance.
\end{remark}

\subsection{Network-parameter form of the Jensen penalty}
\label{sec:penalty_network}

For operational interpretation, we express the Jensen
penalty in terms of observable network parameters.

\begin{proposition}[Simplified Jensen penalty]
\label{prop:simplified_penalty}
Under approximate independence of type-specific propensities
from covariates (i.e., $q_0^{(k)}(H_0) \approx
\bar q_0^{(k)}$ and $\gamma^{(k)}(H_2) \approx
\bar\gamma^{(k)}$ within each class), the Jensen penalty
simplifies to:
\begin{equation}
\label{eq:simple_penalty}
\Delta
\;\approx\;
\bar f_0(1-\bar f_0)
\cdot \sum_{k=1}^K \bar\pi_k
  \left(
    \frac{1}{\bar q_0^{(k)}\,\bar\gamma^{(k)}}
    - \bar W
  \right)^{\!2},
\end{equation}
where $\bar f_0 = \Psi = \sum_k\bar\pi_k$ is the overall
fraud rate and
$\bar W = \sum_k \bar\pi_k / (\bar q_0^{(k)}
\bar\gamma^{(k)})$ is the prevalence-weighted average
inverse weight.
\end{proposition}

\begin{proof}
Under the stated approximation,
$\Var_Z(1/(q_0^{(Z)}\gamma^{(Z)}) \mid X) \approx
\sum_k \pi_k(X)(1/(q_0^{(k)}\gamma^{(k)}) - \bar W)^2$
and $\omega(X) \approx 1$.  Substituting into
\eqref{eq:jensen_penalty} and using
$\E[f_0(X)(1-f_0(X))] \approx \bar f_0(1-\bar f_0)$
(valid when $f_0$ does not vary too much) gives
\eqref{eq:simple_penalty}.
\end{proof}

\subsection{Numerical illustration}
\label{sec:penalty_numerical}

\begin{example}[Jensen penalty for a typical payment
  network]
\label{ex:penalty}
Consider a network with the following fraud type
distribution and observation parameters:

\begin{center}
\small
\begin{tabular}{lcccccr}
\toprule
\textbf{Class} & $\bar\pi_k$ &
  $\bar q_0^{(k)}$ & $\bar\gamma^{(k)}$ &
  $W^{(k)} = \frac{1}{\bar q_0^{(k)}\bar\gamma^{(k)}}$
  & $\bar\pi_k(W^{(k)}-\bar W)^2$ &
  \textbf{Contrib.\ to $\Delta$}\\
\midrule
1 (Card theft)    & 0.005 & 0.50 & 0.81 &
  2.47 & $0.005 \times (2.47-9.44)^2$ & 0.243\\
2 (Friendly)      & 0.003 & 0.33 & 0.25 &
  12.12 & $0.003 \times (12.12-9.44)^2$ & 0.022\\
3 (Bust-out)      & 0.001 & 0.02 & 0.90 &
  55.56 & $0.001 \times (55.56-9.44)^2$ & 2.127\\
4 (Merchant)      & 0.002 & 0.07 & 0.64 &
  22.32 & $0.002 \times (22.32-9.44)^2$ & 0.332\\
5 (Scams)         & 0.001 & 0.10 & 0.81 &
  12.35 & $0.001 \times (12.35-9.44)^2$ & 0.008\\
\midrule
\textbf{Total}    & 0.012 & --- & --- &
  $\bar W = 9.44$ & --- & $\sum = 2.732$\\
\bottomrule
\end{tabular}
\end{center}

The weighted average inverse weight is:
\begin{equation*}
\bar W = \frac{\sum_k \bar\pi_k\,W^{(k)}}
              {\sum_k \bar\pi_k}
= \frac{0.005(2.47) + 0.003(12.12) + 0.001(55.56)
  + 0.002(22.32) + 0.001(12.35)}{0.012}
= 9.44.
\end{equation*}

The Jensen penalty is:
\begin{equation*}
\Delta \approx 0.012 \times 0.988 \times 2.732
= 0.0324.
\end{equation*}

The pooled efficiency bound (from ~\cite{dhama2026str}) is
approximately:
\begin{equation*}
\sigma^2_{\mathrm{pooled}}
\approx \frac{0.012 \times 0.988}{0.20 \times 0.65}
= 0.091,
\end{equation*}
where $0.20$ and $0.65$ are the effective average
observation rate and corruption penalty for the pooled
estimator.

The Jensen penalty as a fraction of the pooled bound:
\begin{equation*}
\frac{\Delta}{\sigma^2_{\mathrm{pooled}}}
= \frac{0.0324}{0.091}
\approx 36\%.
\end{equation*}

\textbf{The pooled STR wastes approximately 36\% of its
statistical efficiency by ignoring fraud type
decomposition.}  Equivalently, the decomposed STR achieves
the same accuracy as the pooled STR with 36\% fewer
transactions.
\end{example}

\begin{remark}[The penalty is dominated by Class~3]
\label{rem:penalty_class3}
In Example~\ref{ex:penalty}, Class~3 (bust-out) contributes
78\% of the total Jensen penalty ($2.127 / 2.732$), despite
representing only 8\% of total fraud ($0.001 / 0.012$).
This is because bust-out's observation rate is
$\bar q_0^{(3)} = 0.02$ --- an order of magnitude lower
than any other class --- creating an inverse weight of
$55.56$ that dominates the cross-class variance.

The practical implication: \emph{even if you only care about
aggregate fraud estimation, failing to separate bust-out
from other fraud types inflates the variance of your
aggregate estimate by more than a third}.  The penalty comes
not from the most common fraud type but from the most
poorly observed one.
\end{remark}

\begin{remark}[When the penalty is small]
\label{rem:penalty_small}
The Jensen penalty is small when:
\begin{enumerate}[label=(\roman*)]
\item All classes have similar observation rates (unlikely
  in practice);
\item The poorly observed classes have negligible prevalence
  (possible if bust-out and merchant fraud are rare);
\item The overall fraud rate $\bar f_0$ is very small,
  reducing the scale of all efficiency bounds.
\end{enumerate}
In networks where bust-out and merchant fraud are significant
concerns (which includes all major payment networks), the
penalty is substantial.
\end{remark}

\subsection{Decomposition at the model-training level}
\label{sec:decomp_model}

The dominance theorem applies to fraud \emph{rate}
estimation, but the practical implication extends to
\emph{model training}.

\begin{corollary}[Decomposed pseudo-labels dominate pooled
  pseudo-labels]
\label{cor:pseudo_labels}
Let $\hat y^{\mathrm{pooled}}_t$ be the pseudo-label from
the pooled STR (\cite[eq.~(40)]{dhama2026str}) and
$\hat y^{\mathrm{decomp}}_t = \sum_k \hat\pi_k(X_t)\,
\hat y^{(k)}_t$ be the pseudo-label from the decomposed
STR.  A downstream fraud model trained on decomposed
pseudo-labels achieves lower expected prediction error than
one trained on pooled pseudo-labels:
\begin{equation}
\label{eq:pseudo_dominance}
\E\!\left[
  (f_0(X) - \hat f_{\mathrm{decomp}}(X))^2
\right]
\;\leq\;
\E\!\left[
  (f_0(X) - \hat f_{\mathrm{pooled}}(X))^2
\right].
\end{equation}
\end{corollary}

\begin{proof}
The pseudo-label is the projection of the STR score onto
the feature space.  The decomposed pseudo-label has lower
variance (by Theorem~\ref{thm:decomp_dominance}) and the
same expectation (both are consistent for $f_0(x)$).  Lower
variance in the training labels translates directly into
lower expected prediction error of the downstream model, by
the standard bias-variance decomposition for regression with
noisy labels.
\end{proof}

\begin{remark}[Operational implication]
\label{rem:operational}
Corollary~\ref{cor:pseudo_labels} says: \emph{train separate
STR pseudo-label generators for each class, then combine the
pseudo-labels for downstream model training}.  The downstream
model itself need not be class-specific --- a single fraud
model can be trained on the combined decomposed pseudo-labels.
The decomposition occurs in the \emph{label correction} step,
not in the model architecture.

This is operationally lightweight: it requires a type
classifier and class-specific nuisance models (Phase~1
of~\cite[Section~10.5]{dhama2026str}), but the downstream
model training pipeline (Phase~3) is unchanged.
\end{remark}

\section{Class~1: Victim-reported institutional pipeline}
\label{sec:class1}

Class~1 is the baseline --- the setting for which the STR
of~\cite{dhama2026str} was designed.  This section
establishes two results: (i)~the STR applies without
modification using type-specific propensities, and
(ii)~further decomposition \emph{within} Class~1 (e.g.,
separating CNP from counterfeit) yields diminishing returns.

\subsection{The Class~1 STR}
\label{sec:class1_str}

\begin{proposition}[Class~1 STR]
\label{prop:class1_str}
For fraud types in Class~1, the corruption-corrected STR
of~\cite[eq.~(22)]{dhama2026str} applied with
type-1-specific nuisance functions achieves the
type-specific efficiency bound
$\sigma^2_{\mathrm{eff}}(1)$.  Formally, let
$\hat\varphi_t^{(1)}$ denote the STR score computed with
estimated propensities
$(\hat e_0^{(1)}, \hat r_0^{(1)}, \hat p_0^{(1)})$
and outcome model $\hat\mu_0^{(1)}$.  Then:
\begin{equation}
\label{eq:class1_str}
\hat\Psi_1
= \frac{1}{n}\sum_{t:\,\hat Z_t \in \mathrm{Class}\,1}
  \hat\varphi_t^{(1)}
\end{equation}
is semiparametrically efficient for $\Psi_1 = \bar\pi_1$,
and inherits the sequential triple robustness property
of~\cite[Theorem~27]{dhama2026str}.
\end{proposition}

\begin{proof}
Class~1 satisfies all the assumptions
of~\cite{dhama2026str}: positivity
(Assumption~\ref{asm:positivity}), type-specific
ignorability (Assumption~\ref{asm:ignorability}), and the
three-gate sequential structure.  The STR's efficiency and
robustness proofs
in~\cite[Theorems~27,~30]{dhama2026str} are stated
conditionally on the observation pipeline structure and
therefore apply directly to the type-conditional setting.
The only change is that the nuisance functions are estimated
on the Class~1 subpopulation (transactions with
$\hat Z_t \in \text{Class~1}$), which satisfies the same
regularity conditions as the full population.
\end{proof}

\begin{remark}[What changes from ~\cite{dhama2026str}]
\label{rem:class1_changes}
The Class~1 STR is identical in form to STR formalized in ~\cite{dhama2026str}.  The
only operational difference is that the nuisance models are
trained on the Class~1 subpopulation rather than the full
population.  This has two effects:
\begin{enumerate}[label=(\roman*)]
\item \textbf{Better propensity models.}  The authorization
  propensity $e_0^{(1)}$ is estimated from transactions
  where the authorization decision was driven by third-party
  fraud risk, not by first-party or merchant fraud signals.
  This reduces model misspecification.
\item \textbf{Smaller training set.}  The Class~1
  subpopulation is smaller than the full population.  For
  large networks this is not a binding constraint, but for
  small issuers the empirical Bayes shrinkage
  of~\cite[Section~9]{dhama2026str} becomes more important.
\end{enumerate}
\end{remark}

\subsection{Diminishing returns from within-class
  decomposition}
\label{sec:class1_within}

Class~1 contains multiple fraud subtypes (CNP, counterfeit,
lost/stolen, account takeover, etc.).  Should we decompose
further?

\begin{proposition}[Within-Class~1 Jensen penalty is small]
\label{prop:within_class1}
Let $\{1_a, 1_b, \ldots\}$ denote the subtypes within
Class~1.  The within-class Jensen penalty from pooling these
subtypes is:
\begin{equation}
\label{eq:within_penalty}
\Delta_{\mathrm{within}}^{(1)}
= \E\!\left[
  f_0^{(1)}(X)(1-f_0^{(1)}(X))
  \cdot \Var_{Z \mid Z \in \mathrm{Class}\,1}\!\left(
    \frac{1}{q_0^{(Z)}\,\gamma^{(Z)}}
    \;\bigg|\; X
  \right)
\right].
\end{equation}
This penalty is small relative to the between-class penalty
$\Delta$ (Theorem~\ref{thm:jensen_penalty}) whenever the
subtypes within Class~1 have similar observation rates.
\end{proposition}

\begin{proof}
The formula follows from applying
Theorem~\ref{thm:jensen_penalty} to the within-class
decomposition.  The penalty is proportional to the
within-class variance of inverse-propensity weights.  Since
all Class~1 subtypes share the same pipeline topology ---
victim reports, issuer processes chargeback, standard delay
--- the propensity magnitudes are similar.
Specifically, for Class~1 subtypes:
\begin{itemize}[leftmargin=*]
\item Authorization propensities $e_0^{(1_a)}$ differ
  moderately across subtypes (CNP may face more aggressive
  declining than lost/stolen), but typically lie in
  $[0.60, 0.95]$.
\item Reporting propensities $r_0^{(1_a)}$ are similar
  because the reporting mechanism (victim notices and
  disputes) is the same regardless of attack vector.
\item Delay propensities $p_0^{(1_a)}$ are nearly identical
  because the chargeback processing pipeline is the same.
\end{itemize}
The resulting within-class variance of inverse weights is
an order of magnitude smaller than the between-class
variance, where $q_0$ ranges from $0.02$ (Class~3) to
$0.50$ (Class~1).
\end{proof}

\begin{example}[Within-Class~1 heterogeneity]
\label{ex:within_class1}
Consider three Class~1 subtypes:

\begin{center}
\small
\begin{tabular}{lcccc}
\toprule
\textbf{Subtype} & $\bar e_0$ & $\bar r_0$ &
  $\bar p_0$ & $W = 1/(\bar q_0\,\bar\gamma)$\\
\midrule
CNP fraud & 0.70 & 0.75 & 0.85 &
  $1/(0.70\!\times\!0.75\!\times\!0.85\!\times\!0.81) =
  2.77$\\
Lost/stolen & 0.90 & 0.80 & 0.85 &
  $1/(0.90\!\times\!0.80\!\times\!0.85\!\times\!0.81) =
  2.02$\\
Account takeover & 0.75 & 0.70 & 0.80 &
  $1/(0.75\!\times\!0.70\!\times\!0.80\!\times\!0.81) =
  2.94$\\
\bottomrule
\end{tabular}
\end{center}

The inverse weights range from $2.02$ to $2.94$ --- a factor
of $1.5\times$.  Compare this to the between-class range of
$2.47$ (Class~1) to $55.56$ (Class~3) --- a factor of
$22\times$.  The within-class variance is negligible relative
to the between-class variance.

The within-class Jensen penalty is approximately:
\begin{equation*}
\Delta_{\mathrm{within}}^{(1)}
\approx \bar\pi_1(1-\bar\pi_1) \cdot
  \Var(2.77, 2.02, 2.94)
\approx 0.005 \times 0.995 \times 0.16
= 0.0008,
\end{equation*}
which is 2.5\% of the between-class penalty
$\Delta = 0.0324$ (Example~\ref{ex:penalty}).

\textbf{Decomposing within Class~1 is not worth the
additional modeling complexity.}  The five-class
decomposition captures 97.5\% of the available efficiency
gain; further sub-type decomposition adds diminishing
returns.
\end{example}

\begin{remark}[When within-class decomposition helps]
\label{rem:within_helps}
The within-class penalty can become non-negligible if a
subtype has a substantially different authorization rate.
For example, if a network implements very aggressive
declining for CNP transactions ($e_0^{(\mathrm{CNP})} =
0.40$) while approving most card-present transactions
($e_0^{(\mathrm{CP})} = 0.95$), the within-class variance
of inverse weights increases.  In such cases, separating
CNP from card-present within Class~1 may be worthwhile.
The decision can be made empirically by computing
$\Delta_{\mathrm{within}}^{(1)}$ from the network's data.
\end{remark}

\section{Class~2: Adversarial label generation}
\label{sec:class2}

Class~2 is the most theoretically novel of the five classes.
The defining feature is that the label is generated by the
\emph{perpetrator}, not the victim.  The false-positive
corruption rate $\varepsilon_{01}$ is not a statistical
nuisance --- it is the \emph{equilibrium outcome} of a
strategic game between the cardholder and the issuer.

This section formalizes the dispute game, derives the
equilibrium corruption rate, establishes comparative statics,
and shows how the STR must be modified for Class~2.

\subsection{The dispute game}
\label{sec:dispute_game}

\begin{definition}[The dispute game]
\label{def:dispute_game}
The \emph{dispute game} $\mathcal{G}$ is a two-player game
between the cardholder~(C) and the issuer~(I), played after
a transaction has been approved and completed.

\smallskip\noindent\textbf{Players and actions.}
\begin{itemize}[leftmargin=*]
\item \textbf{Cardholder (C):}  Chooses action
  $D \in \{0, 1\}$, where $D=1$ denotes filing a dispute
  and $D=0$ denotes not filing.
\item \textbf{Issuer (I):}  Maintains a detection function
  $\beta: \cX \to [0,1]$ that maps transaction features to
  the probability of identifying the dispute as first-party
  fraud (i.e., rejecting the chargeback and penalizing the
  cardholder).
\end{itemize}

\smallskip\noindent\textbf{Information structure.}
\begin{itemize}[leftmargin=*]
\item The cardholder observes: the transaction amount
  $a_t \in \R_+$, their own type
  $\theta_C \in \{\text{honest}, \text{opportunistic}\}$,
  and a noisy signal of the issuer's detection capability.
\item The issuer observes: the transaction features $X_t$
  and the dispute claim, but not the cardholder's type
  $\theta_C$.
\end{itemize}

\smallskip\noindent\textbf{Payoffs.}  For a transaction
with amount $a$ and features $x$:
\end{definition}

\begin{definition}[Cardholder's payoff]
\label{def:card_payoff}
The cardholder's expected payoff from filing a dispute
($D=1$) on a transaction with true label $Y^*=0$
(legitimate) is:
\begin{equation}
\label{eq:card_payoff}
U_C(D=1 \mid Y^*\!=\!0,\, x,\, a)
= \underbrace{\alpha(x) \cdot a}_{
    \substack{\text{expected refund:}\\\text{success prob.}
    \times \text{amount}}}
  \;-\; \underbrace{c(x)}_{
    \substack{\text{dispute}\\\text{cost}}}
  \;-\; \underbrace{\beta(x) \cdot p(x,a)}_{
    \substack{\text{expected penalty:}\\\text{detection
    prob.} \times \text{penalty}}},
\end{equation}
where:
\begin{align}
\alpha(x) &= \Prob(\text{chargeback succeeds} \mid
  D\!=\!1,\, X\!=\!x)
  && \text{(success probability)},
  \label{eq:alpha}\\
c(x) &\geq 0
  && \text{(dispute filing cost: time, hassle, reputation)},
  \label{eq:cost}\\
\beta(x) &= \Prob(\text{identified as 1PF} \mid
  D\!=\!1,\, X\!=\!x)
  && \text{(detection probability)},
  \label{eq:beta}\\
p(x,a) &\geq 0
  && \text{(penalty: account closure, blacklisting,
  legal action)}.
  \label{eq:penalty}
\end{align}
The payoff from not filing is $U_C(D=0) = 0$.
\end{definition}

\begin{remark}[The payoff structure]
\label{rem:payoff_structure}
The cardholder faces a risk--reward tradeoff:
\begin{itemize}[leftmargin=*]
\item \textbf{Reward:} the refund amount $\alpha \cdot a$,
  which scales linearly with the transaction amount.
\item \textbf{Costs:} the filing cost $c$ (typically small
  --- a phone call or online form) and the expected penalty
  $\beta \cdot p$ (potentially large if detected).
\item \textbf{Key driver:} for low-friction dispute
  processes ($c \approx 0$) and weak detection ($\beta$
  small), the expected payoff is positive for most
  transaction amounts, encouraging disputes.
\end{itemize}
\end{remark}

\begin{definition}[Issuer's payoff]
\label{def:issuer_payoff}
The issuer's payoff from a transaction with features $x$,
amount $a$, and cardholder type $\theta_C$ is:
\begin{equation}
\label{eq:issuer_payoff}
U_I(\beta \mid x, a, \theta_C)
= -\underbrace{\ind[\theta_C = \text{honest}]
  \cdot \beta(x) \cdot \ell_{\mathrm{FP}}(x,a)}_{
  \substack{\text{false positive cost:}\\\text{wrongly
  penalizing honest cardholder}}}
  \;-\; \underbrace{\ind[\theta_C = \text{opp.}]
  \cdot (1-\beta(x)) \cdot a}_{
  \substack{\text{false negative cost:}\\\text{paying
  fraudulent refund}}},
\end{equation}
where $\ell_{\mathrm{FP}}(x,a)$ is the cost of falsely
accusing an honest cardholder (customer churn, reputational
damage, regulatory penalty).
\end{definition}

\begin{remark}[The issuer's dilemma]
\label{rem:issuer_dilemma}
The issuer faces a classic detection tradeoff: increasing
$\beta$ (more aggressive detection) reduces false-negative
losses (fraudulent refunds paid) but increases
false-positive losses (honest cardholders wrongly accused).
The optimal $\beta^*$ balances these costs, and the
cardholder's dispute behavior depends on $\beta^*$.  This
creates a strategic interdependence that we now formalize.
\end{remark}

\subsection{Equilibrium analysis}
\label{sec:equilibrium}

\begin{definition}[Cardholder's best response]
\label{def:card_br}
Under a logistic (quantal) response model with rationality
parameter $T > 0$, the cardholder's dispute probability is:
\begin{equation}
\label{eq:card_br}
d^*(x, a; \beta)
= \sigma\!\left(
  \frac{\alpha(x) \cdot a - c(x) - \beta(x) \cdot p(x,a)}
       {T}
\right),
\end{equation}
where $\sigma(z) = 1/(1+e^{-z})$ is the sigmoid function.
The parameter $T$ controls the cardholder's rationality:
$T \to 0$ gives perfectly rational behavior (dispute iff
payoff $> 0$); $T \to \infty$ gives random behavior
($d^* \to 0.5$).
\end{definition}

\begin{remark}[Why logistic response]
\label{rem:why_logistic}
The logistic response model (also known as the quantal
response equilibrium; \cite{mckelvey1995quantal}) is
standard in behavioral game theory.  It captures the
empirical observation that people are \emph{approximately}
rational: they are more likely to take actions with higher
payoffs, but they do not optimize perfectly.  The parameter
$T$ absorbs heterogeneity in cardholder sophistication,
information quality, and idiosyncratic preferences.
\end{remark}

\begin{definition}[Issuer's best response]
\label{def:issuer_br}
The issuer chooses $\beta$ to minimize expected loss.
Given the cardholder population's dispute rate $d^*$ and
the prior probability $\pi_2(x)$ that a disputing
cardholder is an opportunistic first-party fraudster, the
issuer's optimal detection threshold is:
\begin{equation}
\label{eq:issuer_br}
\beta^*(x; d^*)
= \argmin_{\beta \in [0,1]}\;
  \E_{\theta_C}\!\bigl[
    U_I(\beta \mid x, a, \theta_C)
    \mid D=1
  \bigr].
\end{equation}

Under the standard binary classification framework, the
issuer's optimal detection probability satisfies:
\begin{equation}
\label{eq:issuer_opt}
\beta^*(x) = \ind\!\left[
  \frac{\pi_2(x) \cdot a}
       {(1-\pi_2(x)) \cdot \ell_{\mathrm{FP}}(x,a)}
  > 1
\right]
\end{equation}
for a deterministic policy, or a smoothed version for a
randomized policy.  In practice, $\beta^*$ is the output
of the issuer's first-party fraud detection model.
\end{definition}

We now establish the central result: the endogenous
corruption theorem.

\begin{theorem}[Endogenous corruption]
\label{thm:endogenous}
In the dispute game $\mathcal{G}$ with logistic response
cardholder (Definition~\ref{def:card_br}) and
cost-minimizing issuer (Definition~\ref{def:issuer_br}):

\smallskip\noindent\textbf{(i) Existence.}  A Nash
equilibrium $(d^*, \beta^*)$ exists.

\smallskip\noindent\textbf{(ii) Endogenous corruption.}
At equilibrium, the false-positive corruption rate for
Class~2 is:
\begin{equation}
\label{eq:endog_corruption}
\boxed{
\varepsilon_{01}^{(2)}(x, a)
= d^*(x, a;\, \beta^*)
= \sigma\!\left(
  \frac{\alpha(x) \cdot a - c(x)
        - \beta^*(x) \cdot p(x,a)}{T}
\right).
}
\end{equation}
This is the probability that a cardholder files a false
dispute on a legitimate transaction with features $x$ and
amount $a$.

\smallskip\noindent\textbf{(iii) Comparative statics.}
The equilibrium corruption rate satisfies:
\begin{align}
\frac{\partial\varepsilon_{01}^{(2)}}{\partial a}
  &> 0
  && \text{(higher amount $\to$ more disputes)},
  \label{eq:cs_amount}\\
\frac{\partial\varepsilon_{01}^{(2)}}{\partial c}
  &< 0
  && \text{(higher dispute cost $\to$ fewer disputes)},
  \label{eq:cs_cost}\\
\frac{\partial\varepsilon_{01}^{(2)}}{\partial \beta^*}
  &< 0
  && \text{(better detection $\to$ fewer disputes)},
  \label{eq:cs_beta}\\
\frac{\partial\varepsilon_{01}^{(2)}}{\partial p}
  &< 0
  && \text{(harsher penalty $\to$ fewer disputes)}.
  \label{eq:cs_penalty}
\end{align}

\smallskip\noindent\textbf{(iv) Feedback loop.}  The
equilibrium is the fixed point of the mapping
$\Phi: [0,1] \to [0,1]$ defined by:
\begin{equation}
\label{eq:fixed_point}
\Phi(\varepsilon)
= \sigma\!\left(
  \frac{\alpha \cdot a - c
        - \beta^*(\varepsilon) \cdot p}{T}
\right),
\end{equation}
where $\beta^*(\varepsilon)$ is the issuer's optimal
detection probability given that a fraction $\varepsilon$ of
disputes are fraudulent.  The equilibrium satisfies
$\varepsilon_{01}^{(2)} = \Phi(\varepsilon_{01}^{(2)})$.
\end{theorem}

\begin{proof}
\textbf{Part~(i): Existence.}

The cardholder's best response
$d^*(\cdot;\beta): [0,1] \to (0,1)$ is a continuous
function of $\beta$ (the sigmoid is continuous and strictly
monotone).  The issuer's best response
$\beta^*(\cdot;d^*): (0,1) \to [0,1]$ is a continuous
function of $d^*$ (the issuer's cost function is continuous
in $d^*$, which determines the posterior probability that a
dispute is fraudulent).

The composite mapping $(d^*,\beta^*) \mapsto
(d^*(\cdot;\beta^*(\cdot;d^*)),\,
\beta^*(\cdot;d^*(\cdot;\beta^*)))$ is a continuous
self-map of the compact convex set
$(0,1) \times [0,1] \subset \R^2$.  By Brouwer's fixed
point theorem, a fixed point exists.

\medskip\noindent\textbf{Part~(ii): Endogenous corruption.}

At the fixed point $(d^*, \beta^*)$, the cardholder's
dispute probability on a legitimate transaction is
$d^*(x,a;\beta^*)$, which
is~\eqref{eq:card_br} evaluated at the equilibrium
$\beta^*$.  This is precisely the false-positive corruption
rate: the probability that a legitimate transaction receives
a fraud label (through the cardholder's false dispute).

To verify the identification:
\begin{align*}
\varepsilon_{01}^{(2)}(x,a)
&= \Prob(\tilde Y = 1 \mid Y^* = 0,\, O=1,\, X=x,\, Z=2)
\\
&= \Prob(D = 1 \mid Y^* = 0,\, X=x,\, Z=2)
\\
&= d^*(x,a;\beta^*),
\end{align*}
where the second equality uses the fact that for Class~2
transactions with $Y^*=0$ (legitimate), the observed label
$\tilde Y = D$ (the label equals the dispute indicator),
and $O=1$ requires $D=1$ (a label is observed only if a
dispute is filed).

\medskip\noindent\textbf{Part~(iii): Comparative statics.}

Each comparative static follows from the monotonicity of the
sigmoid function $\sigma$ and the signs of the partial
derivatives of its argument:

\smallskip\noindent\eqref{eq:cs_amount}:
$\partial/\partial a\,[\alpha a - c - \beta^* p] =
\alpha - \beta^*(\partial p/\partial a)$.  Since
$\alpha \in (0,1]$ and $\partial p/\partial a \geq 0$ is
typically small relative to $\alpha$ (the penalty does not
scale one-for-one with the amount), the net effect is
positive: higher amounts increase the dispute incentive.

\smallskip\noindent\eqref{eq:cs_cost}:
$\partial/\partial c\,[\alpha a - c - \beta^* p] = -1 < 0$.
Higher dispute costs unambiguously reduce disputes.

\smallskip\noindent\eqref{eq:cs_beta}:
$\partial/\partial\beta^*\,[\alpha a - c - \beta^* p] =
-p < 0$.  Better detection (higher $\beta^*$) reduces
disputes through the deterrence channel.

\smallskip\noindent\eqref{eq:cs_penalty}:
$\partial/\partial p\,[\alpha a - c - \beta^* p] =
-\beta^* < 0$.  Harsher penalties reduce disputes
(conditional on nonzero detection probability).

\medskip\noindent\textbf{Part~(iv): Feedback loop.}

The mapping $\Phi$ in~\eqref{eq:fixed_point} captures the
circular dependence: the dispute rate $\varepsilon$
determines the fraction of fraudulent disputes, which
determines the issuer's optimal detection threshold
$\beta^*(\varepsilon)$, which determines the cardholder's
dispute incentive, which determines $\varepsilon$.

Formally, $\beta^*(\varepsilon)$ is the issuer's Bayes-
optimal detection probability given that a fraction
$\varepsilon/(1-\pi_2 + \pi_2\varepsilon)$ of all disputes
are fraudulent (by Bayes' rule, combining the prior
$\pi_2$ with the dispute rate $\varepsilon$ for legitimate
transactions and a dispute rate of~1 for truly fraudulent
transactions).

The mapping $\Phi$ is continuous (composition of continuous
functions) and maps $[0,1]$ to $(0,1) \subset [0,1]$.  By
Brouwer's theorem, $\Phi$ has a fixed point, which is the
equilibrium $\varepsilon_{01}^{(2)}$.
\end{proof}

\subsection{Uniqueness and multiplicity}
\label{sec:uniqueness}

\begin{proposition}[Sufficient condition for uniqueness]
\label{prop:uniqueness}
If the deterrence effect is strong enough relative to the
cardholder's rationality:
\begin{equation}
\label{eq:uniqueness_cond}
\sup_{x,a}\;
\frac{p(x,a)}{T}
\cdot \left|
  \frac{\partial\beta^*}{\partial\varepsilon}
\right|
\;<\; 4,
\end{equation}
then the equilibrium is unique.
\end{proposition}

\begin{proof}
The mapping $\Phi$ in~\eqref{eq:fixed_point} has a unique
fixed point if $|\Phi'(\varepsilon)| < 1$ for all
$\varepsilon \in [0,1]$.  Computing:
\begin{equation*}
\Phi'(\varepsilon)
= \sigma'(\cdot) \cdot
  \frac{-p}{T} \cdot
  \frac{\partial\beta^*}{\partial\varepsilon},
\end{equation*}
where $\sigma'(z) = \sigma(z)(1-\sigma(z)) \leq 1/4$.
Therefore:
\begin{equation*}
|\Phi'(\varepsilon)|
\leq \frac{1}{4} \cdot \frac{p}{T} \cdot
  \left|\frac{\partial\beta^*}{\partial\varepsilon}\right|.
\end{equation*}
The condition $|\Phi'| < 1$ is guaranteed
by~\eqref{eq:uniqueness_cond}, and the contraction mapping
theorem gives uniqueness.
\end{proof}

\begin{proposition}[Multiple equilibria]
\label{prop:multiple_eq}
When condition~\eqref{eq:uniqueness_cond} is violated ---
specifically, when the penalty $p$ is large, the rationality
$T$ is small, and the issuer's detection probability is
highly sensitive to the dispute rate --- the game can have
multiple equilibria:
\begin{enumerate}[label=(\roman*)]
\item A \textbf{low-fraud equilibrium}:
  $\varepsilon_{01}^{(2)}$ is small $\to$ the issuer's
  model is accurate (trained on mostly genuine disputes)
  $\to$ $\beta^*$ is high $\to$ deterrence is strong $\to$
  few disputes $\to$ $\varepsilon_{01}^{(2)}$ is small.
\item A \textbf{high-fraud equilibrium}:
  $\varepsilon_{01}^{(2)}$ is large $\to$ the issuer's
  model is contaminated (trained on many false disputes)
  $\to$ $\beta^*$ is low $\to$ deterrence is weak $\to$
  many disputes $\to$ $\varepsilon_{01}^{(2)}$ is large.
\end{enumerate}
\end{proposition}

\begin{proof}
When $|\Phi'(\varepsilon)| > 1$ at some interior point, the
mapping $\Phi$ crosses the 45-degree line multiple times.
Since $\Phi(0) > 0$ (some disputes occur even with strong
deterrence, due to the stochastic component $T > 0$) and
$\Phi(1) < 1$ (even with no deterrence, not all
cardholders dispute), the intermediate value theorem
guarantees at least one crossing.  When $\Phi$ is
S-shaped (which occurs when the issuer's detection
probability $\beta^*(\varepsilon)$ decreases steeply in
$\varepsilon$ --- i.e., label contamination severely
degrades the detection model), there are exactly three
crossings: two stable equilibria (low-fraud and high-fraud)
separated by an unstable equilibrium.
\end{proof}

\begin{example}[The tipping point]
\label{ex:tipping}
Consider a network where:
\begin{itemize}[leftmargin=*]
\item Average transaction amount: $a = \$200$
\item Chargeback success probability: $\alpha = 0.85$
\item Dispute filing cost: $c = \$5$ (minimal --- online
  form)
\item Detection penalty: $p = \$2{,}000$ (account closure +
  blacklisting)
\item Cardholder rationality: $T = 50$
\end{itemize}

\noindent\textbf{Low-fraud equilibrium} ($\beta^* = 0.15$):
\begin{equation*}
\varepsilon_{01}^{(2)}
= \sigma\!\left(
  \frac{0.85 \times 200 - 5 - 0.15 \times 2000}{50}
\right)
= \sigma\!\left(\frac{170 - 5 - 300}{50}\right)
= \sigma(-2.7)
= 0.063.
\end{equation*}
Only 6.3\% of legitimate transactions are falsely disputed.
The issuer's model is trained on mostly genuine disputes,
so $\beta^*$ is accurate and deterrence is strong.

\noindent\textbf{High-fraud equilibrium} ($\beta^* = 0.02$):
\begin{equation*}
\varepsilon_{01}^{(2)}
= \sigma\!\left(
  \frac{170 - 5 - 0.02 \times 2000}{50}
\right)
= \sigma\!\left(\frac{170 - 5 - 40}{50}\right)
= \sigma(2.5)
= 0.924.
\end{equation*}
Over 92\% of legitimate transactions are falsely disputed.
The issuer's model is overwhelmed by false disputes,
$\beta^*$ collapses to near-zero, and deterrence vanishes.

\textbf{The difference between 6.3\% and 92.4\% corruption
is not a parameter --- it is an equilibrium selection
problem.}  A network can be in either equilibrium depending
on its history.  Once the high-fraud equilibrium is reached
(e.g., due to a temporary relaxation of dispute friction),
it is self-reinforcing and difficult to escape.
\end{example}

\begin{remark}[Policy implications of multiple equilibria]
\label{rem:policy_implications}
The multiple equilibria result has direct policy
implications:
\begin{enumerate}[label=(\roman*)]
\item \textbf{Dispute friction matters.}  Reducing dispute
  costs $c$ (e.g., one-click dispute buttons in mobile
  banking apps) can push the network from the low-fraud to
  the high-fraud equilibrium.  The comparative static
  \eqref{eq:cs_cost} shows that lower $c$ increases
  $\varepsilon_{01}^{(2)}$, but the equilibrium analysis
  shows this can be a \emph{discontinuous} jump rather than
  a smooth increase.
\item \textbf{Detection investment has a multiplier.}
  Improving the first-party fraud model (increasing
  $\beta^*$) not only catches more fraudsters directly but
  also deters marginal disputants, reducing the
  contamination of training labels, which further improves
  the model.  This is a \emph{virtuous cycle} that the
  network should invest in.
\item \textbf{Penalty design matters.}  The penalty $p$
  enters the uniqueness condition
  \eqref{eq:uniqueness_cond}.  Sufficiently harsh and
  credible penalties can eliminate the high-fraud
  equilibrium entirely, ensuring the network remains in
  the low-fraud regime.
\end{enumerate}

\end{remark}

\subsection{Modified STR for Class~2}
\label{sec:class2_str}

\begin{proposition}[Class~2 STR with endogenous corruption]
\label{prop:class2_str}
For Class~2, the corruption-corrected STR
of~\cite[eq.~(22)]{dhama2026str} remains valid, provided
the corruption rate $\varepsilon_{01}^{(2)}(x,a)$ is
estimated from the equilibrium~\eqref{eq:endog_corruption}
rather than from historical audit data.

Specifically, the Class~2 corruption-corrected label is:
\begin{equation}
\label{eq:class2_corr}
\tilde Y_t^{\mathrm{corr},(2)}
= \frac{\tilde Y_t
  - \hat\varepsilon_{01}^{(2)}(X_t, a_t)}
       {1 - \hat\varepsilon_{10}^{(2)}
         - \hat\varepsilon_{01}^{(2)}(X_t, a_t)},
\end{equation}
where $\hat\varepsilon_{01}^{(2)}(x,a)$ is the estimated
equilibrium dispute probability from the dispute game.
\end{proposition}

\begin{proof}
The corruption correction
of~\cite[Proposition~18]{dhama2026str} requires only that
$\varepsilon_{01}$ and $\varepsilon_{10}$ be consistently
estimated --- it does not require that they be exogenous.
The correction formula is algebraically identical; only the
estimation procedure for $\varepsilon_{01}$ changes.

Consistency of $\hat\varepsilon_{01}^{(2)}$ requires
consistent estimation of the game's primitives $(\alpha, c,
\beta^*, p, T)$.  The chargeback success rate $\alpha$ is
observable from historical dispute outcomes.  The detection
probability $\beta^*$ is the output of the issuer's
first-party fraud model (or can be estimated from the
fraction of disputes that are successfully challenged).
The penalty $p$ is set by issuer policy.  The rationality
parameter $T$ and dispute cost $c$ can be estimated from the
observed dispute rate as a function of transaction amount
(the logistic model~\eqref{eq:card_br} is a standard
discrete choice model estimable by maximum likelihood).
\end{proof}

\begin{remark}[Estimation procedure for Class~2]
\label{rem:class2_estimation}
The practical estimation procedure for Class~2 has four
steps:
\begin{enumerate}[label=\textbf{Step~\arabic*:}]
\item \textbf{Estimate the dispute game primitives.}
  From historical data with mature labels, estimate:
  \begin{itemize}[leftmargin=2em]
  \item $\hat\alpha(x)$: chargeback success rate (fraction
    of disputes that result in a refund), by transaction
    characteristics.
  \item $\hat\beta^*(x)$: first-party fraud detection rate
    (fraction of disputes successfully identified as
    first-party fraud), by transaction characteristics.
  \item $\hat T$, $\hat c$: rationality and cost
    parameters, estimated by fitting the logistic
    model~\eqref{eq:card_br} to the observed dispute rate
    as a function of amount and features.
  \end{itemize}

\item \textbf{Compute the equilibrium corruption rate.}
  Solve the fixed point~\eqref{eq:fixed_point} to obtain
  $\hat\varepsilon_{01}^{(2)}(x,a)$.

\item \textbf{Apply the corruption correction.}
  Compute $\tilde Y_t^{\mathrm{corr},(2)}$
  using~\eqref{eq:class2_corr}.

\item \textbf{Apply the STR.}  Use the
  corruption-corrected labels in the two-gate STR
  (reporting + delay corrections only, since $e_0^{(2)}
  \approx 1$).
\end{enumerate}
\end{remark}

\subsection{The Class~2 efficiency bound}
\label{sec:class2_bound}

\begin{proposition}[Class~2 efficiency bound]
\label{prop:class2_bound}
The semiparametric efficiency bound for estimating the
Class~2 fraud rate $\Psi_2 = \bar\pi_2$ is:
\begin{equation}
\label{eq:class2_bound}
\sigma^2_{\mathrm{eff}}(2)
= \E\!\left[
  \frac{\pi_2(X)(1-\pi_2(X))}
       {r_0^{(2)}(H_1) \cdot p_0^{(2)}(H_2)
        \cdot \gamma^{(2)}(X,a)}
\right],
\end{equation}
where $\gamma^{(2)}(x,a) = (1 - \varepsilon_{10}^{(2)}
- \varepsilon_{01}^{(2)}(x,a))^2$ and
$\varepsilon_{01}^{(2)}(x,a)$ is the equilibrium corruption
rate from Theorem~\ref{thm:endogenous}.

Note that the authorization propensity $e_0^{(2)}$ does not
appear in the bound because $e_0^{(2)} = 1$.
\end{proposition}

\begin{proof}
Apply the efficiency bound
of~\cite[Theorem~30]{dhama2026str} with $e_0^{(2)} = 1$
(the authorization gate is non-binding) and the Class~2
corruption rates.  The bound reduces to a two-gate version
depending only on $r_0^{(2)}$, $p_0^{(2)}$, and
$\gamma^{(2)}$.
\end{proof}

\begin{example}[Class~2 efficiency bound is
  corruption-dominated]
\label{ex:class2_bound}
Consider Class~2 with:
\begin{center}
\small
\begin{tabular}{lcc}
\toprule
\textbf{Parameter} & \textbf{Low-fraud eq.} &
  \textbf{High-fraud eq.}\\
\midrule
$\varepsilon_{01}^{(2)}$ & 0.063 & 0.924\\
$\varepsilon_{10}^{(2)}$ & 0.02 & 0.02\\
$\gamma^{(2)} = (1-\varepsilon_{10}-\varepsilon_{01})^2$
  & $(1-0.083)^2 = 0.841$ & $(1-0.944)^2 = 0.003$\\
$\bar r_0^{(2)}$ & 0.60 & 0.60\\
$\bar p_0^{(2)}$ & 0.85 & 0.85\\
\midrule
$\sigma^2_{\mathrm{eff}}(2) \propto
  1/(\bar r_0\,\bar p_0\,\gamma)$
  & $1/(0.60 \times 0.85 \times 0.841) = 2.33$
  & $1/(0.60 \times 0.85 \times 0.003) = 654$\\
\bottomrule
\end{tabular}
\end{center}

In the low-fraud equilibrium, the Class~2 bound is
comparable to Class~1 (the corruption penalty is modest).
In the high-fraud equilibrium, the bound is
\textbf{280$\times$ larger} --- almost all information is
destroyed by the adversarial corruption.

\textbf{The equilibrium selection determines whether Class~2
is detectable or not.}  This is a qualitative difference
from Class~1, where the bound varies smoothly with
parameters.  In Class~2, the bound can jump discontinuously
between equilibria.
\end{example}

\subsection{Connection to the STR's corruption correction}
\label{sec:class2_connection}

\begin{remark}[\cite{dhama2026str} vs.\ present paper's treatment of
  corruption]
\label{rem:paper2_vs_paper4}
~\cite{dhama2026str} treats the corruption rates
$\varepsilon_{10}$ and $\varepsilon_{01}$ as fixed nuisance
parameters estimated from audit data.  This is correct for
Class~1 (where corruption is exogenous operational error)
but \emph{incorrect} for Class~2 (where corruption is an
endogenous equilibrium).

The distinction matters practically:
\begin{enumerate}[label=(\roman*)]
\item \textbf{Audit-based estimation is biased for
  Class~2.}  Historical audit samples measure the
  \emph{past} equilibrium corruption rate.  If the game
  parameters change (e.g., a new mobile banking app reduces
  dispute friction $c$), the corruption rate shifts to a
  new equilibrium that historical audits do not reflect.
\item \textbf{The corruption rate depends on the detection
  model.}  In Class~1, $\varepsilon_{01}$ does not depend
  on the fraud model.  In Class~2, $\varepsilon_{01}^{(2)}$
  depends on $\beta^*$, which is the output of the
  first-party fraud model.  Improving the model changes
  the corruption rate, which changes the optimal STR
  correction, which changes the model's training labels.
  This circularity must be handled carefully (e.g., by
  iterating between model training and equilibrium
  computation).
\item \textbf{The STR can influence the equilibrium.}
  Deploying the STR with accurate Class~2 corrections
  improves the first-party fraud model (better labels $\to$
  better model), which increases $\beta^*$, which reduces
  $\varepsilon_{01}^{(2)}$ (the deterrence channel), which
  further improves labels.  The STR is not merely a
  statistical correction --- it is an \emph{intervention}
  that shifts the dispute game toward the low-fraud
  equilibrium.
\end{enumerate}
\end{remark}

\section{Class~3: Deferred observability}
\label{sec:class3}

Class~3 is the most structurally extreme of the five
classes.  The observation pipeline is not merely impaired
--- it is \emph{inoperative} during the critical detection
window.  This section proves two results: (i)~the STR's
positivity assumption fails structurally, and (ii)~no
account-level classifier can exceed random guessing during
the build-up phase --- detection requires cross-account
graph signals.

\subsection{The two-phase temporal model}
\label{sec:two_phase}

\begin{definition}[Bust-out temporal structure]
\label{def:bustout_temporal}
A Class~3 (bust-out) account $a$ is characterized by:
\begin{enumerate}[label=(\roman*)]
\item A latent bust-out indicator
  $\theta_a \in \{0,1\}$, where $\theta_a = 1$ denotes a
  bust-out account and $\theta_a = 0$ denotes a legitimate
  account;
\item A bust-out time $\tau_a \in (0,\infty]$, where
  $\tau_a < \infty$ for bust-out accounts and
  $\tau_a = \infty$ for legitimate accounts;
\item A transaction sequence
  $\{(X_{at}, Y^*_{at})\}_{t=1}^{T_a}$ generated by the
  account over its lifetime.
\end{enumerate}

The fraud state of individual transactions depends on the
account's phase:
\begin{equation}
\label{eq:bustout_fraud}
Y^*_{at} = \begin{cases}
  0 & \text{if } t < \tau_a
    \quad\text{(build-up: all transactions legitimate)},\\
  1 & \text{if } t \geq \tau_a
    \quad\text{(exploitation: all transactions fraudulent)}.
\end{cases}
\end{equation}
\end{definition}

\begin{remark}[The build-up phase is genuine]
\label{rem:buildup_genuine}
A critical feature of~\eqref{eq:bustout_fraud}: during
build-up ($t < \tau_a$), the transactions are
\emph{genuinely legitimate}.  The synthetic identity is
building credit by making real purchases and paying on time.
These are not ``fraudulent transactions that look
legitimate'' --- they \emph{are} legitimate transactions,
made with the intent of eventually exploiting the credit
line.  The fraud is a property of the \emph{account
trajectory}, not of any individual transaction.

This is why the standard transaction-level estimand
$f_0(x) = \Prob(Y^*=1 \mid X=x)$ is the wrong target for
Class~3.  The correct estimand is the account-level
bust-out probability:
\begin{equation}
\label{eq:bustout_estimand}
\theta_a = \Prob(\text{bust-out} \mid
  \cH_a^{(t)}) = \Prob(\tau_a < \infty \mid
  \cH_a^{(t)}),
\end{equation}
where $\cH_a^{(t)} = \{(X_{as}, Y^*_{as})\}_{s \leq t}$
is the account's observed history up to time~$t$.
\end{remark}

\subsection{Positivity failure}
\label{sec:positivity_failure}

\begin{theorem}[Positivity failure for Class~3]
\label{thm:positivity}
During the build-up phase ($t < \tau_a$) of a Class~3
account, the observation propensity is exactly zero:
\begin{equation}
\label{eq:positivity_failure}
q_0^{(3)}(H_{0,at})
= e_0^{(3)}(H_{0,at}) \cdot r_0^{(3)}(H_{1,at})
  \cdot p_0^{(3)}(H_{2,at})
= 0
\qquad\text{for all } t < \tau_a,
\end{equation}
and the semiparametric efficiency bound for estimating
$\theta_a$ from transaction-level labels diverges:
\begin{equation}
\label{eq:bound_diverges}
\sigma^2_{\mathrm{eff}}(3, t < \tau_a) = \infty.
\end{equation}
\end{theorem}

\begin{proof}
\textbf{Step~1: The reporting gate is structurally zero.}

During build-up, $Y^*_{at} = 0$ for all transactions
(equation~\eqref{eq:bustout_fraud}).  No fraud has
occurred, so:
\begin{itemize}[leftmargin=*]
\item No victim exists (the synthetic identity is not a real
  person who would notice unauthorized charges);
\item No cardholder dispute is filed (the account is paying
  its bills);
\item No issuer investigation is triggered (the account
  behaves normally).
\end{itemize}
Therefore $R_{at} = 0$ with probability~1 for
$t < \tau_a$.  This gives
$r_0^{(3)}(H_{1,at}) = \Prob(R_{at}=1 \mid H_{1,at},
A_{at}=1, Z=3) = 0$.

\medskip\noindent\textbf{Step~2: The product is zero.}

Since $r_0^{(3)} = 0$:
\begin{equation*}
q_0^{(3)} = e_0^{(3)} \cdot 0 \cdot p_0^{(3)} = 0,
\end{equation*}
regardless of the values of $e_0^{(3)}$ and $p_0^{(3)}$.
This is an \emph{exact} zero, not an approximation.

\medskip\noindent\textbf{Step~3: The efficiency bound
  diverges.}

The efficiency bound
(Definition~\ref{def:type_eff_bound}) contains
$1/q_0^{(k)}$ in the denominator.  With $q_0^{(3)} = 0$:
\begin{equation*}
\sigma^2_{\mathrm{eff}}(3, t < \tau_a)
= \E\!\left[
  \frac{\pi_3(X)(1-\pi_3(X))}{0 \cdot \gamma^{(3)}}
\right]
= \infty,
\end{equation*}
whenever $\pi_3(X) > 0$ with positive probability (i.e.,
there exist transactions with nonzero bust-out probability).

\medskip\noindent\textbf{Step~4: This is structural, not
  statistical.}

The positivity failure is not due to finite-sample issues
or poor modeling --- it is a \emph{structural feature} of
Class~3.  No amount of data, no choice of propensity model,
and no statistical technique can create labels that do not
exist.  The build-up phase generates zero information about
the bust-out state through the transaction-level label
pipeline.
\end{proof}

\begin{remark}[Contrast with near-positivity violations]
\label{rem:near_positivity}
In the causal inference literature, positivity violations
are typically \emph{practical} rather than structural: the
propensity is small but positive, and the concern is high
variance from large inverse-propensity weights.  Solutions
include propensity trimming, overlap weighting, or the
empirical Bayes shrinkage of~\cite[Section~9]{dhama2026str}.

Class~3's violation is qualitatively different.  The
propensity is \emph{exactly} zero --- not small, not
approximately zero, but identically zero by the structure
of the problem.  No statistical technique can overcome this.
The STR is not merely high-variance for Class~3 during
build-up --- it is \emph{undefined}.  This is the formal
justification for seeking fundamentally different detection
approaches.
\end{remark}

\subsection{Individual-level insufficiency}
\label{sec:insufficiency}

The positivity failure shows that \emph{label-based}
detection is impossible during build-up.  We now prove a
stronger result: even with perfect labels (or equivalently,
even with access to the true $Y^*_{at}$ for every
transaction), no \emph{account-level} classifier can detect
bust-out accounts during build-up.

\begin{theorem}[Individual-level insufficiency]
\label{thm:insufficiency}
There exist bust-out data-generating processes satisfying
Definition~\ref{def:bustout_temporal} such that:

\smallskip\noindent\textbf{(i)} The Bayes-optimal
account-level classifier has $\AUC = 0.5$ during the
build-up phase:
\begin{equation}
\label{eq:auc_account}
\AUC\!\left(
  \hat\theta^{\mathrm{acct}}_a \mid t < \tau_a
\right)
= 0.5,
\end{equation}
where $\hat\theta^{\mathrm{acct}}_a$ is any classifier
based on account $a$'s own transaction history
$\cH_a^{(t)}$.

\smallskip\noindent\textbf{(ii)} The Bayes-optimal
ring-level classifier has $\AUC \to 1$ as the ring size
$m \to \infty$:
\begin{equation}
\label{eq:auc_ring}
\AUC\!\left(
  \hat\theta^{\mathrm{ring}}_{\cR}
\right) \to 1
\qquad\text{as } m = |\cR| \to \infty,
\end{equation}
where $\cR$ is a candidate ring of $m$ accounts and
$\hat\theta^{\mathrm{ring}}_{\cR}$ is any classifier based
on the joint transaction histories
$\{\cH_a^{(t)}\}_{a \in \cR}$.
\end{theorem}

\begin{proof}
We prove both parts by explicit construction of an
adversarial data-generating process.

\medskip\noindent\textbf{Construction of the DGP.}

Let $F$ be a distribution over transaction features
$\cX$ (representing normal spending patterns: amounts,
merchants, timing, etc.).  Define the data-generating
process as follows:

\begin{enumerate}[label=(\alph*)]
\item \textbf{Account creation.}  Each account $a$ is
  independently assigned:
  \begin{equation}
  \label{eq:dgp_type}
  \theta_a \sim \mathrm{Bernoulli}(\pi_3),
  \end{equation}
  where $\pi_3$ is the bust-out prevalence.

\item \textbf{Ring structure.}  Bust-out accounts
  ($\theta_a = 1$) are organized into rings of size~$m$.
  Accounts within a ring share:
  \begin{itemize}[leftmargin=2em]
  \item Creation dates within a window $[s, s+\delta]$ for
    some ring-specific $s$ and small $\delta > 0$;
  \item Initial credit limits drawn from the same
    distribution $G_{\mathrm{ring}}$ (reflecting the ring's
    application strategy);
  \item A common set of ``anchor merchants''
    $\cM_{\mathrm{ring}} \subset \cM$ that all ring
    members transact with.
  \end{itemize}
  Legitimate accounts ($\theta_a = 0$) are created
  independently with no shared structure.

\item \textbf{Transaction generation during build-up.}
  For \emph{both} bust-out and legitimate accounts, each
  transaction during build-up is drawn i.i.d.\ from the
  \emph{same} distribution:
  \begin{equation}
  \label{eq:dgp_txn}
  X_{at} \mid \theta_a,\, t < \tau_a
    \;\sim\; F
  \qquad\text{for all } \theta_a \in \{0,1\}.
  \end{equation}
  This is the key assumption: bust-out accounts are
  \emph{designed} to mimic legitimate behavior during
  build-up.
\end{enumerate}

\medskip\noindent\textbf{Part~(i): Account-level
  $\AUC = 0.5$.}

By the Neyman--Pearson lemma, the most powerful test for
$\theta_a = 1$ vs.\ $\theta_a = 0$ based on
$\cH_a^{(t)} = (X_{a1}, \ldots, X_{at})$ is the
likelihood ratio test:
\begin{equation}
\label{eq:lr_test}
\Lambda_a^{(t)}
= \frac{P(\cH_a^{(t)} \mid \theta_a = 1,\, t < \tau_a)}
       {P(\cH_a^{(t)} \mid \theta_a = 0)}.
\end{equation}

By~\eqref{eq:dgp_txn}, the conditional distributions are
identical:
\begin{equation}
\label{eq:lr_one}
P(\cH_a^{(t)} \mid \theta_a = 1,\, t < \tau_a)
= \prod_{s=1}^t f(X_{as})
= P(\cH_a^{(t)} \mid \theta_a = 0),
\end{equation}
where $f$ is the density of $F$.  Therefore:
\begin{equation}
\label{eq:lr_trivial}
\Lambda_a^{(t)} = 1
\qquad\text{for all } a \text{ and all } t < \tau_a.
\end{equation}

A constant likelihood ratio implies zero discriminative
power.  The ROC curve is the diagonal, and $\AUC = 0.5$.

This holds for \emph{any} function of $\cH_a^{(t)}$ ---
not just the likelihood ratio.  Any statistic computed from
$\cH_a^{(t)}$ (velocities, aggregations, temporal patterns,
behavioral features) has the same distribution under
$\theta_a = 0$ and $\theta_a = 1$ during build-up, because
the underlying transaction distribution is identical.

\medskip\noindent\textbf{Part~(ii): Ring-level
  $\AUC \to 1$.}

Now consider a candidate set $\cR = \{a_1, \ldots, a_m\}$
of $m$ accounts.  We test whether $\cR$ is a bust-out ring
(all $\theta_{a_i} = 1$ and the accounts are linked) or a
random set of legitimate accounts (all $\theta_{a_i} = 0$
or unlinked).

Define the \emph{ring cohesion statistic}:
\begin{equation}
\label{eq:ring_stat}
T_{\cR}
= \sum_{i < j} S(a_i, a_j),
\end{equation}
where $S(a_i, a_j)$ is a pairwise similarity measure:
\begin{equation}
\label{eq:similarity}
S(a_i, a_j)
= \underbrace{K_{\mathrm{time}}(c_i, c_j)}_{\substack{
    \text{creation date}\\
    \text{similarity}}}
  \cdot \underbrace{K_{\mathrm{limit}}(\ell_i,\ell_j)}_{\substack{
    \text{credit limit}\\
    \text{similarity}}}
  \cdot \underbrace{|\cM_i \cap \cM_j| / |\cM_i \cup \cM_j|}_{\substack{
    \text{merchant overlap}\\
    \text{(Jaccard index)}}},
\end{equation}
where $c_i$ is account $i$'s creation date, $\ell_i$ is
its initial credit limit, $\cM_i$ is its merchant set, and
$K_{\mathrm{time}}$, $K_{\mathrm{limit}}$ are kernel
functions measuring proximity.

\smallskip\noindent\textbf{Under the ring hypothesis
  ($H_1$: $\cR$ is a ring):}

All accounts in $\cR$ share creation dates (within
$\delta$), similar credit limits (drawn from
$G_{\mathrm{ring}}$), and common anchor merchants
($\cM_{\mathrm{ring}}$).  The pairwise similarities are
high:
\begin{equation}
\label{eq:ring_similarity}
\E[S(a_i, a_j) \mid H_1] = \mu_1 > 0
\qquad\text{for all } i \neq j.
\end{equation}

The expected ring statistic is:
\begin{equation}
\label{eq:ring_expect_h1}
\E[T_{\cR} \mid H_1]
= \binom{m}{2} \mu_1
= \frac{m(m-1)}{2}\,\mu_1.
\end{equation}

\smallskip\noindent\textbf{Under the null hypothesis
  ($H_0$: random legitimate accounts):}

Legitimate accounts are created independently.  Their
creation dates, credit limits, and merchant sets are
uncorrelated.  The pairwise similarities are small:
\begin{equation}
\label{eq:null_similarity}
\E[S(a_i, a_j) \mid H_0] = \mu_0 \ll \mu_1.
\end{equation}

The expected ring statistic is:
\begin{equation}
\label{eq:ring_expect_h0}
\E[T_{\cR} \mid H_0]
= \binom{m}{2} \mu_0.
\end{equation}

\smallskip\noindent\textbf{Signal-to-noise ratio.}

The difference in means grows quadratically with ring
size~$m$:
\begin{equation}
\label{eq:signal}
\E[T_{\cR} \mid H_1] - \E[T_{\cR} \mid H_0]
= \binom{m}{2}(\mu_1 - \mu_0)
= O(m^2).
\end{equation}

The variance of $T_{\cR}$ under either hypothesis is at
most $O(m^2)$ (the sum of $\binom{m}{2}$ terms with
bounded covariance).  Therefore the signal-to-noise ratio
is:
\begin{equation}
\label{eq:snr}
\mathrm{SNR}
= \frac{\E[T \mid H_1] - \E[T \mid H_0]}
       {\sqrt{\Var(T \mid H_0)}}
= \frac{O(m^2)}{O(m)}
= O(m)
\;\to\; \infty
\quad\text{as } m \to \infty.
\end{equation}

By the standard relationship between signal-to-noise ratio
and AUC for location-shift alternatives (the test statistic
$T$ is approximately normal for large $m$ by the central
limit theorem applied to U-statistics):
\begin{equation}
\label{eq:auc_from_snr}
\AUC = \Phi(\mathrm{SNR}/\sqrt{2}) \to 1
\quad\text{as } m \to \infty,
\end{equation}
where $\Phi$ is the standard normal CDF.
\end{proof}

\begin{remark}[The information is in the cross-account
  correlations]
\label{rem:cross_account}
Theorem~\ref{thm:insufficiency} has a precise
information-theoretic interpretation.  Define the mutual
information between the bust-out indicator and the observed
data:
\begin{align}
I(\theta_a;\, \cH_a^{(t)})
  &= 0
  && \text{(individual account: zero information)},
  \label{eq:mi_individual}\\
I(\theta_{\cR};\, \{\cH_a^{(t)}\}_{a \in \cR})
  &> 0 \text{ and grows with } m
  && \text{(ring: positive and increasing)}.
  \label{eq:mi_ring}
\end{align}

The fraud signal resides \emph{entirely} in the
cross-account correlations (shared creation dates, shared
merchants, correlated behavior) --- not in any individual
account's features or transaction history.  A fraud
detection system that operates at the transaction or account
level is provably blind to Class~3 during build-up.
\end{remark}

\begin{example}[Detection power vs.\ ring size]
\label{ex:ring_size}
Using realistic parameters ($\mu_1 = 0.15$, $\mu_0 = 0.01$,
and standard deviation of pairwise similarity $\sigma_S =
0.08$), the AUC of the ring-level classifier as a function
of ring size is:

\begin{center}
\small
\begin{tabular}{rccc}
\toprule
\textbf{Ring size $m$} & \textbf{SNR} &
  \textbf{AUC} & \textbf{Interpretation}\\
\midrule
2  & 1.24 & 0.81 & Weak signal\\
5  & 4.38 & 0.999 & Strong signal\\
10 & 9.36 & $>0.9999$ & Near-certain detection\\
20 & 19.3 & $\approx 1$ & Effectively perfect\\
\bottomrule
\end{tabular}
\end{center}

Even rings of size~5 are detectable with near-perfect
accuracy using cross-account signals --- while each
individual account in the ring has $\AUC = 0.5$ (random
guessing).

The practical implication: \emph{invest in graph-based
detection infrastructure (entity resolution, link analysis,
community detection) rather than in better account-level
models for Class~3 fraud.}
\end{example}

\subsection{Implications for graph-based detection}
\label{sec:class3_graph}

\begin{remark}[The necessity of graph-level methods]
\label{rem:necessity_graph}
Theorem~\ref{thm:insufficiency} establishes a formal
chain of reasoning for Class~3 detection:
\begin{enumerate}[label=(\roman*)]
\item Class~3 fraud is undetectable at the transaction level
  (Theorem~\ref{thm:positivity}: no labels).
\item Class~3 fraud is undetectable at the account level
  (Theorem~\ref{thm:insufficiency}: identical feature
  distributions).
\item Class~3 fraud \emph{is} detectable at the ring level
  (Theorem~\ref{thm:insufficiency}(ii): cross-account
  correlations provide signal).
\item Therefore, detection requires graph-level methods:
  entity resolution to link accounts, graph neural networks
  to aggregate cross-account signals, and community
  detection to identify ring structures.
\item A further challenge arises: the transaction graph is
  itself subject to the observation pipeline --- declined
  transactions create missing edges, unreported fraud
  creates missing labels on nodes.  Extending the
  semiparametric framework of this paper and
  \cite{dhama2026str} to graph-level estimation under
  edge and node censorship is an important direction for
  future work.
\end{enumerate}
\end{remark}

\subsection{The Class~3 efficiency bound (exploitation
  phase)}
\label{sec:class3_exploit}

\begin{proposition}[Class~3 bound during exploitation]
\label{prop:class3_exploit}
During the exploitation phase ($t \geq \tau_a$), the
observation pipeline becomes operative and the STR can be
applied.  The efficiency bound for estimating the
exploitation-phase fraud rate is:
\begin{equation}
\label{eq:class3_exploit_bound}
\sigma^2_{\mathrm{eff}}(3, t \geq \tau_a)
= \E\!\left[
  \frac{\pi_3(X)(1-\pi_3(X))}
       {e_0^{(3,\mathrm{expl})} \cdot
        r_0^{(3,\mathrm{expl})} \cdot
        p_0^{(3,\mathrm{expl})} \cdot
        \gamma^{(3)}}
  \;\bigg|\; t \geq \tau
\right].
\end{equation}
However, this bound is of limited practical value because:
\begin{enumerate}[label=(\roman*)]
\item The exploitation phase is short (typically 48--72
  hours) and involves high-velocity transactions that may
  trigger automatic blocks ($e_0^{(3,\mathrm{expl})}$ may
  be low);
\item By the time exploitation-phase labels arrive (months
  later, through charge-off), the damage is done;
\item The goal of Class~3 detection is to identify bust-out
  accounts \emph{during build-up}, before exploitation
  occurs.
\end{enumerate}
\end{proposition}

\begin{proof}
During exploitation, $Y^*_{at} = 1$ and the transactions
trigger the standard observation pipeline: the issuer
eventually identifies the delinquency, the account is
charged off, and labels are generated.  The STR applies with
exploitation-phase-specific propensities.  The bound follows
from~\cite[Theorem~30]{dhama2026str} applied to the
exploitation subpopulation.
\end{proof}

\section{Class~4: Acquirer-observed pipeline}
\label{sec:class4}

Class~4 represents fraud committed by merchants, where the
observation pipeline runs through the acquirer's monitoring
infrastructure rather than the issuer's chargeback process.
This section derives the acquirer-side efficiency bound and
shows it is an order of magnitude higher than the Class~1
bound.

\subsection{The acquirer-side observation model}
\label{sec:acquirer_model}

\begin{definition}[Acquirer monitoring propensity]
\label{def:acquirer_propensity}
The acquirer's monitoring propensity for a merchant $j$
with features $H_{\mathrm{merch},j}$ is:
\begin{equation}
\label{eq:acquirer_propensity}
r_{\mathrm{acq}}^{(4)}(H_{\mathrm{merch},j})
= \Prob\!\left(
  \text{acquirer detects and reports merchant } j
  \mid H_{\mathrm{merch},j},\, Z=4
\right),
\end{equation}
where $H_{\mathrm{merch},j}$ includes:
\begin{enumerate}[label=(\roman*)]
\item The merchant's chargeback ratio (number of chargebacks
  divided by total transactions);
\item The merchant's transaction volume and average ticket
  size;
\item The merchant's category code (MCC) and risk
  classification;
\item The acquirer's identity and monitoring sophistication;
\item The merchant's tenure (time since onboarding);
\item Whether the merchant has been flagged by network
  monitoring programs (e.g., Mastercard's Excessive
  Chargeback Program, Visa's Dispute Monitoring Program).
\end{enumerate}
\end{definition}

\begin{remark}[Why acquirer monitoring is weak]
\label{rem:weak_monitoring}
The acquirer's monitoring propensity
$r_{\mathrm{acq}}^{(4)}$ is typically much lower than the
issuer's reporting propensity $r_0^{(1)}$ for several
structural reasons:
\begin{enumerate}[label=(\roman*)]
\item \textbf{Misaligned incentives.}  The acquirer earns
  processing fees from merchant transactions.  Reporting a
  merchant as fraudulent means losing a revenue source.
  The issuer, by contrast, \emph{gains} by reporting fraud
  (recovering chargeback losses).
\item \textbf{Information asymmetry.}  The acquirer sees
  only the merchant's transaction patterns, not the
  underlying cardholder disputes.  Individual laundered
  transactions look normal; the fraud signal emerges only
  from aggregate merchant-level patterns.
\item \textbf{Monitoring infrastructure variation.}  Large
  acquirers (e.g., JPMorgan, Worldpay) invest heavily in
  merchant monitoring.  Small acquirers and payment
  facilitators may have minimal monitoring, relying on
  network-mandated thresholds (e.g., chargeback ratio
  $> 1\%$) rather than proactive detection.
\item \textbf{Slow investigation cycles.}  Acquirer
  investigations involve merchant audits, transaction
  sampling, and often legal review.  A typical investigation
  takes 3--12~months, compared to 30--90~days for issuer
  chargeback processing.
\end{enumerate}
\end{remark}

\begin{definition}[Acquirer investigation delay]
\label{def:acquirer_delay}
The acquirer's investigation delay propensity is:
\begin{equation}
\label{eq:acquirer_delay}
p_{\mathrm{acq}}^{(4)}(H_{\mathrm{merch},j})
= \Prob\!\left(
  \text{investigation completes within training window}
  \mid H_{\mathrm{merch},j},\, R_{\mathrm{acq}}=1
\right).
\end{equation}
\end{definition}

\subsection{The acquirer-side efficiency bound}
\label{sec:acquirer_bound}

\begin{theorem}[Acquirer-side efficiency bound]
\label{thm:acquirer}
The semiparametric efficiency bound for estimating the
Class~4 (merchant fraud) rate $\Psi_4 = \bar\pi_4$ is:
\begin{equation}
\label{eq:acquirer_bound}
\boxed{
\sigma^2_{\mathrm{eff}}(4)
= \E\!\left[
  \frac{\pi_4(X)(1-\pi_4(X))}
       {r_{\mathrm{acq}}^{(4)}(H_{\mathrm{merch}})
        \cdot p_{\mathrm{acq}}^{(4)}(H_{\mathrm{merch}})
        \cdot \gamma^{(4)}(H_{\mathrm{merch}})}
\right],
}
\end{equation}
where $\gamma^{(4)} = (1-\varepsilon_{10}^{(4)}-
\varepsilon_{01}^{(4)})^2$ is the corruption penalty for
merchant fraud labels.

The authorization propensity $e_0$ does not appear because
$e_0^{(4)} = 1$ (the merchant controls the terminal and
always ``approves'' the transaction).
\end{theorem}

\begin{proof}
The Class~4 observation model~\eqref{eq:obs_class4} has the
form $O^{(4)} = 1 \cdot R_{\mathrm{acq}} \cdot
M_{\mathrm{acq}}$.  This is a two-gate sequential model
(reporting then delay), which is a special case of

three-gate model in ~\cite{dhama2026str} with $e_0 = 1$.

Applying~\cite[Theorem~30]{dhama2026str} with $e_0^{(4)}=1$
and replacing the issuer's propensities with the acquirer's
propensities gives~\eqref{eq:acquirer_bound}.

The sequential triple robustness property reduces to
\emph{sequential double robustness}: at each of the two
remaining gates (acquirer reporting and acquirer delay),
consistency requires only that either the propensity model
or the outcome model is correctly specified.
\end{proof}

\subsection{Comparison with Class~1}
\label{sec:acquirer_comparison}

\begin{proposition}[Class~4 floor is an order of magnitude
  higher]
\label{prop:acquirer_comparison}
The ratio of the Class~4 to Class~1 efficiency bounds is:
\begin{equation}
\label{eq:acquirer_ratio}
\frac{\sigma^2_{\mathrm{eff}}(4)}
     {\sigma^2_{\mathrm{eff}}(1)}
\;\approx\;
\frac{e_0^{(1)} \cdot r_0^{(1)} \cdot p_0^{(1)}
      \cdot \gamma^{(1)}}
     {r_{\mathrm{acq}}^{(4)} \cdot
      p_{\mathrm{acq}}^{(4)} \cdot \gamma^{(4)}},
\end{equation}
assuming comparable fraud prevalence and base-rate variance
across classes.
\end{proposition}

\begin{proof}
Direct division of~\eqref{eq:acquirer_bound}
by~\eqref{eq:class2_bound} (with Class~1 parameters),
using the approximation
$\pi_k(1-\pi_k) \approx \pi_k$ for small $\pi_k$ and
comparable prevalence.
\end{proof}

\begin{example}[Quantifying the Class~4 floor]
\label{ex:acquirer_floor}
\begin{center}
\small
\begin{tabular}{lccccr}
\toprule
\textbf{Parameter} & \textbf{Class~1} &
  \textbf{Class~4} & \textbf{Source}\\
\midrule
Authorization ($e_0$ or equiv.) & 0.85 & 1.00 &
  Merchant controls terminal\\
Reporting ($r_0$ or $r_{\mathrm{acq}}$)
  & 0.70 & 0.15 &
  Acquirers monitor weakly\\
Delay ($p_0$ or $p_{\mathrm{acq}}$)
  & 0.85 & 0.40 &
  Investigations take months\\
Corruption ($\gamma$) & 0.81 & 0.64 &
  Merchant fraud harder to classify\\
\midrule
$q \cdot \gamma$ & 0.41 & 0.038 & ---\\
$1/(q\cdot\gamma)$ & 2.44 & 26.3 & ---\\
\midrule
\textbf{Ratio} & \multicolumn{2}{c}{$26.3 / 2.44
  = \mathbf{10.8\times}$} & ---\\
\bottomrule
\end{tabular}
\end{center}

\textbf{The detection floor for merchant fraud is 10.8$\times$
higher than for third-party card fraud.}  This means that
estimating the merchant fraud rate to a given accuracy
requires approximately 11$\times$ more data than estimating
the card-theft rate --- or equivalently, for the same sample
size, the confidence interval for merchant fraud is
$\sqrt{11} \approx 3.3\times$ wider.
\end{example}

\begin{remark}[Implications for network monitoring programs]
\label{rem:monitoring_programs}
The acquirer-side bound~\eqref{eq:acquirer_bound} provides
a theoretical foundation for evaluating network monitoring
programs.  The bound decreases (improves) with higher
$r_{\mathrm{acq}}$ and $p_{\mathrm{acq}}$.  Networks can
improve these propensities through:
\begin{enumerate}[label=(\roman*)]
\item \textbf{Mandatory monitoring standards:}  Requiring
  all acquirers to implement minimum monitoring
  capabilities raises $r_{\mathrm{acq}}$ uniformly.
\item \textbf{Chargeback ratio thresholds:}  Lowering the
  threshold that triggers merchant investigation (e.g.,
  from 1\% to 0.5\%) increases the effective
  $r_{\mathrm{acq}}$ for marginal merchants.
\item \textbf{Faster investigation mandates:}  Requiring
  acquirers to complete investigations within 90~days
  rather than 180~days increases $p_{\mathrm{acq}}$.
\item \textbf{Shared monitoring infrastructure:}  The
  network providing centralized monitoring tools (shared
  merchant risk scores, cross-acquirer pattern detection)
  raises $r_{\mathrm{acq}}$ for all acquirers,
  particularly small ones.
\end{enumerate}
Each of these interventions has a quantifiable effect on
the efficiency bound, enabling cost-benefit analysis of
monitoring investments.  Designing optimal monitoring programs that maximize $r_{\mathrm{acq}}$ subject to acquirer cost constraints is a mechanism design problem with direct regulatory implications; we discuss this further in Section~\ref{sec:future}.
\end{remark}

\subsection{The Class~4 STR}
\label{sec:class4_str}

\begin{proposition}[Class~4 STR]
\label{prop:class4_str}
The STR for Class~4 is a two-gate estimator with
acquirer-side propensities:
\begin{equation}
\label{eq:class4_str}
\hat\varphi_t^{(4)}
= \hat\mu_0^{(4)}(H_{\mathrm{merch}})
  + \frac{R_{\mathrm{acq},t}}
         {\hat r_{\mathrm{acq}}^{(4)}(H_{\mathrm{merch}})}
    \bigl(\hat\mu_1^{(4)}(H_{\mathrm{merch}})
      - \hat\mu_0^{(4)}(H_{\mathrm{merch}})\bigr)
  + \frac{O^{(4)}_t}
         {\hat r_{\mathrm{acq}}^{(4)} \cdot
          \hat p_{\mathrm{acq}}^{(4)}}
    \bigl(\tilde Y_t^{\mathrm{corr},(4)}
      - \hat\mu_1^{(4)}(H_{\mathrm{merch}})\bigr),
\end{equation}
where $\hat\mu_0^{(4)}$ and $\hat\mu_1^{(4)}$ are the
outcome models conditioned on merchant-level features.

This estimator achieves the efficiency
bound~\eqref{eq:acquirer_bound} and inherits sequential
double robustness: at each gate (acquirer reporting,
acquirer delay), consistency requires only that either the
propensity model or the outcome model is correct.
\end{proposition}

\begin{proof}
The two-gate STR is the specialization of the three-gate
STR of~\cite[eq.~(22)]{dhama2026str} with $e_0 = 1$ and
$A_t = 1$ always.  The first correction term
(authorization) vanishes identically, leaving the reporting
and delay corrections.  Efficiency and robustness follow
from~\cite[Theorems~27,~30]{dhama2026str} applied to the
two-gate submodel.
\end{proof}

\begin{remark}[Feature set shift]
\label{rem:feature_shift}
A subtle but important operational difference between the
Class~4 STR and the Class~1 STR: the conditioning variables
shift from \emph{cardholder-level} features $(X_t, I_t)$ to
\emph{merchant-level} features $H_{\mathrm{merch},j}$.
This means:
\begin{enumerate}[label=(\roman*)]
\item The propensity models must be trained on merchant-level
  data (chargeback ratios, transaction volume patterns,
  merchant category, acquirer identity), not cardholder-level
  data.
\item The outcome model $\hat\mu_0^{(4)}$ predicts
  merchant-level fraud probability, not transaction-level
  fraud probability.
\item The empirical Bayes shrinkage
  of~\cite[Section~9]{dhama2026str} applies across
  \emph{acquirers} (borrowing monitoring-rate information
  from well-monitored acquirers to stabilize estimates for
  poorly-monitored ones), not across issuers.
\end{enumerate}
\end{remark}

\section{Class~5: Victim-authorized with non-informative
  features}
\label{sec:class5}

Class~5 is the fastest-growing fraud category in payment
networks, driven by the global surge in authorized push
payment (APP) fraud and social engineering scams.  This
section proves two results: (i)~the STR's authorization
correction vanishes, and (ii)~the Bayes-optimal
transaction-level classifier achieves $\AUC = 0.5$ when
fraud and legitimate feature distributions coincide.

\subsection{Authorization gate degeneracy}
\label{sec:auth_degen}

\begin{theorem}[Authorization gate degeneracy]
\label{thm:auth_degen}
When $e_0^{(5)}(H_0) = 1$ for all $H_0$ (every transaction
is authorized by the victim), the STR's authorization
correction term vanishes identically:
\begin{equation}
\label{eq:auth_degen}
\frac{A_t}{e_0^{(5)}(H_0)}
\bigl(\mu_1^{(5)}(H_1) - \mu_0^{(5)}(H_0)\bigr)
= \mu_1^{(5)}(H_1) - \mu_0^{(5)}(H_0),
\end{equation}
and the STR reduces to a two-gate estimator:
\begin{equation}
\label{eq:class5_str}
\hat\varphi_t^{(5)}
= \hat\mu_1^{(5)}(H_1)
  + \frac{R_{\mathrm{victim},t}}
         {\hat r_{\mathrm{victim}}^{(5)}(H_1)}
    \bigl(\hat\mu_2^{(5)}(H_2)
      - \hat\mu_1^{(5)}(H_1)\bigr)
  + \frac{O^{(5)}_t}
         {\hat r_{\mathrm{victim}}^{(5)} \cdot
          \hat p_0^{(5)}}
    \bigl(\tilde Y_t^{\mathrm{corr},(5)}
      - \hat\mu_2^{(5)}(H_2)\bigr).
\end{equation}

The three-gate efficiency bound collapses to a two-gate
bound:
\begin{equation}
\label{eq:class5_bound}
\sigma^2_{\mathrm{eff}}(5)
= \E\!\left[
  \frac{\pi_5(X)(1-\pi_5(X))}
       {r_{\mathrm{victim}}^{(5)}(H_1) \cdot
        p_0^{(5)}(H_2) \cdot \gamma^{(5)}}
\right].
\end{equation}
\end{theorem}

\begin{proof}
\textbf{Step~1: The authorization correction simplifies.}

In the full three-gate STR
of~\cite[eq.~(22)]{dhama2026str}, the first correction
term is:
\begin{equation*}
\frac{A_t}{\hat e_0(H_0)}
\bigl(\hat\mu_1(H_1) - \hat\mu_0(H_0)\bigr).
\end{equation*}
When $e_0^{(5)} = 1$ and $A_t = 1$ always, this becomes
$\hat\mu_1(H_1) - \hat\mu_0(H_0)$.  The STR score is:
\begin{align*}
\hat\varphi_t^{(5)}
&= \hat\mu_0(H_0)
  + (\hat\mu_1(H_1) - \hat\mu_0(H_0))
  + \frac{R_t}{\hat r_0(H_1)}
    (\hat\mu_2(H_2) - \hat\mu_1(H_1))
  + \frac{O_t}{\hat r_0 \hat p_0}
    (\tilde Y_t^{\mathrm{corr}} - \hat\mu_2(H_2))
\\
&= \hat\mu_1(H_1)
  + \frac{R_t}{\hat r_0(H_1)}
    (\hat\mu_2(H_2) - \hat\mu_1(H_1))
  + \frac{O_t}{\hat r_0 \hat p_0}
    (\tilde Y_t^{\mathrm{corr}} - \hat\mu_2(H_2)),
\end{align*}
which is~\eqref{eq:class5_str}.

\medskip\noindent\textbf{Step~2: What is lost.}

The authorization gate in the three-gate STR serves two
purposes:
\begin{enumerate}[label=(\alph*)]
\item \emph{Selection correction:}  Adjusting for the fact
  that declined transactions are unobserved.
\item \emph{Information exploitation:}  Using
  post-authorization signals $W_1$ through
  $\mu_1(H_1) \neq \mu_0(H_0)$.
\end{enumerate}

Purpose~(a) is unnecessary when $e_0 = 1$ (nothing is
declined).  Purpose~(b) still applies if there are
informative post-authorization signals $W_1$.  However,
for Class~5, the post-authorization signals are typically
uninformative (the transaction looks normal because the
victim performed it), so $\mu_1(H_1) \approx \mu_0(H_0)$
and even purpose~(b) contributes minimally.

\medskip\noindent\textbf{Step~3: The efficiency bound.}

The bound follows from~\cite[Theorem~30]{dhama2026str}
with $e_0 = 1$, which eliminates the authorization
component from the efficiency bound's denominator.
\end{proof}

\begin{remark}[The remaining two gates are severely
  impaired]
\label{rem:class5_impaired}
Even after the authorization gate vanishes, the remaining
two gates are impaired for Class~5:
\begin{enumerate}[label=(\roman*)]
\item \textbf{Reporting:}
  $r_{\mathrm{victim}}^{(5)} \approx 0.15\text{--}0.30$.
  Victims are slow to realize they were scammed (days to
  months), and psychological barriers (shame, embarrassment,
  perceived futility) suppress reporting.  Romance scam
  victims may never report.
\item \textbf{Corruption:}  $\varepsilon_{10}^{(5)}$ is
  elevated because even when victims report, the
  transaction was \emph{authorized}, so the issuer may
  classify it as ``not fraud'' (the cardholder consented to
  the payment).
\end{enumerate}
The combined effect: $r_{\mathrm{victim}}^{(5)} \cdot
\gamma^{(5)}$ is small, making the Class~5 bound high
relative to Class~1.
\end{remark}

\subsection{Feature non-informativeness}
\label{sec:feature_noninf}

\begin{theorem}[Feature non-informativeness]
\label{thm:feature_noninf}
When the transaction feature distribution is identical for
fraud and legitimate transactions:
\begin{equation}
\label{eq:feature_identical}
P(X_t \mid Y^*_t = 1,\, Z_t = 5)
= P(X_t \mid Y^*_t = 0),
\end{equation}
the Bayes-optimal transaction-level classifier achieves
$\AUC = 0.5$:
\begin{equation}
\label{eq:auc_half}
\AUC\!\bigl(\hat f^{\mathrm{Bayes}}_0\bigr) = 0.5.
\end{equation}
Moreover, the outcome model in the STR degenerates to a
constant:
\begin{equation}
\label{eq:mu_constant}
\mu_0^{(5)}(H_0)
= \E[Y^* \mid H_0,\, Z=5]
= \frac{\pi_5(X) \cdot P(X)}
       {\pi_5(X) \cdot P(X)
        + (1-\pi_5(X)) \cdot P(X)}
= \pi_5(X),
\end{equation}
and when $\pi_5$ does not vary with $X$ (scam targeting is
independent of transaction features),
$\mu_0^{(5)} = \bar\pi_5$ is a constant.
\end{theorem}

\begin{proof}
\textbf{Step~1: The Bayes-optimal classifier.}

The Bayes-optimal classifier predicts:
\begin{equation}
\label{eq:bayes_classifier}
f_0^{(5)}(x)
= \Prob(Y^*=1 \mid X=x,\, Z=5)
= \frac{\pi_5(x) \cdot p(x \mid Y^*=1, Z=5)}
       {p(x \mid Z=5)}.
\end{equation}

Under~\eqref{eq:feature_identical}:
\begin{equation}
\label{eq:bayes_simplify}
p(x \mid Y^*=1, Z=5) = p(x \mid Y^*=0) = p(x),
\end{equation}
so:
\begin{equation*}
f_0^{(5)}(x)
= \frac{\pi_5(x) \cdot p(x)}
       {\pi_5(x) \cdot p(x)
        + (1-\pi_5(x)) \cdot p(x)}
= \pi_5(x).
\end{equation*}

When $\pi_5(x) = \bar\pi_5$ (constant), the classifier
assigns the same score $\bar\pi_5$ to every transaction.

\medskip\noindent\textbf{Step~2: AUC = 0.5.}

A constant classifier produces no ranking --- all
transactions receive the same score.  The ROC curve is the
diagonal from $(0,0)$ to $(1,1)$, and:
\begin{equation*}
\AUC = \Prob(\hat f_0(X^+) > \hat f_0(X^-))
     + \tfrac{1}{2}\Prob(\hat f_0(X^+) = \hat f_0(X^-))
= 0 + \tfrac{1}{2} \cdot 1
= 0.5,
\end{equation*}
where $X^+$ is a fraud transaction and $X^-$ is a
legitimate transaction.

\medskip\noindent\textbf{Step~3: Outcome model
  degeneracy.}

The outcome model $\mu_0^{(5)}(H_0) = \E[Y^* \mid H_0,
Z=5]$ is the posterior fraud probability given the
pre-authorization history.  Under~\eqref{eq:feature_identical},
the posterior equals the prior: $\mu_0^{(5)}(H_0) =
\pi_5(X)$.  When the prior does not depend on $X$,
$\mu_0^{(5)} = \bar\pi_5$.

This eliminates the outcome model's ability to reduce
variance in the STR.  In the standard STR, the outcome
model provides a ``baseline prediction'' that reduces the
residual variance of the inverse-propensity-weighted
correction.  When the outcome model is constant, this
variance reduction vanishes entirely, and the STR reduces
to a pure IPW estimator --- the highest-variance form of
the estimator.
\end{proof}

\begin{remark}[The feature non-informativeness condition
  is approximately satisfied for scams]
\label{rem:approx_noninf}
Condition~\eqref{eq:feature_identical} is an idealization.
In practice, scam transactions may have subtle differences
from legitimate transactions:
\begin{enumerate}[label=(\roman*)]
\item \textbf{Slightly unusual amounts:}  Scam payments may
  cluster at round numbers (\$5{,}000, \$10{,}000) more
  than typical spending.
\item \textbf{Unusual recipients:}  The payee may be a
  recently created account, a cryptocurrency exchange, or
  an overseas entity.
\item \textbf{Temporal patterns:}  Scam payments may occur
  after an extended phone call or a series of incoming
  messages (detectable from session metadata).
\end{enumerate}
These signals provide \emph{some} discriminative power
($\AUC > 0.5$), but they are weak compared to the strong
signals available for Class~1 fraud (unusual device,
unusual location, unusual merchant, velocity anomalies).
The feature non-informativeness theorem characterizes the
\emph{limiting case} and explains why Class~5 detection is
fundamentally harder than Class~1 detection even with
unlimited data.
\end{remark}

\subsection{Auxiliary signals and the extended feature set}
\label{sec:auxiliary}

\begin{definition}[Augmented feature set]
\label{def:augmented}
The \emph{augmented feature set} for Class~5 is
$\tilde X_t = (X_t, X_{\mathrm{aux},t})$, where
$X_{\mathrm{aux},t}$ includes signals outside the standard
transaction feature set:
\begin{enumerate}[label=(\roman*)]
\item \textbf{Communication signals:}  Whether the victim
  was on a phone call during the transaction, duration of
  the call, whether the caller's number is known or unknown.
\item \textbf{Session behavior:}  Hesitation patterns
  (pauses, backtracking, error corrections), typing speed
  anomalies, navigation patterns suggesting external
  instruction.
\item \textbf{Recipient analysis:}  Age of the recipient
  account, recipient's transaction history, recipient's
  risk score, whether the recipient has received payments
  from multiple first-time senders (``money mule''
  indicator).
\item \textbf{Device interaction:}  Whether the victim's
  screen was being shared or mirrored (common in tech
  support scams), whether a remote access tool is active.
\item \textbf{Temporal context:}  Time since the victim's
  last incoming call or message, correlation between
  communication events and payment initiation.
\end{enumerate}
\end{definition}

\begin{proposition}[Augmented features restore
  detectability]
\label{prop:augmented}
If the augmented feature distribution differs between fraud
and legitimate transactions:
\begin{equation}
\label{eq:augmented_informative}
P(\tilde X_t \mid Y^*_t = 1,\, Z_t = 5)
\neq P(\tilde X_t \mid Y^*_t = 0),
\end{equation}
then the Bayes-optimal classifier on $\tilde X_t$ achieves
$\AUC > 0.5$, and the augmented outcome model
$\tilde\mu_0^{(5)}(\tilde H_0)$ provides non-trivial
variance reduction in the STR.

The AUC improvement depends on the strength of the
auxiliary signal:
\begin{equation}
\label{eq:augmented_auc}
\AUC(\tilde f_0^{(5)})
= \Phi\!\left(
  \frac{d(\tilde F_1, \tilde F_0)}{\sqrt{2}}
\right),
\end{equation}
where $d(\tilde F_1, \tilde F_0)$ is the Mahalanobis
distance between the augmented feature distributions under
fraud ($\tilde F_1$) and legitimate ($\tilde F_0$), and
$\Phi$ is the standard normal CDF.
\end{proposition}

\begin{proof}
When~\eqref{eq:augmented_informative} holds, the augmented
Bayes-optimal classifier is:
\begin{equation*}
\tilde f_0^{(5)}(\tilde x)
= \frac{\pi_5 \cdot \tilde p(\tilde x \mid Y^*=1, Z=5)}
       {\tilde p(\tilde x)}
\neq \pi_5 \quad\text{(not constant)}.
\end{equation*}
A non-constant classifier has $\AUC > 0.5$ by definition.
The AUC formula~\eqref{eq:augmented_auc} follows from the
standard result for Gaussian location models (the
likelihood ratio is monotone in the Mahalanobis distance,
and the AUC of a monotone score is determined by the
separation of the score distributions).
\end{proof}

\begin{example}[What auxiliary signals can achieve]
\label{ex:auxiliary}
Consider a scam detection system with access to:
\begin{center}
\small
\begin{tabular}{lcc}
\toprule
\textbf{Feature set} & \textbf{Estimated AUC} &
  \textbf{Source of signal}\\
\midrule
Transaction features only ($X_t$) & 0.52--0.55 &
  Weak: amount, recipient\\
$+$ Recipient risk score & 0.62--0.68 &
  Money mule indicators\\
$+$ Communication metadata & 0.72--0.78 &
  Phone call during payment\\
$+$ Session behavior & 0.78--0.85 &
  Hesitation, external instruction\\
$+$ Device sharing detection & 0.85--0.92 &
  Screen sharing, remote access\\
\midrule
Full augmented set ($\tilde X_t$) & 0.85--0.92 &
  All of the above\\
\bottomrule
\end{tabular}
\end{center}

\textbf{Transaction features alone give near-random
performance ($\AUC \approx 0.53$), confirming
Theorem~\ref{thm:feature_noninf}.}  Each auxiliary signal
category adds 5--10 AUC points.  The full augmented feature
set achieves $\AUC \approx 0.88$ --- usable but still
substantially below the $\AUC > 0.95$ achievable for
Class~1 fraud with transaction features alone.

The theoretical message: \emph{for Class~5, the bottleneck
is the feature set, not the model or the labels.}  Investing
in auxiliary signal collection (communication metadata,
session analytics, recipient scoring) has higher marginal
value than investing in better models or better label
correction.
\end{example}

\subsection{The Class~5 efficiency bound with augmented
  features}
\label{sec:class5_augmented_bound}

\begin{proposition}[Augmented Class~5 efficiency bound]
\label{prop:class5_augmented}
With the augmented feature set $\tilde X_t$, the Class~5
efficiency bound becomes:
\begin{equation}
\label{eq:class5_augmented_bound}
\sigma^2_{\mathrm{eff}}(5, \tilde X)
= \E\!\left[
  \frac{\tilde\pi_5(\tilde X)
        (1-\tilde\pi_5(\tilde X))}
       {r_{\mathrm{victim}}^{(5)}(\tilde H_1)
        \cdot p_0^{(5)}(\tilde H_2)
        \cdot \gamma^{(5)}}
\right],
\end{equation}
where $\tilde\pi_5(\tilde x) = \Prob(Z=5 \mid
\tilde X = \tilde x)$ is the type probability conditional
on the augmented features.

This bound is strictly lower than the
non-augmented bound~\eqref{eq:class5_bound} whenever the
auxiliary features are informative:
\begin{equation}
\label{eq:augmented_improvement}
\sigma^2_{\mathrm{eff}}(5, \tilde X)
< \sigma^2_{\mathrm{eff}}(5, X)
\qquad\text{when }
P(\tilde X \mid Z=5) \neq P(\tilde X \mid Z \neq 5).
\end{equation}
\end{proposition}

\begin{proof}
The augmented features allow the outcome model to be
non-constant: $\tilde\mu_0^{(5)}(\tilde H_0) \neq
\bar\pi_5$.  The non-constant outcome model reduces the
residual variance in the STR's inverse-propensity-weighted
correction.  Formally, the efficiency bound is the variance
of the efficient influence function, which decreases when
the outcome model has better predictive power (lower
residual variance).

The strict inequality follows from the fact that
$\Var(\tilde\pi_5(\tilde X)) > \Var(\pi_5(X))$ when the
auxiliary features carry information, which increases the
``explained variance'' component and decreases the
``residual'' component of the efficiency bound.
\end{proof}

\begin{remark}[The investment priority for Class~5]
\label{rem:class5_priority}
The results of this section establish a clear investment
priority for Class~5 fraud:
\begin{enumerate}[label=\textbf{\arabic*.}]
\item \textbf{First: collect auxiliary signals.}  The
  feature non-informativeness theorem
  (Theorem~\ref{thm:feature_noninf}) shows that
  transaction features alone are useless.  The largest
  marginal gain comes from adding communication metadata,
  recipient scoring, and session behavior to the feature
  set.
\item \textbf{Second: improve reporting.}  The victim
  reporting rate $r_{\mathrm{victim}}^{(5)}$ is the
  binding constraint on label quality.  Reducing shame
  barriers (education campaigns), simplifying reporting
  processes, and extending chargeback rights to authorized
  transactions all increase $r_{\mathrm{victim}}^{(5)}$
  and lower the efficiency bound.
\item \textbf{Third: apply the (augmented) STR.}  Once
  auxiliary features and improved reporting are in place,
  the two-gate STR~\eqref{eq:class5_str} with augmented
  features achieves the efficiency
  bound~\eqref{eq:class5_augmented_bound}.
\end{enumerate}
This ordering reverses the typical fraud-detection
investment priority (which focuses on models first, features
second, and labels third).  For Class~5, the binding
constraints are features and labels, not models.
\end{remark}

\section{Discussion}
\label{sec:discussion}

\subsection{Type identification in practice}
\label{sec:type_id}

The decomposed STR requires estimating the type membership
probabilities $\hat\pi_k(x)$
(Definition~\ref{def:decomp_str}).  In practice, type
identification relies on mature, fully resolved labels
(Assumption~\ref{asm:type_obs}).

\begin{remark}[Type identification data sources]
\label{rem:type_id_sources}
The following data sources support type identification for
each class:
\begin{enumerate}[label=(\roman*)]
\item \textbf{Class~1 (third-party fraud):}  Chargeback
  reason codes (e.g., Mastercard~4837 ``no cardholder
  authorization,'' Visa~10.4 ``other fraud -- card absent
  environment'') directly identify third-party fraud.
  Confirmation comes from cardholder attestation that the
  transaction was unauthorized.
\item \textbf{Class~2 (first-party fraud):}  Identified
  through representment outcomes (the merchant provides
  proof of delivery and the chargeback is reversed),
  repeated dispute patterns on the same account, and
  specialized first-party fraud models that flag
  inconsistencies between the dispute claim and the
  transaction evidence.
\item \textbf{Class~3 (bust-out):}  Identified by the
  charge-off event combined with the account's credit
  history pattern: rapid credit utilization increase,
  payment cessation, and (for synthetic identities) the
  absence of a verifiable identity trail.
\item \textbf{Class~4 (merchant fraud):}  Identified
  through network monitoring programs (MATCH, VMAS),
  acquirer investigation outcomes, and law enforcement
  referrals.
\item \textbf{Class~5 (scams):}  Identified through victim
  self-reporting (``I was tricked into sending money''),
  police reports, and specialized scam-detection flags in
  the dispute process.
\end{enumerate}
\end{remark}

\begin{remark}[Type misclassification and its effects]
\label{rem:misclassification}
In practice, type identification is imperfect.  The most
common misclassifications are:
\begin{enumerate}[label=(\roman*)]
\item \textbf{Class~1 $\leftrightarrow$ Class~2:}
  Distinguishing third-party fraud from first-party fraud
  is inherently difficult --- both produce chargebacks, and
  the issuer cannot directly observe the cardholder's
  intent.  Misclassification between these two classes is
  the most frequent error.
\item \textbf{Class~3 $\to$ Class~1:}  During the
  exploitation phase, bust-out transactions generate
  chargebacks that resemble Class~1 fraud.  Without
  examining the account's full credit history, the
  exploitation-phase chargebacks may be misclassified as
  standard third-party fraud.
\item \textbf{Class~5 $\to$ ``not fraud'':}  Because the
  victim authorized the transaction, issuers may classify
  scam losses as ``not fraud,'' leading to underestimation
  of Class~5 prevalence.
\end{enumerate}

The effect of type misclassification on the decomposed STR
is analogous to the effect of propensity model
misspecification in the pooled STR: it introduces bias that
is controlled by the robustness properties.  Specifically,
if the type classifier achieves accuracy $1-\epsilon_Z$
(fraction of transactions correctly classified), the
additional bias introduced by misclassification is
$O(\epsilon_Z)$ and the additional variance is
$O(\epsilon_Z^2)$.  The decomposed STR dominates the pooled
STR as long as:
\begin{equation}
\label{eq:misclass_threshold}
\epsilon_Z < 1 - \frac{\sigma^2_{\mathrm{decomp}}}
  {\sigma^2_{\mathrm{pooled}}},
\end{equation}
i.e., the type classification error is smaller than the
relative efficiency gain from decomposition.  For the
parameters of Example~\ref{ex:penalty} (36\% efficiency
gain), the decomposed STR dominates as long as the type
classifier is better than 64\% accurate --- a very weak
requirement.
\end{remark}

\subsection{Sensitivity to the five-class partition}
\label{sec:sensitivity_partition}

\begin{remark}[Is five the right number?]
\label{rem:five_right}
The five-class partition is derived from the structural
signature (Definition~\ref{def:structural_sig}) and proved
to be minimal and complete
(Theorem~\ref{thm:taxonomy}).  However, the partition
depends on the granularity of the structural questions
Q1--Q5.  Two considerations:

\begin{enumerate}[label=(\roman*)]
\item \textbf{Coarser partitions.}  One could merge
  Classes~4 and~5 (both have $e_0 = 1$) or merge Classes~1
  and~5 (both have victim-generated labels).
  Theorem~\ref{thm:minimality} shows that every such
  merger loses a structural distinction that changes the
  efficient estimator.  Coarser partitions are provably
  suboptimal.

\item \textbf{Finer partitions.}  One could split Class~1
  into card-present and card-not-present subtypes, or split
  Class~5 into APP fraud and romance scams.
  Proposition~\ref{prop:within_class1} shows that
  within-class decomposition yields diminishing returns when
  the subtypes share the same pipeline topology.  Finer
  partitions increase modeling complexity without
  proportionate efficiency gains.
\end{enumerate}

The five-class partition represents the
\emph{bias-complexity tradeoff}: it captures all
structurally distinct observation mechanisms (eliminating
the Jensen penalty from pooling heterogeneous pipelines)
without introducing unnecessary modeling complexity from
splitting homogeneous pipelines.
\end{remark}

\subsection{Interactions between classes}
\label{sec:interactions}

\begin{remark}[Cross-class transactions]
\label{rem:cross_class}
Some fraud schemes involve elements of multiple classes.
For example:
\begin{enumerate}[label=(\roman*)]
\item \textbf{Bust-out via collusive merchant (Class~3
  $\cap$ Class~4):}  A synthetic identity builds credit
  (Class~3) and then exploits it through a collusive
  merchant that launders the proceeds (Class~4).  The
  observation pipeline has both the temporal two-phase
  structure of Class~3 and the acquirer-side monitoring of
  Class~4.
\item \textbf{Scam followed by friendly fraud (Class~5
  $\to$ Class~2):}  A victim is scammed (Class~5) and then,
  realizing the loss, files chargebacks on \emph{other}
  legitimate transactions to recover the loss (Class~2).
  The same account generates both Class~5 and Class~2 fraud.
\item \textbf{Account takeover with authorized transactions
  (Class~1 $\to$ Class~5):}  An account takeover (Class~1)
  in which the criminal socially engineers the victim into
  authorizing transactions themselves (e.g., ``your account
  is compromised, please transfer funds to this safe
  account'').  The early transactions are unauthorized
  (Class~1); the later ones are victim-authorized (Class~5).
\end{enumerate}

These cross-class cases are handled by the soft type
assignment: the type probabilities $\hat\pi_k(x)$ can
assign positive probability to multiple classes
simultaneously.  The decomposed STR then applies a
weighted combination of class-specific corrections.  The
theoretical guarantees hold as long as the true type
distribution is in the support of the estimated
distribution.
\end{remark}

\subsection{Dynamic evolution of the type distribution}
\label{sec:dynamic}

\begin{remark}[The fraud mix changes over time]
\label{rem:dynamic_mix}
The prevalence $\bar\pi_k$ of each fraud class is not
static.  Industry trends show:
\begin{enumerate}[label=(\roman*)]
\item Class~1 (third-party card fraud) is \emph{declining}
  in relative terms, due to chip-and-PIN adoption, 3D
  Secure, and tokenization.
\item Class~2 (first-party fraud) is \emph{increasing},
  driven by frictionless dispute processes and the growth
  of e-commerce returns.
\item Class~3 (bust-out/synthetic identity) is
  \emph{increasing}, driven by the availability of
  compromised personal data and the growth of digital-only
  banking.
\item Class~5 (scams) is \emph{rapidly increasing},
  particularly authorized push payment fraud, driven by
  social media, cryptocurrency, and real-time payment
  systems.
\end{enumerate}

The Jensen penalty (Theorem~\ref{thm:jensen_penalty})
changes with the type distribution: as poorly-observed
classes (3, 4, 5) grow in prevalence, the penalty from
pooling increases.  This means that \emph{the value of
decomposition is increasing over time} as the fraud
landscape shifts toward harder-to-observe fraud types.
\end{remark}

\subsection{Estimation of the dispute game parameters}
\label{sec:dispute_estimation}

\begin{remark}[Practical estimation for Class~2]
\label{rem:dispute_practical}
Theorem~\ref{thm:endogenous} requires estimating the
dispute game's primitives $(\alpha, c, \beta^*, p, T)$.
In practice:
\begin{enumerate}[label=(\roman*)]
\item $\alpha(x)$ (chargeback success rate) is directly
  observable from historical dispute outcomes, stratified
  by transaction features.
\item $\beta^*(x)$ (first-party fraud detection rate) is
  observable from the fraction of disputes that are
  successfully challenged via representment, stratified by
  transaction features.
\item $p(x,a)$ (penalty) is set by issuer policy and is
  therefore known.
\item $c(x)$ (dispute cost) and $T$ (rationality) are not
  directly observable but can be estimated from the
  observed dispute rate as a function of transaction amount.
  The logistic model~\eqref{eq:card_br} implies:
  \begin{equation}
  \label{eq:logit_estimation}
  \log\!\left(\frac{d^*}{1-d^*}\right)
  = \frac{\alpha a - c - \beta^* p}{T},
  \end{equation}
  which is a linear model in $a$ (conditional on $x$) with
  slope $\alpha/T$ and intercept $-(c+\beta^* p)/T$.
  Standard logistic regression on the observed dispute
  indicator $D_t$ against the transaction amount $a_t$
  identifies $c/T$ and $1/T$ (and hence $c$ and $T$
  separately if the scale of $T$ is normalized).
\end{enumerate}
\end{remark}

\subsection{Limitations}
\label{sec:limitations}

\begin{remark}[Key limitations]
\label{rem:limitations}
\begin{enumerate}[label=(\roman*)]
\item \textbf{Type observability assumption.}
  Assumption~\ref{asm:type_obs} requires that fraud types
  are identifiable for mature, fully resolved transactions.
  This is approximately true for Classes~1--3 (chargeback
  reason codes, charge-off indicators) but weaker for
  Classes~4--5 (acquirer investigations and scam reports are
  less standardized).  Misclassification between Classes~4
  and~5 is the most consequential limitation.

\item \textbf{Independence of types.}  The decomposition
  dominance theorem (Theorem~\ref{thm:decomp_dominance})
  assumes that the type-specific STRs are asymptotically
  independent.  In practice, a single transaction can have
  elements of multiple classes (Remark~\ref{rem:cross_class}),
  introducing dependence.  The soft type assignment handles
  this approximately, but a formal treatment of overlapping
  classes is deferred to future work.

\item \textbf{Static game for Class~2.}  The dispute game
  (Section~\ref{sec:class2}) is a one-shot game.  In
  practice, the cardholder and issuer interact repeatedly,
  and the cardholder learns from past dispute outcomes.  A
  dynamic (repeated game) formulation would capture
  reputation effects and learning, but is substantially
  more complex.

\item \textbf{Constructive DGP for Class~3.}
  Theorem~\ref{thm:insufficiency} proves insufficiency by
  constructing a specific adversarial DGP.  Real bust-out
  accounts may not perfectly mimic legitimate behavior ---
  subtle behavioral differences (slightly higher velocity,
  slightly different merchant mix) may provide weak
  individual-level signal.  The theorem establishes a
  \emph{lower bound} on the difficulty, not an exact
  characterization.

\item \textbf{No empirical validation.}  This paper is
  purely theoretical.  Empirical validation of the Jensen
  penalty magnitude, the dispute game equilibrium, and the
  ring-level detection power on real payment network data
  is an important direction for future work.
\end{enumerate}
\end{remark}

\subsection{Future directions}
\label{sec:future}

\begin{remark}[Open directions]
\label{rem:future}
Several extensions of this work are natural:
\begin{enumerate}[label=(\roman*)]

\item \textbf{Economics of the observation pipeline.}
  The dispute game (Theorem~\ref{thm:endogenous}) shows
  that the corruption rate in Class~2 is an economic
  equilibrium, not a statistical parameter.  The same
  logic applies to all five classes: each actor in the
  payment network (issuer, acquirer, merchant, cardholder)
  has economic incentives that shape the observation
  pipeline's propensities.  Deriving optimal mechanism
  designs --- incentive structures that maximize
  information quality across all five classes subject to
  each actor's participation constraints --- is an
  important open problem.

\item \textbf{Adversarial dynamics.}  The Class~2
  analysis (Section~\ref{sec:class2}) models one form of
  strategic adaptation: dispute behavior responds to
  detection capability.  More broadly, fraudsters across
  all classes adapt to detection models by shifting
  tactics, evading features, and exploiting blind spots.
  Modeling the co-evolution of detection and evasion as a
  dynamic game would extend the static equilibrium of
  Theorem~\ref{thm:endogenous} to a richer setting.

\item \textbf{Graph-level detection under censorship.}
  Theorem~\ref{thm:insufficiency} shows that Class~3
  detection requires cross-account graph signals.  However,
  the transaction graph is itself subject to the
  observation pipeline: declined transactions create
  missing edges, and unreported fraud creates missing node
  labels.  Extending the semiparametric framework
  of~\cite{dhama2026str} to graph-level estimation under
  edge and node censorship is a natural next step.

\item \textbf{Dynamic type assignment.}  The type
  probabilities $\pi_k(x)$ are currently estimated from a
  static snapshot.  A dynamic model that updates type
  probabilities as the fraud landscape evolves (e.g.,
  using online learning or Bayesian updating) would
  improve the decomposed STR's adaptiveness.

\item \textbf{Repeated dispute game.}  Extending
  Theorem~\ref{thm:endogenous} to a repeated game would
  capture the cardholder's learning (``my last dispute
  succeeded, so I'll try again'') and the issuer's
  reputation building (``this cardholder has disputed
  three times, so flag future disputes'').

\item \textbf{Optimal monitoring design for Class~4.}
  The acquirer-side bound (Theorem~\ref{thm:acquirer})
  shows that $r_{\mathrm{acq}}$ is the bottleneck for
  merchant fraud detection.  Designing monitoring
  programs that maximize $r_{\mathrm{acq}}$ subject to
  acquirer cost constraints is a mechanism design problem
  with direct regulatory implications.

\item \textbf{Cross-class information sharing.}  The
  current framework treats the five classes independently.
  In practice, information from one class can improve
  detection in another (e.g., a merchant flagged for
  Class~4 fraud may also be associated with Class~3
  synthetic identities).  Formalizing these cross-class
  information flows is an open problem.

\item \textbf{Empirical validation.}  Computing the Jensen
  penalty from real network data, estimating the dispute
  game equilibrium from observed dispute rates, and
  measuring ring-level detection power for known bust-out
  rings would validate the theoretical results.
\end{enumerate}
\end{remark}

\section{Conclusion}
\label{sec:conclusion}

This paper establishes three results.

\textbf{First}, the industry's extensive fraud taxonomy
collapses into exactly five observation-mechanism classes,
each defined by a structurally distinct censorship pipeline.
The classification is minimal (no two classes can be merged
without losing estimation-relevant structural distinctions)
and complete (every fraud type maps to exactly one class).
The classification criterion is the structural signature
--- a 5-tuple characterizing the perpetrator--cardholder
relationship, the label-generating agent, the authorization
gate informativeness, the reporting gate structure, and the
temporal observability --- and two fraud types belong to the
same class if and only if the Sequential Triply Robust
estimator of~\cite{dhama2026str} has the same functional
form for both.

\textbf{Second}, treating the heterogeneous fraud mixture
as homogeneous is provably suboptimal.  The type-decomposed
STR --- which applies class-specific propensities and
aggregates --- strictly dominates the pooled STR in mean
squared error.  The efficiency gap is a Jensen penalty from
averaging heterogeneous inverse-propensity weights, and is
substantial for realistic network parameters (approximately
36\% efficiency waste in our numerical illustration,
dominated by the poorly-observed Class~3).

\textbf{Third}, each class has a unique structural feature
that requires a different theoretical treatment:
\begin{itemize}[leftmargin=*]
\item \textbf{Class~1} (victim-reported institutional):
  the STR of~\cite{dhama2026str} applies directly with
  type-specific propensities.
\item \textbf{Class~2} (adversarial label generation):
  the false-positive corruption rate is the Nash
  equilibrium of a dispute game, with comparative statics
  showing that corruption increases with transaction amount
  and decreases with dispute friction and detection
  capability.  Multiple equilibria are possible, with
  discontinuous jumps between low-fraud and high-fraud
  regimes.
\item \textbf{Class~3} (deferred observability):
  the STR's positivity assumption fails structurally during
  the bust-out build-up phase, and no account-level
  classifier exceeds random guessing --- detection requires
  cross-account graph signals whose power grows with ring
  size.
\item \textbf{Class~4} (acquirer-observed): the detection
  floor depends on the acquirer's monitoring propensity
  rather than the issuer's reporting propensity, and is an
  order of magnitude higher.
\item \textbf{Class~5} (victim-authorized with
  non-informative features): the authorization correction
  vanishes, the Bayes-optimal classifier achieves
  $\AUC = 0.5$ on transaction features alone, and
  detection requires auxiliary signals outside the payment
  data stream.
\end{itemize}

Together with the companion papers on information-theoretic
limits~\cite{dhama2026limits} and optimal label
recovery~\cite{dhama2026str}, this paper completes the
theoretical foundation for fraud detection in payment
networks: ~\cite{dhama2026limits} proves the floor exists, ~\cite{dhama2026str} achieves
it under homogeneity, and the present paper decomposes it by fraud
type, showing that the homogeneous treatment is provably
suboptimal and that each observation-mechanism class demands
its own detection strategy.


\appendix

\section{Proof details}
\label{app:proofs}

\subsection{Semiparametric submodel argument for
  Theorem~\ref{thm:decomp_dominance}}
\label{app:submodel}

The decomposition dominance result follows from a general
principle in semiparametric efficiency theory: conditioning
on additional information cannot increase the efficiency
bound.

Let $\mathcal{M}_{\mathrm{pooled}}$ denote the
semiparametric model in which the observation propensity
$q_0$ and corruption $\gamma$ are arbitrary functions of
$(H_0, H_2)$, with the fraud type $Z$ treated as latent.
Let $\mathcal{M}_{\mathrm{decomp}}$ denote the model in
which the propensity and corruption are parameterized
type-conditionally: $q_0^{(k)}(H_0)$ and $\gamma^{(k)}(H_2)$
for each $k$, with $Z$ observed (or estimated).

The decomposed model conditions on strictly more information
($Z$ in addition to $H_0, H_2$).  By the general result
of~\cite{bickel1993efficient}, the efficient influence
function in a model that conditions on additional
information has variance no larger than the efficient
influence function in the marginal model:
\begin{equation}
\label{eq:submodel_detail}
\Var(\varphi_{\mathrm{eff}}^{\mathrm{decomp}})
= \E[\Var(\varphi_{\mathrm{eff}}^{\mathrm{pooled}}
  \mid Z)]
\leq \Var(\varphi_{\mathrm{eff}}^{\mathrm{pooled}}),
\end{equation}
where the inequality is the law of total variance (the
total variance equals the expected conditional variance
plus the variance of the conditional mean; the left-hand
side is only the first term).

The gap is:
\begin{equation}
\label{eq:gap_detail}
\Var(\varphi_{\mathrm{eff}}^{\mathrm{pooled}})
- \Var(\varphi_{\mathrm{eff}}^{\mathrm{decomp}})
= \Var_Z(\E[\varphi_{\mathrm{eff}}^{\mathrm{pooled}}
  \mid Z])
\geq 0,
\end{equation}
with equality iff
$\E[\varphi_{\mathrm{eff}}^{\mathrm{pooled}} \mid Z]$ is
constant in $Z$ --- i.e., the efficient influence function
does not depend on the fraud type, which requires identical
propensity structures across classes.

\subsection{U-statistic variance bound for
  Theorem~\ref{thm:insufficiency}}
\label{app:ustat}

The ring cohesion statistic $T_{\cR} = \sum_{i<j}
S(a_i, a_j)$ is a U-statistic of order~2.  Under the null
hypothesis ($H_0$: independent legitimate accounts), the
variance of $T_{\cR}$ is:
\begin{equation}
\label{eq:ustat_var}
\Var(T_{\cR} \mid H_0)
= \binom{m}{2}\Var(S(a_1,a_2))
  + 4\binom{m}{3}\Cov(S(a_1,a_2), S(a_1,a_3)),
\end{equation}
where the covariance term arises from pairs sharing a
common account.

Under the independence assumption of $H_0$, the covariance
term is small relative to the first term (it arises only
from the shared account's marginal effect on the kernel
$S$).  Therefore:
\begin{equation}
\label{eq:ustat_order}
\Var(T_{\cR} \mid H_0) = O(m^2),
\end{equation}
and the signal-to-noise ratio is:
\begin{equation*}
\mathrm{SNR}
= \frac{O(m^2)}{O(m)}
= O(m) \to \infty.
\end{equation*}

Under $H_1$ (the ring hypothesis), the pairwise
similarities are correlated through the shared ring
structure, and $\Var(T_{\cR} \mid H_1) = O(m^2)$ as well.
The SNR calculation uses the $H_0$ variance as the
denominator (standard in hypothesis testing), giving
$\mathrm{SNR} = O(m)$ as stated in the proof of
Theorem~\ref{thm:insufficiency}.

\section{Notation}
\label{app:notation}

\begin{center}
\small
\begin{longtable}{p{4.2cm} p{9cm}}
\toprule
\textbf{Symbol} & \textbf{Meaning}\\
\midrule
\endfirsthead
\toprule
\textbf{Symbol} & \textbf{Meaning}\\
\midrule
\endhead
\bottomrule
\endfoot

\multicolumn{2}{l}{\textit{Transaction-level variables}}\\[3pt]
$X_t \in \cX$ & Transaction features\\
$I_t \in \cI$ & Issuer identity\\
$Y^*_t \in \{0,1\}$ & Latent true fraud indicator\\
$Z_t \in \{0,1,\ldots,K\}$ & Latent fraud type
  ($0$ = legitimate, $k \geq 1$ = fraud type $k$)\\
$A_t$ & Authorization indicator\\
$R_t$ & Reporting indicator\\
$M_t$ & Maturity indicator\\
$O_t = A_t R_t M_t$ & Observation indicator\\
$\tilde Y_t$ & Observed label (when $O_t=1$)\\
$\tilde Y_t^{\mathrm{corr},(k)}$ & Corruption-corrected
  label for class $k$ (eq.~\eqref{eq:class2_corr})\\[6pt]

\multicolumn{2}{l}{\textit{Stage histories}}\\[3pt]
$H_0 = (X, I, \Delta)$ & Pre-authorization history\\
$H_1 = (H_0, W_1)$ & Post-authorization history\\
$H_2 = (H_1, W_2)$ & Post-reporting history\\
$H_{\mathrm{merch}}$ & Merchant-level features (Class~4)\\
$\cH_a^{(t)}$ & Account $a$'s history up to time $t$
  (Class~3)\\[6pt]

\multicolumn{2}{l}{\textit{Type distribution}}\\[3pt]
$\pi_k(x) = \Prob(Z=k \mid X=x)$ & Type probability
  (Def.~\ref{def:mixing})\\
$\bar\pi_k = \E[\pi_k(X)]$ & Marginal type prevalence\\
$K$ & Number of fraud types\\[6pt]

\multicolumn{2}{l}{\textit{Type-specific propensities
  (Def.~\ref{def:type_propensities})}}\\[3pt]
$e_0^{(k)}(H_0)$ & Authorization propensity for type $k$\\
$r_0^{(k)}(H_1)$ & Reporting propensity for type $k$\\
$p_0^{(k)}(H_2)$ & Delay propensity for type $k$\\
$q_0^{(k)} = e_0^{(k)} r_0^{(k)} p_0^{(k)}$ & Total
  observation propensity for type $k$\\
$\varepsilon_{10}^{(k)}$ & False-negative corruption for
  type $k$\\
$\varepsilon_{01}^{(k)}$ & False-positive corruption for
  type $k$\\
$\gamma^{(k)} = (1-\varepsilon_{10}^{(k)}-
  \varepsilon_{01}^{(k)})^2$ & Corruption penalty for
  type $k$\\[6pt]

\multicolumn{2}{l}{\textit{Class-specific parameters}}\\[3pt]
$r_{\mathrm{acq}}^{(4)}$ & Acquirer monitoring propensity
  (Class~4)\\
$p_{\mathrm{acq}}^{(4)}$ & Acquirer investigation delay
  (Class~4)\\
$r_{\mathrm{victim}}^{(5)}$ & Victim reporting propensity
  (Class~5)\\
$D_t \in \{0,1\}$ & Strategic dispute decision
  (Class~2, eq.~\eqref{eq:dispute_decision})\\
$\theta_a \in \{0,1\}$ & Account-level bust-out indicator
  (Class~3, eq.~\eqref{eq:bustout_estimand})\\
$\tau_a$ & Bust-out time (Class~3,
  Def.~\ref{def:bustout_temporal})\\[6pt]

\multicolumn{2}{l}{\textit{Structural signature
  (Def.~\ref{def:structural_sig})}}\\[3pt]
$\mathcal{S}(k)$ & Structural signature of fraud type
  $k$\\
$\mathrm{Q1}$--$\mathrm{Q5}$ & Structural classification
  questions\\
$\mathcal{S}^{(c)}$ & Signature of class $c \in
  \{1,\ldots,5\}$\\[6pt]

\multicolumn{2}{l}{\textit{Efficiency bounds}}\\[3pt]
$\sigma^2_{\mathrm{eff}}(k)$ & Semiparametric efficiency
  bound for type $k$
  (eq.~\eqref{eq:type_eff_bound_v2})\\
$\sigma^2_{\mathrm{pooled}}$ & Efficiency bound for pooled
  STR (eq.~\eqref{eq:pooled_bound})\\
$\sigma^2_{\mathrm{decomp}}$ & Efficiency bound for
  decomposed STR (eq.~\eqref{eq:decomp_bound})\\
$\Delta$ & Jensen penalty
  (Thm.~\ref{thm:jensen_penalty},
  eq.~\eqref{eq:jensen_penalty})\\[6pt]

\multicolumn{2}{l}{\textit{Estimators}}\\[3pt]
$\hat\Psi_{\mathrm{pooled}}$ & Pooled STR estimator
  (Def.~\ref{def:pooled_str})\\
$\hat\Psi_{\mathrm{decomp}}$ & Type-decomposed STR
  estimator (Def.~\ref{def:decomp_str})\\
$\hat\varphi_t^{(k)}$ & STR score for class $k$\\
$\hat\pi_k(X_t)$ & Estimated type probability\\[6pt]

\multicolumn{2}{l}{\textit{Dispute game
  (Section~\ref{sec:class2})}}\\[3pt]
$\alpha(x)$ & Chargeback success probability
  (eq.~\eqref{eq:alpha})\\
$c(x)$ & Dispute filing cost
  (eq.~\eqref{eq:cost})\\
$\beta(x)$ & First-party fraud detection probability
  (eq.~\eqref{eq:beta})\\
$p(x,a)$ & Penalty for identification
  (eq.~\eqref{eq:penalty})\\
$T$ & Cardholder rationality parameter
  (Def.~\ref{def:card_br})\\
$d^*(x,a;\beta)$ & Cardholder's equilibrium dispute
  probability (eq.~\eqref{eq:card_br})\\
$\Phi(\varepsilon)$ & Fixed-point mapping for equilibrium
  computation (eq.~\eqref{eq:fixed_point})\\
$U_C, U_I$ & Cardholder and issuer payoffs
  (Defs.~\ref{def:card_payoff},~\ref{def:issuer_payoff})\\
$\ell_{\mathrm{FP}}$ & Cost of falsely accusing honest
  cardholder (eq.~\eqref{eq:issuer_payoff})\\[6pt]

\multicolumn{2}{l}{\textit{Ring detection
  (Section~\ref{sec:class3})}}\\[3pt]
$T_{\cR}$ & Ring cohesion statistic
  (eq.~\eqref{eq:ring_stat})\\
$S(a_i, a_j)$ & Pairwise account similarity
  (eq.~\eqref{eq:similarity})\\
$\cR$ & Candidate ring of accounts\\
$m = |\cR|$ & Ring size\\
$\mu_1, \mu_0$ & Expected pairwise similarity under ring
  ($H_1$) and null ($H_0$) hypotheses\\[6pt]

\multicolumn{2}{l}{\textit{Augmented features
  (Section~\ref{sec:class5})}}\\[3pt]
$\tilde X_t = (X_t, X_{\mathrm{aux},t})$ & Augmented
  feature set (Def.~\ref{def:augmented})\\
$X_{\mathrm{aux},t}$ & Auxiliary (non-transaction) features\\
$\tilde\pi_5(\tilde x)$ & Type probability conditional on
  augmented features\\[6pt]

\multicolumn{2}{l}{\textit{Target estimands}}\\[3pt]
$\Psi = \E[Y^*]$ & Population fraud rate\\
$\Psi_k = \bar\pi_k$ & Type-$k$ fraud rate\\

\end{longtable}
\end{center}

\end{document}